\documentclass[11pt,twoside]{article} % For LaTeX2e

\usepackage{graphicx}    % For including images
\usepackage{subfig}      % For \subfloat
\usepackage{caption}     % Helps with better caption formatting
\usepackage{float}       % Ensures proper float handling

\usepackage{eqnarray,amsmath}
\usepackage[utf8]{inputenc} % allow utf-8 input
\usepackage[T1]{fontenc}    % use 8-bit T1 fonts
\usepackage{booktabs}       % professional-quality tables
\usepackage{amsfonts}       % blackboard math symbols
\usepackage{nicefrac}       % compact symbols for 1/2, etc.
\usepackage{microtype}      % microtypography

\usepackage{epsf}
\usepackage{epsfig}
\usepackage{fancyhdr}
\usepackage{graphics}
\usepackage{graphicx}
\usepackage{psfrag}
\usepackage{fullpage}
\usepackage{pdfpages}

\usepackage{url}% for url's in bib
\usepackage[colorlinks,linkcolor=magenta,citecolor=blue, pagebackref=true]{hyperref}
\renewcommand*{\backrefalt}[4]{%
    \ifcase #1 \footnotesize{(Not cited.)}%
    \or        \footnotesize{(Cited on page~#2.)}%
    \else      \footnotesize{(Cited on pages~#2.)}%
    \fi}
\usepackage{color}

\usepackage{amsthm}
\usepackage{amsmath}
\usepackage{amssymb,bbm}
\usepackage{algorithmic}
\usepackage{algorithm}
\usepackage{textcomp}
\usepackage{siunitx}
\usepackage{wrapfig}
\usepackage{algorithmic}
\usepackage{algorithm}
\usepackage{multirow}
\usepackage{multicol}% colors

\newtheorem{assumption}{Assumption}
\newtheorem{lemma}{Lemma}

\newtheorem{theorem}{Theorem}
\newtheorem{proposition}{Proposition}
\newtheorem{definition}{Definition}
\newtheorem{corollary}{Corollary}%opening
\usepackage{macro_commands}

\begin{document}

\begin{center}

{\bf{\LARGE
{Sigmoid Self-Attention has Lower Sample Complexity
\\
than Softmax Self-Attention: 
\\ \vspace{2.5mm}
A Mixture-of-Experts Perspective}
% {
% Sigmoid Self-Attention is Better than Softmax Self-Attention: A Mixture-of-Experts Perspective}
}}

\vspace*{.2in}
{{
\begin{tabular}{ccccc}
Fanqi Yan$^{\diamond,\star}$ & Huy Nguyen$^{\diamond,\star}$ & Pedram Akbarian$^{\diamond}$  & Nhat Ho$^{\diamond,\dagger}$ & Alessandro Rinaldo $^{\diamond,\dagger}$ 
\end{tabular}
}}

\vspace*{.2in}

\begin{tabular}{c}
The University of Texas at Austin$^{\diamond}$
\end{tabular}

\today

\vspace*{.2in}

\end{center}

% \begin{center}

% {\bf{\LARGE{Sigmoid Self-Attention is Better than Softmax Self-Attention: A Mixture-of-Experts Perspective}}}

% \vspace*{.2in}
% {\large{
% \begin{tabular}{ccccc}
% Fanqi Yan$^{\diamond,\star}$ & Huy Nguyen$^{\diamond,\star}$ & Dung Le$^{\dagger,\star}$ & Pedram Akbarian$^{\diamond}$  & Nhat Ho$^{\ddagger}$ 
% \end{tabular}
% }}

% %\vspace*{.2in}
% %{\large{
% %\begin{tabular}{ccc}
% %Huy Nguyen$^{\dagger}$ & Nhat Ho$^{\dagger}$
% %\end{tabular}
% %}}

% \vspace*{.2in}

% \begin{tabular}{c}
% The University of Texas at Austin$^{\dagger}$
% \end{tabular}

% \today

% \vspace*{.2in}

% \end{center}

\begin{abstract}
  At the core of the popular Transformer architecture is the self-attention mechanism, which dynamically assigns softmax weights to each input token so that the model can focus on the most salient information. However, the softmax structure slows down the attention computation due to its row-wise nature, and it inherently introduces competition among tokens: as the weight assigned to one token increases, the weights of others decrease. This competitive dynamic may narrow the focus of self-attention to a limited set of features, potentially overlooking other informative characteristics. 
    Recent experimental studies have shown that using the element-wise sigmoid function helps eliminate token competition and reduce the computational overhead. Despite these promising empirical results, a rigorous comparison between sigmoid and softmax self-attention mechanisms remains absent in the literature. This paper closes this gap by theoretically demonstrating that sigmoid self-attention is more sample-efficient than its softmax counterpart. Toward that goal, we represent the self-attention matrix as a mixture of experts and show that ``experts'' in sigmoid self-attention require significantly less data to achieve the same approximation error as those in softmax self-attention. %We corroborate our theoretical findings through several numerical experiments.
\end{abstract}

\let\thefootnote\relax\footnotetext{$\star$ Equal contribution. $\dagger$ Co-last authors.}

\section{Introduction}
\label{sec:introduction}
Transformer models \cite{vaswani2017attention} have been known as the state-of-the-art architecture for a wide range of machine learning and deep learning applications, including language modeling \cite{Devlin2019BERTPO,brown2020language,raffel2020exploring,touvron2023llama}, computer vision \cite{Dosovitskiy2020AnII,carion2020end,radford2021learning,liu2021swin}, and reinforcement learning \cite{chen2021decision,kim2023preference,Hu2024LearningTA}, etc. One of the central components that contribute to the success of the Transformer models is the self-attention mechanism, which enables sequence-to-sequence models to concentrate on relevant parts of the input data. In particular, for each token in an input sequence, the self-attention mechanism computes a context vector formulated as a weighted sum of the tokens, where more relevant tokens to the context are assigned larger weights than others (see Section~\ref{sec:background} for a formal definition). Therefore, self-attention is able to capture long-range dependencies and complex relationships within the data. 

However, since the weights in the context vector are normalized by the softmax function, there might be an undesirable competition among the tokens, that is, an increase in the weight of a token leads to a decrease in the weights of others. As a consequence, the traditional softmax self-attention mechanism might focus only on a few aspects of the data and possibly ignore other informative features \cite{ramapuram2024sigmoidattention}. 
Additionally, Gu et al. \cite{gu2024attentionsink} also discovered that the tokens' inner dependence on the attention scores owing to the softmax normalization partly causes the attention sink phenomenon occurring when auto-regressive language models assign significant attention to the initial token regardless of their semantic importance. Furthermore, the softmax self-attention is not computationally efficient as it involves a row-wise softmax operation. In response to these issues, Ramapuram et al. \cite{ramapuram2024sigmoidattention} proposed replacing the row-wise softmax function with the element-wise sigmoid function, which not only alleviates unnecessary token competition but also speeds up the computation. Then, they extended the FlashAttention2 framework for the softmax self-attention \cite{dao2022flashattention,dao2024flashattention} to the setting of the sigmoid self-attention and showed that the latter shared the same performance as the former in several tasks but with faster training and inference. On the theoretical side, they demonstrated that the Transformer model equipped with the sigmoid self-attention was a universal function approximator on sequence-to-sequence tasks. Nevertheless, a theoretical comparison between the sigmoid self-attention and the softmax self-attention is lacking in the literature. 

% In this paper, our objective is to compare the sample efficiency of these two attention variants through their connection to the Mixture-of-Experts (MoE) model \cite{Jacob_Jordan-1991}. More specifically, we will show in the next section that each row of the self-attention matrix can be represented as an MoE with either the softmax gating or the sigmoid gating, where each row of the value matrix plays a role as an expert. %; see Section~\ref{sec:attention-moe}.
% Then, we say that the sigmoid self-attention is more sample efficient than the softmax self-attention if it takes the ``experts'' in the former model fewer data points to approximate a target function of interest within a given margin of error.

% \textcolor{red}{Unclear if here you’re claiming sigmoid attn is more sample-efficient, or if you’re defining what “sample efficiency” means in this context}
% \textcolor{orange}{To make this comparison precise, we define sample efficiency as the number of data points required for the experts to approximate a target function within a given margin of error. Under this definition, we say that the sigmoid self-attention is more sample-efficient than the softmax self-attention if the former requires fewer data points to achieve the same approximation accuracy.}

\textbf{Contributions.} In this paper, our objective is to compare the sample complexity of these two attention variants through their connection to the Mixture-of-Experts (MoE) model \cite{Jacob_Jordan-1991}. Our contributions are twofold and can be summarized as follows (see also Table~\ref{table:sample-complexity}):

\begin{table*}[t!]
\centering
\caption{Summary of the sample complexity of sigmoid/softmax self-attention under the MoE perspective, that is, the expert sample complexity to attain the approximation error $\epsilon$ (ignoring logarithmic factors of lower order). In this work, we consider two types of expert functions, including expert networks with ReLU, GELU activations; and polynomial experts. Below, $\tau$ denotes some positive constant.}
\begin{tabular}{@{}lccc@{}}
\toprule
                  & ReLU, GELU Experts & Polynomial Experts  \\ 
                  & Sparse Regime \quad Dense Regime & Sparse Regime \quad Dense Regime  \\
\midrule
$\mathrm{SigmoidAttn}$ (ours)          & 
$\mathcal{O}(\epsilon^{-4})$ \quad\quad\quad 
$\mathcal{O}(\epsilon^{-2})$ & 
\hspace{-8mm} $\mathcal{O}(\exp(\epsilon^{-1/\tau}))$ 
\quad\quad\quad 
$\mathcal{O}(\epsilon^{-2})$  \\
$\mathrm{SoftmaxAttn}$ \cite{akbarian2024quadratic}         & 
$\mathcal{O}(\epsilon^{-4})$ \quad\quad\quad 
$\mathcal{O}(\epsilon^{-2})$ & 
$\mathcal{O}(\exp(\epsilon^{-1/\tau}))$
\quad \quad
$\mathcal{O}(\exp(\epsilon^{-1/\tau}))$  \\
\bottomrule
\end{tabular}
\label{table:sample-complexity}
\end{table*}

\emph{1. Connection between self-attention and MoE:} In Section~\ref{sec:background}, we provide a rigorous derivation for illustrating that each row of a sigmoid self-attention matrix (resp. a softmax self-attention) head can be represented as an MoE with the sigmoid gating (resp. the softmax gating) and quadratic affinity scores, where each row of the value matrix plays a role as an expert. Although this connection has been mentioned in previous works \cite{akbarian2024quadratic,csordas2023switchhead,wu2024multi}, there have not been any theoretical guarantee for this result, to the best of our knowledge. 

\emph{2. Sample complexity comparison between sigmoid self-attention and softmax self-attention:}   
From the MoE perspective, we then perform a comprehensive convergence analysis of sigmoid self-attention. The results of our analysis demonstrate that polynomially many data points -- of order $\mathcal{O}(\epsilon^{-4})$ -- are needed to achieve an estimation accuracy of order $\epsilon$ for the value matrix's rows. On the other hand, a similar convergence analysis of softmax self-attention conducted in \cite{akbarian2024quadratic} reveals that, in order to reach the same approximation error, exponentially many data points -- of order $\mathcal{O}(\exp(\epsilon^{-1/\tau}))$, for some constant $\tau>0$ -- are needed. Thus, we claim that sigmoid self-attention is more sample-efficient than softmax self-attention. Furthermore, it is worth emphasizing that the convergence analysis performed in this work is more technically challenging than that in \cite{akbarian2024quadratic} due to an issue in the model convergence induced by the structure of the sigmoid function, which will be elaborated in Section~\ref{sec:problem-setup}.

\section{Background: Self-Attention and MoEs}
\label{sec:background}
% In this section, we first restate the Mixture of Experts (MoE), then we introduce the sigmoid attention, a variation on self-attention. And we will construct the connection between the sigmoid attention and MoE.

% Below, we introduce the self-attention in Section~\ref{sec:attention} and in Section~\ref{sec:moe} we describe the Mixture-of-Experts (MoE) modeling framework. 
% Finally, in Section~\ref{sec:attention-moe}, we establish the connection between the self-attention mechanism and the MoE. \textcolor{violet}{*ALE*: I would remove this paragraph}

%\subsection{Self-Attention Mechanism}
%\label{sec:attention}
{\bf Self-Attention Mechanism.}
The self-attention mechanism plays a crucial role in the Transformer architecture \cite{vaswani2017attention}. Given an input sequence $\mathbb{X}\in\mathbb{R}^{N \times d}$ of $N$ feature vectors of dimension $d$, the self-attention mechanism first projects it into the query matrix $Q\in \mathbb{R}^{N \times d_{k}}$, the key matrix $K\in \mathbb{R}^{N \times d_{k}}$, and the value matrix $V\in \mathbb{R}^{N\times d_v}$ through three following linear transformations: 
\begin{align*}
  Q  =  \mathbb{X}  W_Q ,
  \quad
  K  =  \mathbb{X}  W_K ,
  \quad
  V  =  \mathbb{X}  W_V ,    
\end{align*}
where 
$
  W_Q \in \mathbb{R}^{d \times d_{k}}, 
  W_K \in \mathbb{R}^{d \times d_{k}},
  W_V \in \mathbb{R}^{d \times d_v}
$
are learnable weight matrices.
% \paragraph{Scaled Dot-Product Attention.}
% A widely used variant of attention is the \emph{scaled dot-product} mechanism \cite{vaswani2017attention}.  
% Here 
% $Q \in \mathbb{R}^{N \times d_{k}}$
% is the query matrix,
% $K \in \mathbb{R}^{N \times d_{k}}$ is the key matrix with $N$ key vectors, and 
% $V \in \mathbb{R}^{N\times d_v}$ 
% is the value matrix with $N$ corresponding value vectors.
% Given a query matrix $Q \in \mathbb{R}^{N \times d_k}$, 
Then, the vanilla \emph{softmax self-attention} can be compactly written as
\begin{align}
\label{eq:softmax_attn}
  \mathrm{SoftmaxAttn}(\mathbb{X})
   = 
  \softmax\Bigl(\frac{Q K^T}{\sqrt{d_{k}}}\Bigr) V,
\end{align}
% with $\softmax(u_1,\cdots,u_d)=[\exp(u_1)/(\sum_{i=1}^d\exp(u_i)),\cdots,\exp(u_d)/(\sum_{i=1}^d\exp(u_i))]$
where the $\softmax$ function acts row-wise on the matrix $Q K^T/\sqrt{d_k} \in \mathbb{R}^{N \times N}$ as follows: $\softmax(u)=\frac{1}{\sum_{j=1}^N\exp(u_j)}(\exp(u_1),\cdots,\exp(u_N))$, where $u=(u_1,\cdots,u_N)$ is a row vector in $\mathbb{R}^N$.
% as $\softmax(u)_i=\exp(u_i)/(\sum_{j=1}^N\exp(u_j))$ for a row vector $u=[u_1,\cdots,u_N]$ 
% % \paragraph{Multi-Head Self-Attention (MSA).}
% Transformers often use multiple attention heads in parallel, referred to as \emph{Multi-Head Self-Attention (MSA)}~\cite{ramapuram2024sigmoidattention}.
% % Let $X^Q, X^K, X^V \in \mathbb{R}^{N \times d}$ denote the query, key, and value matrices for an input sequence of length~$N$, i.e.,
% % \begin{align*}
% %   X^Q  =  X^K  =  X^V  = 
% %   \begin{bmatrix}
% %     x_{1}\\begin{align*}3pt]
% %     x_{2}\\begin{align*}1pt]
% %     \vdots\\begin{align*}1pt]
% %     x_{N}
% %   \end{bmatrix}
% %    \in \mathbb{R}^{N \times d}.
% % \end{align*}
% The output of a single MSA layer is:
% \begin{equation}
%   \mathrm{MSA}(X)
%    := 
%   \mathrm{Concat}\bigl(h_{1}, \dots, h_{m}\bigr) W^{O}
%    \in \mathbb{R}^{N \times d},
% \end{equation}
% where each head $h_{i}$ is computed by standard attention:
% \begin{equation}
%   h_{i}
%    := 
%   \mathrm{Attention}\!\bigl(
%     X^Q W^{Q}_{i}, 
%     X^K W^{K}_{i}, 
%     X^V W^{V}_{i}
%   \bigr),
%   \quad
%   i \in [m].
% \end{equation}
% Here,
% \begin{align*}
%   W^{O} \in \mathbb{R}^{(m d_v)\times d}, 
%   W^{Q}_{i} \in \mathbb{R}^{d \times d_{k}},
%   \\
%   W^{K}_{i} \in \mathbb{R}^{d \times d_{k}},
%   W^{V}_{i} \in \mathbb{R}^{d \times d_{v}}, 
% \end{align*}
% are projection matrices, and $m$ is the number of heads.  
% \paragraph{SigmoidAttn: Replacing Softmax.}
On the other hand, the \emph{sigmoid self-attention} is computed by replacing the above softmax function with the sigmoid function $\sigma:z\mapsto (1+\exp(-z))^{-1}$ applied element-wise to the matrix ${Q K^T}/{\sqrt{d_{k}}}$:
% Whereas the usual self‐attention employs a row‐wise $\mathrm{softmax}$, one can replace it with a \emph{sigmoid} activation:
\begin{align}
\label{eq:sigmoid_attn}
  \mathrm{SigmoidAttn}(\mathbb{X})
   = 
  \sigma\Bigl(\frac{Q K^T}{\sqrt{d_{k}}}\Bigr) V.
\end{align}
% where
% \begin{align*}
%   % \text{where}\quad
%   \sigma: u  \mapsto  \frac{1}{1+\exp(-u)}
%   % \bigl(1 + e^{- (u )}\bigr)^{-1}.
% \end{align*}
% Here, $\sigma$ is applied \emph{element‐wise} to ${Q K^T}/{\sqrt{d_{k}}}$.
% , and $b\in\mathbb{R}$ is a bias parameter (e.g., $b=-\log(n)$).  
% Letting $\{ y_1,\dots,y_n\} = \text{SigmoidAttn}(X)$, we obtain
% \begin{align*}
%   y_i
%    = 
%   \sum_{j=1}^n
%   \frac{\exp\!\bigl(\langle W_q x_i, W_k x_j\rangle\bigr)}
%        {\exp\!\bigl(\langle W_q x_i, W_k x_j\rangle\bigr) + n} 
%   W_v x_j
%   \\
%     \longrightarrow  
%   \int
%   % {\exp\!\bigl(\langle W_q x_i, W_k x\rangle\bigr)}
%        {\exp\!\bigl(\langle W_q x_i, W_k x\rangle\bigr)} 
%        W_v x
%    \mathrm{d}\mu(x),
% \end{align*}
% where $\mu = \tfrac{1}{n}\sum_{j=1}^n \delta_{x_j}$ is the empirical measure of the training points $\{x_j\}$.  Even as $n \to \infty$, $\mu$ remains non‐Dirac, and the normalization from $\sigma$ with $b=-\log(n)$ preserves meaningful behavior \cite{Wortsman2023a}.

% \paragraph{Multi-head SigmoidAttn.}
% A multi‐head version of the model \eqref{eq:sigmoid_attn} simply concatenates the outputs of several $\mathrm{SigmoidAttn}$ blocks:
% \begin{equation}
%   \bigl[\mathrm{SigmoidAttn}_1(\mathbb{X}), \dots, \mathrm{SigmoidAttn}_h(\mathbb{X})\bigr] W_o,
%   \tag{5}
% \end{equation}
% where $W_o \in \mathbb{R}^{(h d_v)\times d}$ is a learnable output matrix and $h$ is the number of heads.

%\subsection{Mixture-of-Experts Model}
%\label{sec:moe}

{\bf Mixture of Experts.}
% \cite{le2024mixture}
% Mixture of experts (MoE) extends classical mixture models with an adaptive gating mechanism.
% An MoE model consists of a group of N expert networks
% $f_i:\Real^d\rightarrow\Real^{d_v}$, for all $i\in[N]$, 
% and a gate function $G:\Real^d\rightarrow\Real^N$.
% Given an input $\h\in\Real^d$, 
% MoE computes a weighted sum of expert outputs $f_i(\h)$ based on learned score function 
% $s_i:\Real^d\rightarrow\Real$ for each expert:
% \begin{align}
%     \textbf{y}:=\sum_{j=1}^N G(\h)_j\cdot f_j(\h)
%     :=\sum_{j=1}^N
%     % \frac{1}{1+\exp{(-s_j(\h))}}
%     \sigma(s_j(\h))
%     \cdot f_j(\h)
% \end{align}
% where $G(\h)$ is the gating function. 
% And we specified  %$:=\sigmoid(s_1(\h),\cdots,s_N(\h))$.
% \begin{align*}
%   \sigma : u \mapsto \sigmoid (u+b) =
%   \frac{1}{1+\exp(-u-b)},
% \end{align*}
% where $b$ is a hyper-parameter.
% Building on this concept, works by \cite{Eigen_learning_2014, Quoc-conf-2017} established
% the MoE layer as a fundamental building block to scale up model capacity efficiently.
 Mixture of experts (MoE) \cite{Jacob_Jordan-1991} is an extension of classical mixture models \cite{Lindsay-1995} that aggregates the power of $N$ sub-models called experts, each of which can be formulated as a feed-forward network \cite{shazeer2017topk}, a regression function \cite{faria2010regression}, or a classifier \cite{chen2022theory} denoted by $\mathcal{E}_i:\mathbb{R}^d \to \mathbb{R}^{d_v}$ for $1\leq i\leq N$. 
 For that purpose, the MoE employs an adaptive gating mechanism, denoted by $\mathcal{G}: \mathbb{R}^d\to\mathbb{R}^N$, to calculate input-dependent weights for these experts in a dynamic way, 
 % \textcolor{red}{ -> “dynamic way, so that the more”}
 so that the more relevant experts to the input will be assigned larger weights.
 Then, given an input $ h \in \mathbb{R}^d $, the MoE output $y\in\mathbb{R}^{d_v}$ is expressed as a weighted sum of the expert outputs:
% $\mathcal{E}_i(h) $. The weights are derived from the gating function $ \mathcal{G}(h) $, which uses a learned score function $ s_i: \mathbb{R}^d \to \mathbb{R} $ for each expert. Formally, the computation is expressed as:
% \begin{align}
% y = \sum_{j=1}^N \mathcal{G}(h)_j \cdot \mathcal{E}_j(h) := \sum_{j=1}^N \sigma(s_j(h)) \cdot \mathcal{E}_j(h),    
% \end{align}
\begin{align}
y = \sum_{j=1}^N \mathcal{G}(h)_j \cdot \mathcal{E}_j(h).    
\end{align}
In practice, there are two main types of gating functions, namely \emph{the sigmoid gating function} \cite{deepseekv3,csordas2023approximating,chi_representation_2022} and \emph{the softmax gating function} \cite{Jacob_Jordan-1991,Jordan-1994} defined as 

(i) \emph{Sigmoid gating}: $\mathcal{G}(h)_j:=\sigma(s_j(h))$ for $1\leq j\leq N$, where $s_j(h)\in\mathbb{R}$ represents the affinity score between the input $h$ and the $j$-th expert $\mathcal{E}_j$;

(ii) \emph{Softmax gating}: $\mathcal{G}(h)_j=\softmax(s(h))_j$ for $1\leq j\leq N$, where $s(h):=(s_1(h), s_2(h),\ldots,s_N(h))$.
% \begin{itemize}
%     \item \textbf{Sigmoid gating}: $\mathcal{G}(h)_j:=\sigma(s_j(h))$ for $1\leq j\leq N$, where $s_j(h)\in\mathbb{R}$ represents the affinity score between the input $h$ and the $j$-th expert $\mathcal{E}_j$;
%     \item \textbf{Softmax gating}: $\mathcal{G}(h)_j=\softmax(s(h))_j$ for $1\leq j\leq N$, where $s(h):=(s_1(h), s_2(h),\ldots,s_N(h))$.
% \end{itemize}

Due to its adaptability and expressiveness, the MoE has been widely leveraged in several fields, including natural language processing \cite{Du2021GLaMES,fedus2021switch,zhou2023brainformers,jiang2024mixtral,Muennighoff2024OLMoEOM,puigcerver2024sparse}, multi-task learning \cite{ma2018modeling,hazimeh2021dselect,chen2023mod}, and speech recognition \cite{jain2019multi,gaur2021mixture,you2021speechmoe,kwon2023mole}, etc.
{\bf Self-Attention meets Mixture of Experts.}
We will now show that each row of the sigmoid/softmax self-attention matrix can be represented as an MoE with quadratic affinity scores. Since the subsequent arguments apply for both attention variants, we will present only the derivation using the sigmoid self-attention matrix, which we recall is defined as
\begin{align*}
  \mathrm{SigmoidAttn}(\mathbb{X}) =
  \sigma(\mathbb{X}B\mathbb{X}^{\top})\mathbb{X}W_V,
\end{align*}
where $B := \frac{W_QW_K^{\top}}{\sqrt{d_k}} \in \mathbb{R}^{d \times d}$.
% $A := \frac{QK^{\top}}{\sqrt{d_k}}
%     = \frac{XW_QW_K^{\top}X^{\top}}{\sqrt{d_k}}$.
% \begin{align*}
%     A := \frac{QK^{\top}}{\sqrt{d_k}}
%     = \frac{XW_QW_K^{\top}X^{\top}}{\sqrt{d_k}}
%     = XBX^{\top},
% \end{align*}
% with $B := \frac{W_QW_K^{\top}}{\sqrt{d_k}} \in \mathbb{R}^{d \times d}$.
% \begin{align*}
% B := \frac{W_QW_K^{\top}}{\sqrt{d_k}} \in \mathbb{R}^{d \times d}.
% \end{align*}
Let $ x_i \in \mathbb{R}^{1 \times d} $ denote the $ i $-th row vector of the input matrix $ \mathbb{X} $. 
% Then, the $ i $-th row vector of the matrix $ A $ is given by:
% \begin{align*}
% [A]_{i,:} = x_iBX^{\top} \in \mathbb{R}^{1 \times N},
% \end{align*}
% where 
Since the sigmoid function is applied element-wise, the $(i, j)$-th element of the matrix $ \sigma(\mathbb{X}B\mathbb{X}^{\top}) \in \mathbb{R}^{N \times N} $ is
% \begin{align*}
% [\sigma(XBX^{\top})]_{i,j} = \frac{1}{1 + \exp(-x_iBx_j^{\top})},
% \end{align*}
\begin{align*}
[\sigma(\mathbb{X}B\mathbb{X}^{\top})]_{i,j} = \sigma(x_iBx_j^{\top}).
\end{align*}
As a result, the $ i $-th row vector of the matrix $ \mathrm{SigmoidAttn}(\mathbb{X}) $ takes the form
\begin{align*}
   [\mathrm{SigmoidAttn}(\mathbb{X})]_{i,:}
   % &=
   % \sum_{j=1}^N
   % \frac{1}{1+\exp(-x_iBx_j^{\top})} \cdot [V_{j,:}] \\
   &=
   \sum_{j=1}^N
   \sigma(x_iBx_j^{\top}) \cdot x_jW_V.
\end{align*}

Next, let $ X = [x_1, x_2, \ldots, x_N] \in \mathbb{R}^{1 \times Nd} $ be the concatenation of $ N $ input tokens. For each $ 1 \leq i \leq N $, let $ E_i \in \mathbb{R}^{Nd \times d} $ be the matrix such that $ XE_i = x_i $. Then, the $ i $-th row vector of the sigmoid self-attention matrix can be rewritten as
\begin{align*}
   [\mathrm{SigmoidAttn}(\mathbb{X})]_{i,:}
   &=
   \sum_{j=1}^N
   \sigma(XE_iBE_j^{\top}X^{\top}) \cdot XE_jW_V 
   \sum_{j=1}^N
   \sigma(XM_{ij}X^{\top}) \cdot XP_j,
\end{align*}
where $M_{ij} := E_iBE_j^{\top} = \frac{E_iW_QW_K^{\top}E_j^{\top}}{\sqrt{d_k}}$ and $P_j := E_jW_V$.
% \begin{align*}
% M_{ij} := E_iBE_j^{\top} = \frac{E_iW_QW_K^{\top}E_j^{\top}}{\sqrt{d_k}}, \quad
% P_j := E_jW_V^{\top}.
% \end{align*}

Hence, each row of the sigmoid self-attention matrix can be represented as an MoE with quadratic affinity scores.

% \begin{theorem}
% \label{function-convergence}
%     \begin{align*}
%         \Vert f_{\hGn}-f_{\Gs}\Vert_{L^2(\mu)}=
%         \mathcal{O}_P\left(\frac{\log^{\frac{1}{2}}(n)}{n^{\frac{1}{2}}} \right)
%     \end{align*}
% \end{theorem}

% \section{Convergence Rates}

\section{Problem Setup}
\label{sec:problem-setup}
%In this section, we present the problem setup for studying the sample efficiency of the sigmoid self-attention under the perspective of the sigmoid gating MoE with quadratic affinity scores. A similar convergence analysis for the softmax gating MoE with quadratic affinity scores can be found in \cite{akbarian2024quadratic}. %, so it is omitted in this paper.

Suppose that the data $(X_1,Y_1),(X_2,Y_2),\cdots,(X_n,Y_n)\in\Real^d\times\Real $ are generated according to the regression model
\begin{align}
    \label{eq:regression-model}
    Y_i=f_{\Gs}(X_i)+\varepsilon_i,~i=1,2,\ldots,n,
\end{align}
where $X_1,X_2,\cdots,X_n$ are i.i.d. samples from a probability distribution $\mu$ on $\Real^d$, and 
$\varepsilon_1,\varepsilon_2,\ldots,\varepsilon_n$ are i.i.d. Gaussian noise variables with $\bbE[\varepsilon_i|X_i]=0$ and $\mathrm{Var}[\varepsilon_i|X_i]=\nu$, for  $1\leq i\leq n$.
Additionally, the regression function $f_{\Gs}:\mathbb{R}^d\to\mathbb{R}$ is unknown and formulated as a sigmoid gating mixture of $N^*$ experts, i.e.
\begin{align}
    \label{eq:quadratic_MoE}
    f_{\Gs}(x):=\sum_{i=1}^{\ns}\frac{1}{1+
    % \exp(-x^{\top}\alpha_i^* x)
    \exp(-s(x,\theta^*_i))
    }\cdot
    \mathcal{E}(x,\ei),
\end{align}
where $\mathcal{E}(x,\eta^*_i)$ denotes the parametric expert function, while the affinity score function $s(x,\theta^*_i)$ takes a quadratic form. In this work, we consider  two types of affinity score functions:

(i) the \emph{fully quadratic score}: $s(x,\theta^*_i)=x^{\top}A^*_ix+(b^*_i)^{\top}x+c_i^*$;

(ii) and the \emph{partially quadratic score}: 
$s(x,\theta^*_i)=x^{\top}A^*_ix+c_i^*$;

% \textcolor{violet}{*ALE*: maybe better to say ``fully quadratic" vs ``partially quadratic"}
% \begin{itemize}
%     \item Full quadratic: $s(x,\theta^*_i)=x^{\top}A^*_ix+(b^*_i)^{\top}x+c_i^*$;
%     \item Partial quadratic: $s(x,\theta^*_i)=x^{\top}A^*_ix+c_i^*$,
% \end{itemize}
where 
% $\theta:=(A, b,c,\eta)\in\Theta\subseteq\mathbb{R}^{d\times d}\times\Real^{d}\times\Real\times\Real^q$ 
$\Theta=\{(A, b,c,\eta)\in\mathbb{R}^{d\times d}\times\Real^{d}\times\Real\times\Real^q\}$ 
% :1\leq i\leq N\}$ 
denotes the parameter space. Given this setup, our goal is to determine the sample size necessary for the estimators of the parameters and experts in model~\eqref{eq:quadratic_MoE} to reach an approximation error $\epsilon>0$. Due to space limitation, we will present only the results for the case of a fully quadratic affinity score function in Section~\ref{sec:sample-efficiency} and defer those for its partial version to Appendix~\ref{appsec:additional-results}. We also refer to Akbarian et al. \cite{akbarian2024quadratic} for a similar convergence analysis for softmax gating MoE with quadratic affinity scores.

\textbf{Notation.} We let $[n]$ be the set $\{1,2,\ldots,n\}$ for any  $n\in\mathbb{N}$. Next, for any set $S$, we denote $|S|$ as its cardinality. For any vectors $v:=(v_1,v_2,\ldots,v_d) \in \mathbb{R}^{d}$ and $\alpha:=(\alpha_1,\alpha_2,\ldots,\alpha_d)\in\mathbb{N}^d$, we let $v^{\alpha}=v_{1}^{\alpha_{1}}v_{2}^{\alpha_{2}}\ldots v_{d}^{\alpha_{d}}$, $|v|:=\sum_{i=1}^{d}v_i$ and $\alpha!:=\alpha_{1}!\alpha_{2}!\ldots \alpha_{d}!$, while $\|v\|$ denotes its $2$-norm value. Lastly, for any positive sequences $(a_n)_{n\geq 1}$ and $(b_n)_{n\geq 1}$, we write $a_n = \mathcal{O}(b_n)$ or $a_{n} \lesssim b_{n}$ if $a_n \leq C b_n$ for all $ n\in\mathbb{N}$, where $C > 0$ is some universal constant. For a sequence $(A_n)_{n\geq 1}$ of positive random variables, the notation $A_{n} = \mathcal{O}_{P}(b_{n})$ signifies $A_{n}/b_{n}$ is stochastically bounded, i.e., for any $\epsilon>0$, there exists an $M>0$ such that $\mathbb{P}( A_{n}/b_{n} > M) < \epsilon $ for all sufficiently large $n$.
%The notation $a_{n} = \mathcal{O}_{P}(b_{n})$ indicates that $a_{n}/b_{n}$ is stochastically bounded. 

\textbf{Least squares estimation.} 
To estimate the unknown ground-truth parameters $\{\Asi,\bsi,\csi,\etasi\}_{i=1}^{\ns}$, 
we use the least squares method \cite{Vandegeer-2000} to compute the estimator
\begin{align}
\label{eq:lse}
    \hGn:=\argmin_{G\in\calm_N(\Theta)}
    \sum_{i=1}^{n}
    \left(
    Y_i-f_G(X_i)
    \right)^2,
\end{align}
where 
$\calm_N(\Theta):=\{ G=\sum_{i=1}^{N^\prime}\frac{1}{1+\exp(-c_i)}\delta_{(A_i,b_i,\eta_i)}:1\leq N^\prime \leq N, (A_i,b_i,c_i,\eta_i)\in\Theta \}$
is the set of all mixing measures with at most $N$ atoms. As the number of ground-truth experts $\ns$ is unknown in practice, we assume that the number of fitted 
\textcolor{black}{experts $N$ is larger} 
% \textcolor{red}{“experts $N$ larger” -> “experts $N$ is larger”} 
than $\ns$, i.e. $N>\ns$.

%\subsection{Convergence behavior of the mixture weights}
\textbf{A challenge in the convergence of the regression function estimator.} 
% We fit the ground-truth MoE model with a mixture of $k>\ns$, 
Since the number of fitted experts is larger than the number of ground-truth experts,
there must exist some atom $(A_i^*,b_i^*,\eta_i^*)$ of the mixing measure $\Gs$ that is fitted by at least two atoms of $\hGn$; we will refer to $(A_i^*,b_i^*,\eta_i^*)$ as an \emph{over-specified} atom of $\Gs$. 
% We over-specify the true MoE model by a mixture of $k$ experts where $k>\ns$.
% There exist some atoms $(\Asi,\bsi,\etasi)$ approximated by at least two fitted components, the over-specified atoms.
For example, if, say, $(\hAin,\hbin,\hetain)\to(\Asone,\bsone,\etasone)$  in probability as $n \rightarrow \infty$, for $i\in\{1,2\}$, then the corresponding estimators of the expert functions converge in probability as well, i.e. $\mathcal{E}(x,\hetain)\to\mathcal{E}(x,\etasone)$, in probability, for $i=1,2$.
This will, in turn, ensure the convergence of the regression function, i.e. $\| f_{\hGn}-f_{\Gs} \|_{L^2(\mu)}\rightarrow 0$ in probability, provided that, for $\mu$-almost every $x$,
\begin{align*}
    \sum_{i=1}^2
    \frac{1}{1+\exp
    \left(
    -x^{\top}\hAin x-\hbint x-\hcin
    \right)
    }
    \rightarrow
    \frac{1}{1+\exp
    \left(
    -x^{\top}\Asone x-(\bsone)^{\top} x-\csone
    \right)
    },
\end{align*}
in probability as $n\rightarrow\infty$.  
Notably, the above limit holds only if $\Asone=0_{d\times d}$ and $\bsone=0_d$ (see Appendix~\ref{appendix:cov_single_sigmoid} for further details).
Thus, we will divide our analysis into two complementary regimes of the gating parameters:

\emph{(i) Sparse regime}: all the over-specified gating parameters are zero, $(\Asi,\bsi)=(0_{d\times d},0_d)$;

\emph{(ii) Dense regime}: at least one among the over-specified gating parameters is non-zero, $(\Asi,\bsi)\neq(0_{d\times d},0_d)$.

The next results derive separate convergence rates for the regression function estimator $f_{\hGn}$ under the two regimes.
\begin{proposition}
\label{thm:function-convergence-specified}
Under the sparse regime of the gating parameters, 
%the regression estimator $f_{\hGn}$ converges to $f_{\Gs}$ at the  rate
    \begin{align}
        \label{eq:regression_rate_sparse_regime}
        \Vert f_{\hGn}-f_{\Gs}\Vert_{L^2(\mu)}=
        \mathcal{O}_P(\sqrt{\log(n)/n}).
    \end{align}
\end{proposition}
The proof of Proposition~\ref{thm:function-convergence-specified} can be found in Appendix \ref{appendix:proof_function_convergence}. 
The bound~\eqref{eq:regression_rate_sparse_regime} reveals that the convergence rate of the regression function estimator $f_{\hGn}$ to the ground-truth regression function is of parametric order $\mathcal{O}_P(\sqrt{\log(n)/n})$ under the sparse regime. In contrast, under the dense regime, the smallest $L_2(\mu)$ distance between the regression function estimator $f_{\hGn}$ and the set of MoE regression functions $f_{\lG}$ with more than $\ns$ experts functions where $\lG \in \overline{\calm}_N(\Theta):=\argmin_{G\in\calm_N(\Theta)\setminus\calm_{\ns}(\Theta)}\| f_G-f_{\Gs} \|_{L_2{(\mu)}}$ vanishes to 0 as $n \to \infty$. The rate for that convergence is given by the following result.
%converges to $f_{\lG}$, where 
%$\lG\in\overline{\calm}_N(\Theta):=\argmin_{G\in\calm_N(\Theta)\setminus\calm_{\ns}(\Theta)}\| f_G-f_{\Gs} \|_{L_2{(\mu)}}$, i.e. any MoE with more than $\ns$ experts that is closest to $f_{\Gs}$ in $L_2(\mu)$.  %Using similar arguments from Proposition~\ref{thm:function-convergence-specified}, we can also determine the convergence behavior of the regression function estimator in the following theorem:
\begin{corollary}
\label{thm:function-convergence-misspecified}
Under the dense regime of the gating parameters, 
%the regression estimator $f_{\hGn}$ converges to $f_{\lG}$ at the  rate
    \begin{align*}
        \inf_{\lG\in\overline{\calm}_N(\Theta)}\Vert f_{\hGn}-f_{\lG}\Vert_{L^2(\mu)}=
        \mathcal{O}_P(\sqrt{\log(n)/n}).
    \end{align*}
\end{corollary}

%\textcolor{violet}{*ALE*: just to make sure: the above results says that the distance between $f_{\hGn}$ and the set $\overline{\calm}_N(\Theta)$ of minimizers vanishes at the rate $\mathcal{O}_P(\sqrt{\log(n)/n})$, right? }

\textbf{From regression convergence rate to expert rate.} In light of  Proposition~\ref{thm:function-convergence-specified} and Corollary~\ref{thm:function-convergence-misspecified}, we conclude that, if there exists a loss function, say,  $\mathcal{L}(\cdot,\cdot)$ between the parameters $G$ and $G_*$ such that the lower bound $\Vert f_{\hGn}-f_{\Gs}\Vert_{L^2(\mu)}\gtrsim\mathcal{L}(\hGn,\Gs)$ holds true, then we immediately deduce $\mathcal{L}(\hGn,\Gs)=\mathcal{O}_P(\sqrt{\log(n)/n})$, thereby obtaining a convergence rate for the expert parameters. 
We carry out this program in the next section.%convergence rate that we will study in Section~\ref{sec:sample-efficiency}.

\section{Sample Complexity of Sigmoid Self-Attention}
\label{sec:sample-efficiency}
In this section, we investigate the sample complexity of the sigmoid self-attention through the perspective of the sigmoid gating MoE with fully quadratic affinity score function. In particular, we determine how much data the experts need to reach an approximation error $\epsilon$, which can be deduced from the expert convergence rate. We start with the sparse regime of the gating parameters in Section~\ref{sec:sparse_regime}, and then proceed with the dense regime in Section~\ref{sec:dense_regime}.

% \subsection{Strong identifiability}

\subsection{Sparse Regime of Gating Parameters}
\label{sec:sparse_regime}
%\subsubsection{Strong identifiability}
% \textcolor{red}{Could we still assume a Strong identifiability situation?}
Recall that under the sparse regime, all the over-specified gating parameters are zero, that is, $(\Asi,\bsi)=(0_{d\times d}, 0_{d})$.
Suppose that $\{\Asi,\bsi\}_{i=1}^{\bn}$ are \emph{over-specified} parameters, i.e., those fitted by at least two estimators, where $1\leq\bn\leq\ns$. The remaining gating parameters are \emph{exactly-specified} 
% \textcolor{violet}{*ALE*: maybe say exactly-specified} 
parameters $\{\Asi,\bsi \}_{\bn+1}^{\ns}$, i.e., those fitted by exactly one estimator. 
\begin{align*}
% \downarrow
    \underbrace{(A^*_1,b^*_1),\cdots,(A^*_{\bn},b^*_{\bn})}_{\text{over-specified}},~
    \underbrace{(A^*_{\bn+1},b^*_{\bn+1}),\cdots,(A^*_{\ns},b^*_{\ns})}_{\text{exactly-specified}}
\end{align*}
As mentioned in Section~\ref{sec:problem-setup}, in order to obtain the expert convergence rate, it is sufficient to establish the lower bound $\Vert f_{\hGn}-f_{\Gs}\Vert_{L^2(\mu)}\gtrsim\mathcal{L}(\hGn,\Gs)$. A popular approach adopted in previous works \cite{manole22refined,nguyen2024squares} for this problem is to decompose the difference $f_{\hGn}(x)-f_{\Gs}(x)$ by applying a Taylor expansion to the product of the sigmoid gating function and the expert function given by 
% \begin{align*}
%     F(x;A,b,c,\eta)&:=\sigma(x,A,b,c)\cdot \cale(x,\eta)
%     \\
%     &=
%     \frac{1}{1+\exp
%     \left(
%     -x^{\top}Ax-b^{\top}x-c
%     \right)
%     }
%     \cdot
%     \cale
%     \left(
%     x,\eta
%     \right),
% \end{align*}
\begin{align*}
    F(x;A,b,c,\eta)&:=\sigma(x^{\top}Ax+b^{\top}x+c)\cdot \cale(x,\eta).
\end{align*}
This decomposition of the regression function is expected to consist of linearly independent terms so that when $\Vert f_{\hGn}-f_{\Gs}\Vert_{L^2(\mu)}\to0$ as $n\to\infty$, the parameter discrepancies in the decomposition will also converge to zero, leading to both parameter and expert convergence.
% where $\sigma(x,A,b,c)=1/(1+\exp
%     \left(
%     -x^{\top}Ax-b^{\top}x-c
%     \right))$
% is the sigmoid function.
To secure such linear independence, we need to impose a \emph{strong identifiability} condition  on the function $x\mapsto F(x;A,b,c,\eta)$.
\begin{definition}[Strong identifiability]
\label{def:strong_identifiability}
    We call an expert function $x\mapsto\cale(x,\eta)$ \emph{strongly identifiable} if it is twice differentiable w.r.t its parameter $\eta$ for $\mu$-almost all $x$ and, for any natural number $\ell$
    and any distinct parameters 
    $\left\{ (A_i, b_i,c_i,\eta_i) \right\}_{i=1}^{\ell}$, 
    the sets of functions 
    \begin{align*}
        &\hspace{-0.em}\Bigg\{  
        \frac{\partial^{|\gamma_1|+|\gamma_2|+|\gamma_3|} F}{\partial A^{\gamma_1}\partial b^{\gamma_2}\partial\eta^{\gamma_3}}
    (x,0_{d\times d},0_d,c_i,\eta_i)
    \hspace{-0.em}:
    i\in[\ell],
    % 1\leq
    \sum_{j=1}^3|\gamma_j|
    \in[2]
    % \leq 2
    % (\gamma_1,\gamma_2,\gamma_3)\in\Real^{d\times d}
     \hspace{-0.em}   \Bigg\} \quad{and} \\
    % \end{align*}
    % and
    % \begin{align*}
    &\hspace{-0.em}\Bigg\{  
        \frac{\partial
        % ^{|\tau_1|+|\tau_2|+|\tau_3|+|\tau_4|} 
        F}
        {\partial A^{\tau_1}\partial b^{\tau_2}\partial c^{\tau_3}\partial\eta^{\tau_4}}
    (x;A_i,b_i,c_i,\eta_i)
    \hspace{-0.em}:
    i\in[\ell],
    \sum_{j=1}^4|\tau_j|
    =1
        \Bigg\}
    \end{align*}
    each contains linearly independent, funcitons for $\mu$-almost all $x$,
    where $(\gamma_1,\gamma_2,\gamma_3)\in\bbN^{d\times d}\times\bbN^d\times\bbN^q$ and 
    $(\tau_1,\tau_2,\tau_3,\tau_4)\in\bbN^{d\times d}\times\bbN^d\times\bbN\times\bbN^q$.
\end{definition}
\textbf{Examples.}
Let us consider two-layer neural networks of the form $\mathcal{E}(x,(\alpha,\beta,\lambda))=\lambda\phi(\alpha^{\top}x+\beta)$, where $\phi$ is some activation function and $(\alpha,\beta,\lambda)\in\mathbb{R}^d\times\mathbb{R}\times\mathbb{R}$. It can be verified that if $\phi$ is the $\relu$ or $\gelu$ function, $\alpha\neq0_d$, and $\lambda\neq 0$, then the function $x\mapsto\mathcal{E}(x,(\alpha,\beta,\lambda))$ is strongly identifiable. In contrast, if the expert function is of polynomial form $\mathcal{E}(x,(\alpha,\beta))=(\alpha^{\top}x+\beta)^p$ for some $p\in\mathbb{N}$, then it fails to satisfy the strong identifiability condition. For the case of $p=1$, this is due to the  linear dependence expressed by the  partial differential equations (PDEs) 
% \begin{align}
%     \frac{\partial^2 F}{\partial A \partial c}=\frac{\partial^2 F}{\partial b \partial b^{\top}},
%     % {\partial b^2},
% \end{align}
% also the parameters of the gating function and 
% the parameters of 
\begin{align}
    \label{eq:PDE}
    \frac{\partial^2 F}{\partial A \partial c}=\frac{\partial^2 F}{\partial b \partial b^{\top}},
    ~~
    \frac{\partial^2 F}{\partial A \partial \beta}=\frac{\partial^2 F}{ \partial b\partial \alpha},~~
    \frac{\partial^2 F}{\partial b \partial \beta}=\frac{\partial^2 F}{ \partial c\partial \alpha}.
\end{align}
Intuitively, the strong identifiability condition helps eliminate potential interactions among parameters expressed in the language of PDEs, namely those in equation~\eqref{eq:PDE}, where gating parameters interact with themselves and
with expert parameters. We will show later in Theorem~\ref{thm:dtwor_loss_linear} that those interactions lead to strikingly slow expert convergence rates, thereby reducing the model sample complexity.

In the next sections, we will study the convergence behavior of strongly identifiable experts and polynomial experts.% in Section~\ref{sec:strongly-identifiable} and Section~\ref{sec:polynomial}, respectively.

\subsubsection{Strongly Identifiable Experts}
\label{sec:strongly-identifiable}
\textbf{Voronoi loss.} 
For a mixing measure $G$ with $1\leq \np\leq N$ atoms, we allocate its atoms across the Voronoi cells $\{\mathcal{A}_j\equiv
    \mathcal{A}_j(G),j\in[\ns] \}$ generated by the atoms of $\Gs$, where
% \begin{align}
%     \hspace{-0.5em}\mathcal{A}_j:=
%     % \mathcal{A}_j(G):=
%     \left\{
%     i\in[\np]:
%     \| \omega_i-\omega_j^* \|
%     \leq
%     \| \omega_i-\omega_{\ell}^* \|,
%     \forall \ell\neq j
%     \right\},
%     % j=1,\cdots,\ns
% \end{align}
% with $\omega_i:=(A_i,b_i,\eta_i)$ and
% $\omega^*_j:=(A^*_j,b^*_j,\eta^*_j)$ for all $j\in[\ns]$. 
\begin{align}
    \hspace{-0.5em}\mathcal{A}_j:=
    % \mathcal{A}_j(G):=
    \left\{
    i\in[\np]:
    \| \theta_i-\theta_j^* \|
    \leq
    \| \theta_i-\theta_{\ell}^* \|,
    \forall \ell\neq j
    \right\},
    % j=1,\cdots,\ns
\end{align}
with $\theta_i:=(A_i,b_i,\eta_i)$ and
$\theta^*_j:=(A^*_j,b^*_j,\eta^*_j)$ for all $j\in[\ns]$.
Then, the Voronoi loss function is defined as
\begin{align*}
%\label{eq:done}
    % &
    \lone  := & \sum_{j=1}^{\bn}
    \left|
    \sum_{i\in\calAj} \frac{1}{1+\exp(-c_i)}-\frac{1}{1+\exp(-c_j^*)}
    \right|
    %\nonumber
  %& 
    +\sum_{j=1}^{\bn}
    \sum_{i\in\calAj}
    \Big[
    \|\Delta A_{ij} \|^2
    +\|\Delta b_{ij} \|^2 \\
    % \nonumber\\
  &   +\|\Delta \eta_{ij} \|^2
    \Big]
   % \nonumber
    %\\
    +\sum_{j=\bn+1}^{\ns}
    \sum_{i\in\calAj}
    \Big[
    \|\Delta A_{ij} \|
    +\|\Delta b_{ij} \|
    +|\Delta c_{ij} |
    +\|\Delta \eta_{ij} \|
    \Big],
\end{align*}
where we denote $\Delta A_{ij}:=A_i-A_j^*$,
$\Delta b_{ij}:=b_i-b_j^*$,
$\Delta c_{ij}:=c_i-c_j^*$
and $\Delta \eta_{ij}:=\eta_i-\eta_j^*$.
In the statement above, if the Voronoi cell $ \mathcal{A}_j $ is empty, the corresponding summation term is conventionally defined to be zero.
Additionally, the Voronoi loss function $ \mathcal{L}_1 $ can be computed efficiently, with a computational complexity of $ \mathcal{O}(N \times \ns) $.

With the above Voronoi loss at hand, we finally obtain the following parameter convergence rate.% in Theorem~\ref{thm:done_loss}.
\begin{theorem}
\label{thm:done_loss}
    If the expert function $x\mapsto\cale(x,\eta)$ is strongly identifiable, then the lower bound $\Vert f_{G}-f_{\Gs}\Vert_{L^2(\mu)}
        \gtrsim
        \lone$ holds true for any $G\in\calm_N(\Theta)$, then
    % \begin{align*}
    %     \Vert f_{G}-f_{\Gs}\Vert_{L^2(\mu)}
    %     \gtrsim
    %     \lone.
    % \end{align*}
    %Consequently, the result in Proposition~\ref{thm:function-convergence-specified} implies that
    \begin{align*}
        \mathcal{L}_1(\hGn,\Gs)=\mathcal{O}_P(
        \sqrt{
        {\log(n)}/{n} 
        }).
    \end{align*}
\end{theorem}
The proof of Theorem~\ref{thm:done_loss} is in Appendix \ref{proof:done_loss}. A few comments regarding the above result are in order.

\emph{(i) Parameter convergence rates:} From the construction of the Voronoi loss $\mathcal{L}_1$, the convergence rates for estimating the over-specified parameters $\Asi,\bsi,\etasi, i\in[\bn]$, are all of the same order $\mathcal{O}_P([{\log(n)}/{n}]^{\frac{1}{4}})$. On the other hand, those for the exactly-specified parameters $\Asi,\bsi,\etasi, \bn+1\leq i\leq \ns$ are faster, of the order $\mathcal{O}_P([{\log(n)}/{n}]^{\frac{1}{2}})$.

\emph{(ii) Expert convergence rates:} Since the expert function $\cale(\cdot,\eta)$ is twice differentiable w.r.t $\eta$ over a bounded domain, it is also a Lipschitz function w.r.t $\eta$. Therefore, by denoting $\hGn=\sum_{i=1}^{\widehat{N}_n}\frac{1}{1+\exp(-\hat{c}^n_i)}\delta_{(\widehat{A}^n_i,\widehat{b}^n_i,\widehat{\eta}^n_i)}$, we conclude that
\begin{align}
    \label{eq:expert-rates}
    \sup_{x}|\cale(x,\hetain)-\cale(x,\eta_j^*)|
    &\leq
    L_1
    \|\hetain-\eta^*_j \|,
    % \\&
    % \lesssim
    % \mathcal{O}_P\left(
    %     [{\log(n)}/{n}]^{\frac{1}{4}} 
    %     \right),
\end{align}
for any $i\in\calAj(\hGn)$, where $L_1\geq 0$ is a Lipschitz constant. The above bound implies that the convergence rates for estimating the exactly-specified experts and over-specified experts are of orders $\mathcal{O}_P([{\log(n)}/{n}]^{\frac{1}{2}})$ and $\mathcal{O}_P([{\log(n)}/{n}]^{\frac{1}{4}})$, respectively. Thus, it takes the exactly-specified experts a polynomial number $\mathcal{O}(\epsilon^{-2})$ of data points to achieve an approximation error of $\epsilon$, while the over-specified experts need a polynomial number $\mathcal{O}(\epsilon^{-4})$ of data to achieve the same error.

\subsubsection{Polynomial Experts}
\label{sec:polynomial}
We now investigate polynomial experts of the form $ \cale(x,(\alpha, \beta)) = (\alpha^{\top}x + \beta)^{p}$, where $p\in\mathbb{N}$. As mentioned in the example paragraph following the Definition~\ref{def:strong_identifiability}, the polynomial experts do not meet the strong identifiability condition due to the linear dependence among the derivative of the function $F(x;A,b,c,\eta)$. For example, when $p=1$, such linear dependence is exhibited via the interaction among the gating parameters (see the first PDE in equation~\eqref{eq:PDE})
% \begin{align}
%     \frac{\partial^2 F}{\partial A \partial c}=\frac{\partial^2 F}{\partial b \partial b^{\top}},
%     % {\partial b^2},
% \end{align}
and the interaction between the gating parameters and the expert parameters (see the last two PDEs in equation~\eqref{eq:PDE}).
% \begin{align}
%     \frac{\partial^2 F}{\partial A \partial \beta}=\frac{\partial^2 F}{ \partial b\partial \alpha},~
%     \frac{\partial^2 F}{\partial b \partial \beta}=\frac{\partial^2 F}{ \partial c\partial \alpha},
% \end{align}
% where $F(x;A,b,c,\alpha,\beta)=\sigma(x,A,b,c)(\alpha^{\top}x+\beta)$.
Those PDEs account for the non-strong identifiability of the polynomial experts. 
Notably, we will demonstrate in Theorem~\ref{thm:dtwor_loss_linear} that such parameter interactions lead to slow convergence rates for parameter estimation and expert estimation. Toward that goal, let us introduce the  Voronoi loss function 
\begin{align}
\label{eq:dtwor}
    % &
    \hspace{-0.5em}\ltwor:= 
    \sum_{j=1}^{\bn}
    \left|
    \sum_{i\in\calAj} \frac{1}{1+\exp(-c_i)}-\frac{1}{1+\exp(-c_j^*)}
    \right|
    \nonumber
 %   \\
  %  &
    +\sum_{j=1}^{\bn}
    \sum_{i\in\calAj}
    \Big[
    \|\Delta A_{ij} \|^r
    +\|\Delta b_{ij} \|^r
   % \nonumber
    \\
   +\|\Delta \alpha_{ij} \|^r  +|\Delta \beta_{ij} |^r
    \Big]
    +\sum_{j=\bn+1}^{\ns}
    \sum_{i\in\calAj}
    \big[
    \|\Delta A_{ij} \|^r
    +\|\Delta b_{ij} \|^r
    +|\Delta c_{ij} |^r
    +\|\Delta \alpha_{ij} \|^r %\nonumber
   % \\ &\hspace{5.5cm}
    +|\Delta \beta_{ij} |^r
    \big],
\end{align}
where we denote $\Delta \alpha_{ij}:=\alpha_i-\alpha^*_j$ and $\Delta\beta_{ij}:=\beta_i-\beta^*_j$.
\begin{theorem}
\label{thm:dtwor_loss_linear}
    Suppose that the expert function takes a polynomial form $\mathcal{E}(x,\alpha,\beta)=(\alpha^{\top}x+\beta)^p$, for some $p\in\bbN$.
    Then, for any $r\geq 1$,
    \begin{align*}
        \inf_{\tGn\in\calm_N(\Theta)}
        \sup_{G\in\calm_N(\Theta)\setminus \calm_{\ns-1}(\Theta)}
        \bbE_{f_G}
        [\mathcal{L}_{2,r}(\tGn,G)]
        \gtrsim
        \frac{1}{\sqrt{n}},
    \end{align*}
    where $\bbE_{f_G}$ indicates the expectation taken w.r.t. the product measure with $f^n_G$.
\end{theorem}
The proof of Theorem \ref{thm:dtwor_loss_linear} is in Appendix \ref{proof:dtwor_loss_linear}. Below we highlight some important implications of the above result.

\emph{(i) Parameter convergence rates:} The minimax lower bound in Theorem~\ref{thm:dtwor_loss_linear} implies that the convergence rates for estimating parameters $A^*_j,b^*_j,\alpha^*_j,\beta^*_j$ are slower than any polynomial rates $\mathcal{O}_P(n^{-1/2r})$ for any $r\geq 1$, potentially as slow as $\mathcal{O}_P(1/\log^{\tau}(n))$, for some constant $\tau>0$.

\emph{(ii) Expert convergence rates:} Following the parameter convergence rates and using the same arguments as in equation~\eqref{eq:expert-rates}, we deduce that the convergence rates for estimating experts could also be as slow as $\mathcal{O}_P(1/\log^{\tau}(n))$, for some constant $\tau>0$. Therefore, it may require the experts an exponential number of data $\mathcal{O}(\exp(\epsilon^{-1/\tau}))$ to obtain the approximation error $\epsilon$.

% \emph{(iii) Sample efficiency of sigmoid self-attention:} Given the above expert convergence rates, it might require the experts an exponential number of data $\mathcal{O}(\exp(\epsilon^{-1/\tau}))$ to obtain the approximation error $\epsilon$.

\textbf{Sample complexity comparison under the sparse regime:} According to the results in \cite{akbarian2024quadratic}, estimating the experts in the softmax self-attention share the same sample complexity as estimating those in the sigmoid version. In particular, strongly identifiable experts and polynomial experts need $\mathcal{O}(\epsilon^{-4})$ and $\mathcal{O}(\exp(\epsilon^{-1/\tau}))$ data points to achieve the approximation error $\epsilon$. Thus, we claim that the sigmoid self-attention is as sample-efficient as the softmax self-attention under the sparse regime of gating parameters. 
% \textcolor{blue}{Note that since it is unlikely that all the gating parameters vanish in practice} 
% \textcolor{red}{-> can you expand on this?}, the sparse regime is less popular than the dense regime, which will be studied in the next section.
Note that it is unlikely for all gating parameters to vanish in practice, as typical initialization schemes, training dynamics, or architectural designs (e.g., ReLU activations, residual connections) generally ensure that all gating values (mixture weights) are input-dependent. As a result, the model rarely enters an extremely sparse regime. That is, the sparse regime is less common than the dense regime, which will be studied in the next section.

% , except when $\alpha=1_d,\beta=0$:
% \begin{align*}
% &\Longrightarrow~
%     \frac{\partial^2 F}{\partial A \partial c}=\frac{\partial^2 F}{\partial b^2},~\\
%     \frac{\partial F}{\partial A}&=\frac{\partial^2 F}{\partial A \partial\alpha},~\\
%     \frac{\partial F}{\partial b}&=\frac{\partial^2 F}{\partial A \partial \beta}=\frac{\partial F}{ \partial b\partial \alpha},~\\
%     \frac{\partial F}{\partial c}&=\frac{\partial^2 F}{\partial b \partial \beta}=\frac{\partial F}{ \partial c\partial \alpha}
% \end{align*}

% \subsubsection{Monomial  \texorpdfstring{${\varphi}$}{}}
% Now consider that $\beta=0$, then $h(x,(\alpha,\beta))=\varphi(x,(\alpha,0))=\alpha^{\top}x$.
% % The interaction is the same as the non-linear case
% % that t
% The parameters of the gating function will interact via the PDE: 
% \begin{align}
%     \frac{\partial^2 F}{\partial A \partial c}=\frac{\partial^2 F}
%     {\partial b \partial b^{\top}},
%     % {\partial b^2},
% \end{align}
% And the first and the second order derivatives will interact via:
% \begin{align*}
% % &\Longrightarrow~
% %     \frac{\partial^2 F}{\partial A \partial c}=\frac{\partial^2 F}{\partial b^2},~\\
%     \alpha^{\top}
%     \frac{\partial F}{\partial A}&=\frac{\partial^2 F}{\partial A \partial\alpha},~\\
%     \alpha^{\top}\frac{\partial F}{\partial b}&
%     % =\frac{\partial^2 F}{\partial A \partial \beta}
%     =\frac{\partial F}{ \partial b\partial \alpha},~\\
%     \alpha^{\top}\frac{\partial F}{\partial c}&
%     % =\frac{\partial^2 F}{\partial b \partial \beta}
%     =\frac{\partial F}{ \partial c\partial \alpha}
% \end{align*}

\subsection{Dense Regime of Gating Parameters}
\label{sec:dense_regime}
We now turn our attention to the dense regime where we assume that there exists some over-specified gating parameter  different from zero, that is, $(\Asi,\bsi)\neq(0_{d\times d}, 0_d)$, for some $i\in[\ns]$.
As demonstrated in Section~\ref{sec:problem-setup}, the smallest distance between the regression function estimator $f_{\hGn}$ and the set of the regression functions $f_{\lG}$ goes to 0 where 
$\lG\in\lcalm_N(\Theta):=\argmin_{G\in\calm_N(\Theta)\setminus\calm_{\ns}(\Theta)}
\|
f_G-f_{\Gs}
\|_{L^2(\mu)}$.
Without loss of generality (WLOG), we assume that 
\begin{align*}
    \lG:=\sum_{i=1}^N
    \frac{1}{1+\exp(-\Bar{c}_i)}
    \delta_{(\Bar{A}_i,\Bar{b}_i,\Bar{\eta}_i)}.
\end{align*}
Similar to the sparse regime of gating parameters, we also establish an identifiability condition on the expert function to avoid any interaction among the parameters expressed via PDEs under the dense regime. Since the expert function is required to satisfy only a subset of the strong identifiability condition in Definition~\ref{def:strong_identifiability} in this case, we refer to the condition as \emph{weak identifiability}.
\begin{definition}[Weak identifiability]
\label{def:weak_identifiability}
    We call an expert function $x\mapsto\cale(x,\eta)$ \emph{weakly identifiable} if it is differentiable w.r.t its parameter $\eta$ for $\mu$-almost all $x$ and, for any positive integer $\ell$
    and any distinct parameters 
    $\left\{ (A_i, b_i,c_i,\eta_i) \right\}_{i=1}^{\ell}$, 
    the functions in the family
    \begin{align*}
        \hspace{-0.5em}\left\{  
        \frac{\partial
        % ^{|\tau_1|+|\tau_2|+|\tau_3|+|\tau_4|} 
        F}
        {\partial A^{\tau_1}\partial b^{\tau_2}\partial c^{\tau_3}\partial\eta^{\tau_4}}
    (x,A_i,b_i,c_i,\eta_i):
    i\in[\ell],
    \sum_{j=1}^4|\tau_j|
    =1
    % (\gamma_1,\gamma_2,\gamma_3)\in\Real^{d\times d}
        \right\}
    \end{align*}
    are linearly independent, for $\mu$-almost all $x$,
    where 
    % $(\gamma_1,\gamma_2,\gamma_3)\in\bbN^{d\times d}\times\bbN^d\times\bbN^q$ and 
    $(\tau_1,\tau_2,\tau_3,\tau_4)\in\bbN^{d\times d}\times\bbN^d\times\bbN\times\bbN^q$.
\end{definition}
\textbf{Examples.} Strongly identifiable experts meet the weak identifiability condition. For instance, it can be verified that the previously mentioned two-layer neural networks of the form $\mathcal{E}(x,(\alpha,\beta,\lambda))=\lambda\phi(\alpha^{\top}x+\beta)$, where $\phi$ is $\relu$ or $\gelu$ function, $\alpha\neq0_d$, and $\lambda\neq 0$ are weakly identifiable. Moreover, polynomial experts $\mathcal{E}(x,(\alpha,\beta))=(\alpha^{\top}x+\beta)^p$, for $p\in\mathbb{N}$, also satisfy the weak identifiability condition although they are not strongly identifiable. 

%Consider an expert network $ \cale(x, (\alpha, \beta)) = \varphi(\alpha^\top x + \beta) $. It can be verified that if $ \alpha \neq 0_d $ and the activation function $ \varphi $ is a ReLU, GELU, or a polynomial, then the expert $ \cale(x, (\alpha, \beta)) $ is weakly identifiable. Conversely, if $ \alpha = 0_d $, meaning the expert does not depend on the input, the weak identifiability condition is not satisfied, regardless of the choice of the activation function.

In Theorem~\ref{thm:dthree_loss} below, we provide convergence rates 
for weakly identifiable experts based on the Voronoi loss 
\begin{align*}
    \lthree :=
    \sum_{j=1}^N
    \sum_{i\in\calAj}
    \big[
    \|A_i-\bai\|
    +\|b_i-\bbi\|
    +|c_i-\bci|
    +\|\eta_i-\betai\|
    \big].
\end{align*}
\begin{theorem}
\label{thm:dthree_loss}
    If the function $x\mapsto\cale(x,\eta)$ is weakly identifiable, then the lower bound $ \inf_{\lG\in\lcalm_N(\Theta)}
        \Vert f_{G}-f_{\lG}\Vert_{L^2(\mu)}
        \gtrsim
        \lthree$ holds true for any mixing measure $G\in\calm_N(\Theta)$.
    % \begin{align*}
    %     \inf_{\lG\in\lcalm_N(\Theta)}
    %     \Vert f_{G}-f_{\lG}\Vert_{L^2(\mu)}
    %     \gtrsim
    %     \lthree.
    % \end{align*}
As a consequence,  
\begin{align*}
    \inf_{\lG\in\lcalm_N(\Theta)}
    \call_3(\hGn,\lG)=
        \mathcal{O}_P(\sqrt{\log(n)/n}).
\end{align*}
\end{theorem}
The proof of Theorem \ref{thm:dthree_loss} is in Appendix \ref{proof:dthree_loss}. Since the convergence rates for the parameter estimators $\hat{\eta}^n_i$ are of order $\mathcal{O}_P([\log(n)/n]^{\frac{1}{2}})$, the weakly identifiable expert estimators $\mathcal{E}(x,\hat{\eta}^n_i)$ also admit the same convergence rates, as indicated by the inequality~\eqref{eq:expert-rates}. Therefore, it takes those experts only a polynomial number of samples
% \textcolor{red}{“a polynomial number” -> +”of samples”} 
$\mathcal{O}(\epsilon^{-2})$ to achieve an approximation error of $\epsilon$.

\textbf{Sample complexity comparison under the dense regime:} Recall that it costs strongly identifiable experts and polynomial experts in the softmax self-attention $\mathcal{O}(\epsilon^{-4})$ and $\mathcal{O}(\exp(\epsilon^{-1/\tau}))$ data points to reach the approximation error $\epsilon$, respectively \cite{akbarian2024quadratic}. On the other hand, as those experts satisfy the weak identifiability condition, they need only a sample size of order $\mathcal{O}(\epsilon^{-2})$ to achieve the same error under the dense regime of sigmoid self-attention. For that reason, we claim that the sigmoid self-attention is more sample efficient than its softmax counterpart under the dense regime, which is more likely to occur in practice than the sparse regime.

\section{Numerical Experiments}
\label{sec:experiments}

\begin{figure*}[!t]
    \centering
    \subfloat[\textbf{MoE with ReLU Experts}\label{fig:relu-experts}]{
        \includegraphics[width=0.75\textwidth]{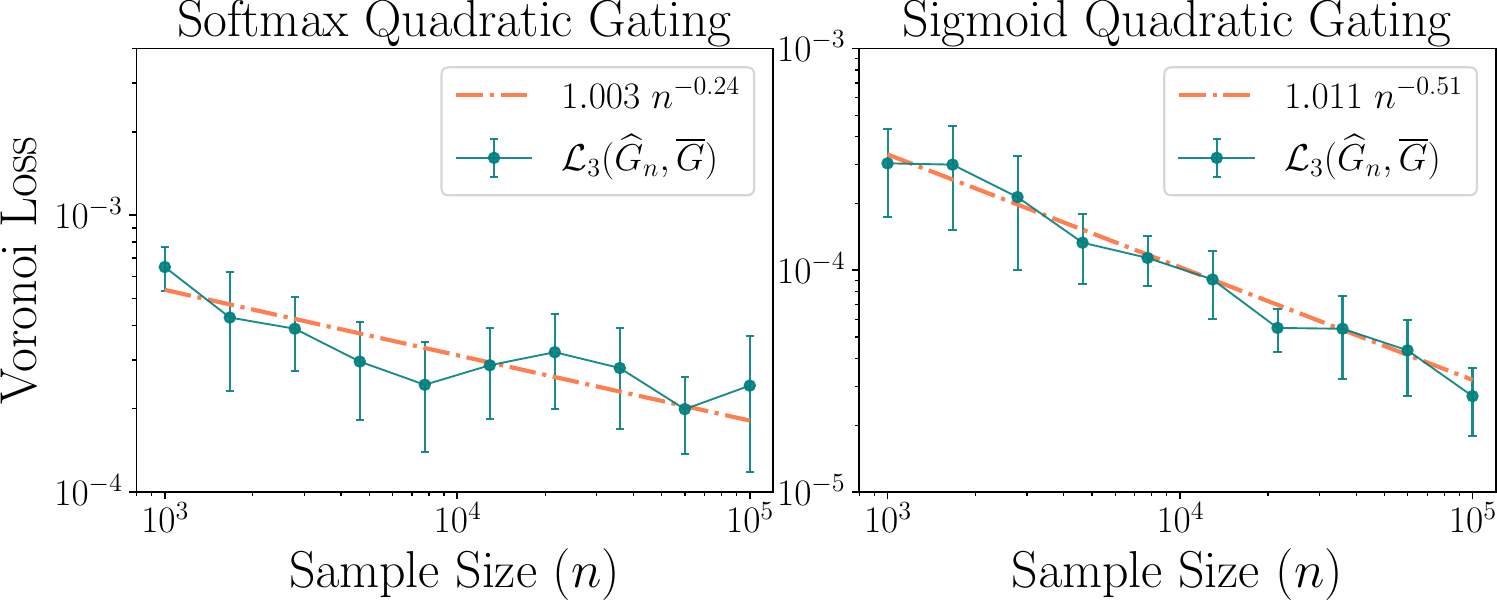}
    }
    % [0.5em]
    \vspace{0.5em}
    \subfloat[\textbf{MoE with Linear Experts}\label{fig:linear-experts}]{
        \includegraphics[width=0.75\textwidth]{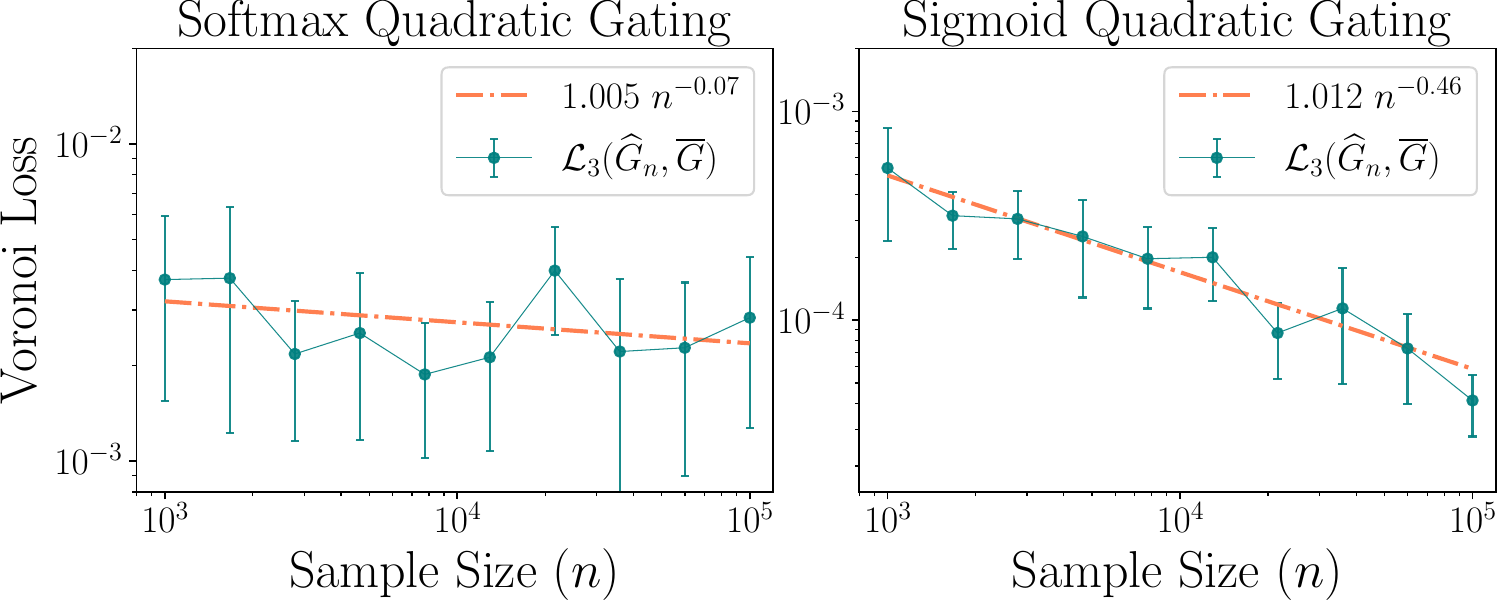}
    }
    \vspace{0.5em}
    \caption{Log-log plots of the convergence rates of Voronoi losses for softmax and sigmoid quadratic gating MoE models. 
    \ref{fig:relu-experts} Comparison between softmax quadratic gating and sigmoid quadratic gating MoE with ReLU experts. 
    \ref{fig:linear-experts} Comparison between softmax quadratic gating and sigmoid quadratic gating with linear experts. 
    Each plot illustrates the empirical Voronoi loss convergence rates, with solid lines representing the Voronoi losses and dash-dotted lines showing fitted trends.}
    \label{fig:sigmoid-vs-softmax}
    \vspace{-1em}
\end{figure*}

Below, we present some numerical experiments confirming  our theoretical findings that on the superior sample complexity of sigmoid attention versus softmax attention. % provided in Section~\ref{sec:sample-efficiency} from the MoE perspective.
% quadratic gating in MoE models in terms of parameter estimation rates. Additionally, we analyze the estimation rates of sigmoid gating under a well-specified setting for various expert configurations as discussed in Section~\ref{sec:sample-efficiency}, with detailed results provided in the Appendix.
% Next, we highlight the superior performance of sigmoid attention compared to standard softmax attention in a language modeling task.
%\subsection{Numerical Experiments}
% In the subsequent experiments, we compare sigmoid attention and softmax attention by viewing them as MoE models. 
In particular, we capture the empirical convergence rates of parameter estimation in sigmoid quadratic gating and softmax quadratic gating MoE models with two expert configurations: expert networks with ReLU activation (ReLU experts) and those with linear activation (linear experts). 
%This comparison highlights the relative effectiveness of the two gating mechanisms across different expert setups.

\textbf{Setup.} We generated ynthetic data from a softmax quadratic gating MoE model. 
% For each experiment, we generate synthetic data using the corresponding MoE model.
Both sigmoid quadratic gating MoE and softmax quadratic gating MoE models were fitted, with varying sample sizes. The empirical convergence rates are then evaluated for each model across the two expert configurations (ReLU experts and linear experts), providing insights into sigmoid quadratic gating performance. 
The details about the values of the ground-truth parameters and the training procedure are presented in Appendix~\ref{appendix:experimental-details}.

\textbf{Results.}
Figure~\ref{fig:sigmoid-vs-softmax} shows the empirical convergence rates of Voronoi loss for sigmoid and softmax quadratic gating MoE models, with error bars representing three standard deviations to account for variability across runs. In Figure~\ref{fig:relu-experts} for MoE models with ReLU experts, we observe that the softmax quadratic gating MoE converges at the rate of order $\mathcal{O}(n^{-0.24})$, whereas the sigmoid version achieves a significantly faster rate of order $\mathcal{O}(n^{-0.51})$. Similarly, in Figure~\ref{fig:linear-experts} for MoE models with linear experts, the sigmoid quadratic gating MoE attains a convergence rate of order $\mathcal{O}(n^{-0.46})$, while the softmax counterpart exhibits a significantly slower convergence rate of order $\mathcal{O}(n^{-0.07})$. These results empirically validate our theoretical findings that sigmoid quadratic gating MoE admits faster parameter estimation rates across both the configurations of ReLU experts and linear experts. %In other words, we can conclude that the sigmoid self-attention is still more sample-efficient than the softmax self-attention on the empirical side.

\section{Conclusion}
\label{sec:conclusion}
In this paper, we establish a mathematical connection between the self-attention mechanism in the Transformer architecture and the Mixture-of-Experts (MoE) model.
% Next, we carry out the convergence analysis of parameter and expert estimation under the sigmoid gating MoE models with fully quadratic affinity score function, for the sparse regime and the dense regime.
Withing the MoE framework, we investigate the sample complexity of the sigmoid self-attention by conducting a convergence analysis of parameter and expert estimation under the sigmoid gating MoE models with fully and partially quadratic affinity score functions in both the sparse and dense regimes for the gating parameters.
% Our theories indicate that the sigmoid self-attention is more sample efficient than its softmax counterpart under the dense regime, which is more likely to occur in practice than the sparse regime.
Our results show that sigmoid self-attention has a higher sample complexity than the softmax version in the more common dense regime, a finding further confirmed by our simulations. %which is more common in practice than the sparse regime.
% Through empirical evaluation, we show the potential of sigmoid attention as a competitive and efficient alternative to softmax attention, particularly in resource-constrained settings.
%Furthermore, empirical results also demonstrate the potential of sigmoid self-attention as a competitive and efficient alternative to softmax self-attention, especially in resource-constrained settings.
% However, in this work we just analysis for the  single head self attention situation. 
% For future work, we plan to extend the analysis to the setting of mutli-head attention in the future.
A limitation of our work is that we consider only a single attention head rather than a more popular multi-head attention mechanism \cite{vaswani2017attention}. However, we believe this problem can be overcome by formulating the multi-head attention as a hierarchical MoE \cite{Jordan-1994,nguyen2024hmoe}, which we leave to future work.

%%%%%%%%%%%%%%%%%%%%%%%%%%%%%%%%%%%%%%%%%%%%%%%%%%%%%%%%%%%%

%%%%%%%%%%%%%%%%%%%%%%%%%%%%%%%%%%%%%%%%%%%%%%%%%%%%%%%%%%%%

% \bibliography{references}
% \bibliographystyle{abbrv}

% \newpage
\vspace{1cm}
\appendix

\begin{center}
{}\textbf{\Large{Supplement to
``Sigmoid Self-Attention has Lower Sample Complexity than Softmax Self-Attention: 
\\ \vspace{2mm}
A Mixture-of-Experts Perspective''}}
\end{center}

The supplementary material is structured as follows: Appendix \ref{appsec:related-works} provides a discussion of related works on the self-attention mechanism and Mixture of Experts (MoE) models. Appendix \ref{appsec:proof} contains detailed proofs of the results presented in Sections \ref{sec:problem-setup} and \ref{sec:sample-efficiency}. In Appendix \ref{appsec:additional-results}, we investigate the sample complexity of sigmoid self-attention under the partially quadratic score for the gating mechanism of the MoE model, accompanied by proofs of the relevant results. Finally, Appendix \ref{appsec:additional-experiments} presents additional experimental results, while Appendix \ref{appendix:experimental-details} provides the experimental details.  

% Appendix organization:
% In Appendix \ref{appsec:related-works},
% we discuss related works to the sigmoid self-attention model.
% In Appendix \ref{appsec:proof},
% we present the proofs of the results in section \ref{sec:problem-setup} and \ref{sec:sample-efficiency}.
% In Appendix \ref{appsec:additional-results},
% we establish the sample efficiency under the partial quadratic score of the gating for MoE model, with proof of the results.
% In Appendix \ref{appsec:additional-experiments},
% we discuss additional experiments.

% In Appendix \ref{appendix:analysis_quadratic_mono},
% we establish the expert estimation rates under the quadratic monomial gating MoE model.
% In Appendix \ref{appendix:proof_function_convergence},
% we prove Theorem \ref{thm:function-convergence-specified}.
% In Appendix \ref{appendix:proof_quadratic_poly}, we prove the parameter estimation rates for quadratic polynomial gates.
% In Appendix \ref{appendix:proof_quadratic_mono}, we prove the parameter estimation rates for quadratic mononomial gates.  

\section{Related Works}
\label{appsec:related-works}

Attention mechanisms have become a cornerstone of the Transformer architecture. In its standard formulation, attention weights are calculated as the softmax of the dot products between keys and queries \cite{vaswani2017attention}. More recently, Ramapuram et al. \cite{ramapuram2024sigmoidattention} established that Transformers with sigmoid-based attention are universal function approximators and exhibit enhanced regularity compared to those using softmax attention. 
Beyond their fundamental role in Transformers, attention mechanisms have been closely linked to Mixture of Experts (MoE) architectures. Csordas et al. \cite{csordas2023switchhead} introduces SwitchHead, an MoE-based attention method designed to reduce both the computation and storage requirements for attention matrices. Wu et al.  
\cite{wu2024multi} proposes Multi-Head MoE (MH-MoE), where each parallel layer contains a set of \(N\) experts decoupled from the head in multi-head self-attention, focusing on scalability and architectural improvements. Jin et al.  
\cite{jin2024moh} demonstrates that multi-head attention can be written in summation form, motivating the Mixture-of-Head (MoH) architecture, which treats each attention head as an MoE expert. This design boosts inference efficiency without compromising accuracy or increasing parameter counts. 
Le et al.
%\cite{le2024mixture} demonstrated that attention blocks in pre-trained models, such as Vision Transformers, inherently encode a specialized MoE structure characterized by linear experts and quadratic gating score functions. 
Furthermore, Akbarian et al. \cite{akbarian2024quadratic} examined the connection between softmax self-attention and softmax quadratic gating MoE, showing that MoE models with softmax quadratic gating outperform their counterparts using traditional softmax linear gating \cite{nguyen2024squares}.

% \cite{csordas2023switchhead}
% presents a novel MoE-based attention method, SwitchHead, whose mechanism allows to reduce the number of attention matrices that need to be computed and stored.
% \cite{wu2024multi} introduces Multi-Head MoE (MH-MoE), where each parallel layer contains a set of $N$ experts that is decoupled from the head in the multi-head self-attention layer. Its primary focus is on enhancing the design and scalability of MoE models through this multi-head approach. 
% \cite{jin2024moh} show that multi-head attention can be expressed in the
% summation form. And  propose Mixture-of-Head attention (MoH), a new architecture that treats attention heads as experts in the Mixture-of-Experts (MoE) mechanism. It enhancing inference efficiency without compromising accuracy or increasing the number of parameters.

Another line of research has also explored convergence rates for expert estimation in Gaussian Mixture of Experts models. 
% In \cite{mendes2011convergence}, the authors studied Mixture of Experts models whose experts come from a one-parameter exponential family, and they established the convergence rate of the maximum likelihood estimator to densities within that family.
First, Mendes et al. \cite{mendes2011convergence} examined maximum likelihood estimation for Mixture of Experts models with polynomial regression experts. They analyzed how quickly the estimated density converges to the true density under the Kullback–Leibler (KL) divergence and offered insights on selecting the appropriate number of experts. Ho et al.
\cite{ho2022gaussian} derived convergence rates for parameter estimation by utilizing a connection between the algebraic independence of expert functions and model parameters, which they formulated through a class of partial differential equations (PDEs). Building on this foundation, Nguyen et al. \cite{nguyen2023demystifying} investigated Gaussian MoE models with softmax gating, uncovering that expert estimation rates are influenced by the solvability of a system of polynomial equations resulting from the interplay between gating and expert parameters.
% \cite{yan2024understanding} conducted the convergence analysis of parameter estimation in the contaminated Gaussian mixture of experts, i.e. a  pre-trained Gaussian expert are contaminated with prompt Gaussian expert by an unknown mixing proportion , via the algebraic interaction among parameters of the pre-trained model and the prompt.
Yan et al. \cite{yan2025contaminated} analyzed the convergence of parameter estimation in a contaminated Gaussian mixture of experts, where a pre-trained Gaussian expert is mixed with a prompt Gaussian expert by an unknown proportion. Their approach leverages the algebraic interaction between the pre-trained model’s parameters and the prompt.
For Gaussian MoE models incorporating Top-K sparse softmax gates, Nguyen et al. \cite{nguyen2024statistical} analyzed convergence rates using novel loss functions specifically designed for sparse gating mechanisms. 

%{\color{red} Fanqi: please help discuss \cite{mendes2011convergence}}
%In addition, \cite{nguyen2024squares} studied deterministic MoE models with softmax linear gating. Lastly, \cite{nguyen2024sigmoid} focused on Gaussian MoE models with sigmoid gating, demonstrating that sigmoid gating yields higher sample efficiency compared to softmax gating, thus offering superior parameter estimation properties.

\section{Proof of the Results in Sections \ref{sec:problem-setup} and \ref{sec:sample-efficiency}}
\label{appsec:proof}

\subsection{Proof of Proposition \ref{thm:function-convergence-specified} }
\label{appendix:proof_function_convergence}

% \begin{theorem}
% Under the Regime \text{1} and with the least squares estimator 
% $\hGn$ defined in equation \eqref{eq:lse}, 
% the regression estimator $f_{\hGn}$ admits the following rate of convergence to $f_{\Gs}$:
%     \begin{align*}
%         \Vert f_{\hGn}-f_{\Gs}\Vert_{L^2(\mu)}=
%         \mathcal{O}_P\left(
%         \sqrt{
%         \frac{\log(n)}{n} 
%         }
%         \right).
%     \end{align*}
% \end{theorem}

First we will introduce some necessary notations used throughout this appendix.
We let $\calf_N(\Theta):=\{f_G(x):G\in\calm_N(\Theta) \}$ 
to be the set of all regression functions in $\calm_N(\Theta)$.
Then we consider the intersection  between $L^2(\mu)$-ball centered around the regression function $f_{\Gs}(x)$ with radius $\delta>0$ and the set $\calf_N(\Theta)$ and define as
\begin{align*}
    \calf_N(\Theta,\delta):=
    \left\{
    f\in\calf_N(\Theta):
    \|f-f_{\Gs} \|_{L^2(\mu)}\leq \delta
    \right\}.
\end{align*}
We assess the complexity of the above class using the bracketing entropy integral in \cite{Vandegeer-2000},
\begin{align}
\label{def:bracketing-entropy-integral}
    \calj_B(\delta,\calf_N(\Theta,\delta)):=
    \int^{\delta}_{\delta^2/2^{13}}
    H_B^{1/2}(t,\calf_N(\Theta,t),\|\cdot\|_{L^2(\mu)})dt\vee\delta,
\end{align}
where $ H_B(t,\calf_N(\Theta,t),\|\cdot\|_{L^2(\mu)})$ represents the bracketing entropy \cite{Vandegeer-2000} of $\calf_N(\Theta,t)$ under the ${L^2(\mu)}$-norm, and $t\vee\delta:=\max\{t,\delta\}$.
Using similar arguments as in Theorems 7.4 and 9.2 of \cite{Vandegeer-2000}, with the notation adapted to our context, we derive the following lemma:
\begin{lemma}
\label{app_lemma:prop_bound}
Let $\Psi(\delta)\geq\calj_B(\delta,,\calf_N(\Theta,\delta))$
such that $\Psi(\delta)/\delta^2$ 
is a non-increasing function of $\delta$.
Then, for some universal constant $c$ and for some sequence $(\delta_n)$ that satisfied $\sqrt{n}\delta^2_n\geq c\Psi(\delta_n)$,
the following holds for any $\delta\geq\delta_n$:
    \begin{align*}
        \mathbb{P}
        \left(
        \| 
f_{\hGn}-f_{\Gs}
        \|_{L^2(\mu)>\delta} \right)
        \le
        c\exp\left(-\frac{n\delta^2}{c^2} \right).
    \end{align*}
\end{lemma}

\textbf{Outline of Proof.} To prove the result, it suffices to establish the following bound for the bracketing entropy $ H_B(\cdot) $ of the function class $ \mathcal{F}_N(\Theta) $:  
\begin{equation} \label{eq:entropy_bound}
H_B(\epsilon, \mathcal{F}_N(\Theta), \|\cdot\|_{L^2(\mu)}) \lesssim \log(1/\epsilon), \quad \forall    \epsilon \in (0, 1/2].
\end{equation}

Using equation~\eqref{eq:entropy_bound}, the integral term $ \mathcal{J}_B(\delta, \mathcal{F}_N(\Theta, \delta)) $ can be bounded as:  
\begin{align}
\label{eq:entropy_integral_bound}
\mathcal{J}_B(\delta, \mathcal{F}_N(\Theta, \delta)) = \int_{\delta^{2}/2^{13}}^\delta H_B^{1/2}(t, \mathcal{F}_N(\Theta, t), \|\cdot\|_{L^2(\mu)})    dt\vee\delta \lesssim \int_{\delta^{2}/2^{13}}^\delta \sqrt{\log(1/t)}    dt\vee\delta.
\end{align}
This integral evaluates to:  
\begin{align*}
\mathcal{J}_B(\delta, \mathcal{F}_N(\Theta, \delta)) \lesssim \delta \cdot \log(1/\delta)^{1/2}.
\end{align*}

To ensure $ \Psi(\delta) \geq \mathcal{J}_B(\delta, \mathcal{F}_N(\Theta, \delta)) $, we define $ \Psi(\delta) := \delta \cdot \log(1/\delta)^{1/2} $, which satisfies the condition that $ \Psi(\delta)/\delta^2 $ is a non-increasing function.

\textbf{Choice of Sequence $ (\delta_n) $.} Set $ \delta_n := \sqrt{\log(n)/n} $. By construction, this sequence satisfies $ \sqrt{n}\delta_n^2 \geq c \Psi(\delta_n) $ for some universal constant $ c $. Substituting $ \delta = \delta_n $ into Lemma \ref{app_lemma:prop_bound} yields the desired probability bound.

\textbf{Conclusion.} By applying Lemma \ref{app_lemma:prop_bound}, we reduce the problem to verifying the entropy bound \eqref{eq:entropy_bound}, which will be established now:

% \textbf{Proof of inequality \eqref{eq:entropy_bound}.}
\subsection*{Proof of Inequality \eqref{eq:entropy_bound}:}
\begin{proof}
To prove inequality \eqref{eq:entropy_bound}, 
we begin by noting that the expert functions are bounded, implying  
\begin{align*}
|f_G(x)| \leq M \quad \text{for almost every } x,
\end{align*}
where $ M > 0 $ is a constant. 

\textbf{Step 1: Covering the Function Class.}  
Let $ \tau \leq \epsilon $, and consider a $ \tau $-cover $ \{\zeta_1, \dots, \zeta_{V} \} $ of the function class $ \mathcal{F}_N(\Theta) $ under the $ L^2(\mu) $-norm. The covering number $ V := V(\tau, \mathcal{F}_N(\Theta), \|\cdot\|_{L^\infty}) $ represents the $ \tau $-covering number of the metric space $ \mathcal{F}_N(\Theta) $ with the $ L^\infty $-norm.  

We construct brackets of the form $ [L_i(x), U_i(x)] $ for $ i \in [V] $, where  
\begin{align*}
L_i(x) := \max\{\zeta_i(x) - \tau, 0\}, \quad U_i(x) := \max\{\zeta_i(x) + \tau, M\}.
\end{align*}

\textbf{Step 2: Properties of the Brackets.}  
It follows that for all $ i \in [V] $:  
\begin{itemize}
    \item $ \mathcal{F}_N(\Theta) \subseteq \bigcup_{i=1}^V [L_i(x), U_i(x)] $,
    \item $ U_i(x) - L_i(x) \leq \min\{2\tau, M\} $ .
\end{itemize}

Thus, the bracket size satisfies  
\begin{align*}
\|U_i - L_i\|_{L^2(\mu)}^2 = \int (U_i - L_i)^2    d\mu(x) \leq \int 4\tau^2    d\mu(x) = 4\tau^2.
\end{align*}  
This implies  
\begin{align*}
\|U_i - L_i\|_{L^2(\mu)} \leq 2\tau.
\end{align*}

\textbf{Step 3: Bracketing Entropy.}  
From the definition of bracketing entropy, we have  
\begin{align*}
H_B(2\tau, \mathcal{F}_N(\Theta), \|\cdot\|_{L^2(\mu)}) \leq \log V = \log V(\tau, \mathcal{F}_N(\Theta), \|\cdot\|_{L^\infty}),
\end{align*}
leading to the bound:  
\begin{equation} \label{eq:bracketing_entropy}
H_B(2\tau, \mathcal{F}_N(\Theta), \|\cdot\|_{L^2(\mu)}) \lesssim \log V.
\end{equation}

\vspace{0.5em}
\textbf{Step 4: Covering Compact Sets.}  
Next, consider the parameter space $ \Theta $, which we decompose into:  
\begin{align*}
\Delta := \{(A, b,c) \in \mathbb{R}^{d\times d}\times\mathbb{R}^d \times \mathbb{R} : (A, b,c, \eta) \in \Theta\}, \quad \Omega := \{\eta \in \mathbb{R}^q : (A, b,c, \eta) \in \Theta\}.
\end{align*}
Both $ \Delta $ and $ \Omega $ are compact sets. For any $ \tau > 0 $, there exist $ \tau $-covers $ \Delta_\tau $ and $ \Omega_\tau $ such that:  
\begin{align*}
|\Delta_\tau| \lesssim \mathcal{O}_P(\tau^{-(d^2+d+1)N}), \quad |\Omega_\tau| \lesssim \mathcal{O}_P(\tau^{-qN}).
\end{align*}

\textbf{Step 5: Difference of the Functions.}
% Now we let $\check{\eta}_i\in\Omega_{\tau}$ such that $\check{\eta}_i$ is the closest to $\eta_i$ in that set, 
% and $(\check{A}_i, \check{b}_i, \check{c}_i)\in\Delta_{\tau}$ is the closest to $({A}_i, {b}_i, {c}_i)$ in that set 
Define $\check{\eta}_i \in \Omega_{\tau}$ such that  
\begin{align*}
\check{\eta}_i = \arg\min_{\eta \in \Omega_{\tau}} \|\eta - \eta_i\|,  
\end{align*}  
and $(\check{A}_i, \check{b}_i, \check{c}_i) \in \Delta_{\tau}$ such that  
\begin{align*}
(\check{A}_i, \check{b}_i, \check{c}_i) = \arg\min_{(A, b, c) \in \Delta_{\tau}} \|(A, b, c) - (A_i, b_i, c_i)\|.  
\end{align*}
Now we consider mixing measures as
\begin{align*}
    \widetilde{G}=\sum_{i=1}^{N}\frac{1}{1+\exp(-c_i)}\delta_{(A_i,b_i,\check{\eta}_i)},~~
    \check{G}=\sum_{i=1}^{N}\frac{1}{1+\exp(-\check{c})}\delta_{(\check{A}_i,\check{b}_i,\check{\eta}_i)}
\end{align*}
given $ {G}=\sum_{i=1}^{N}\frac{1}{1+\exp(-c_i)}\delta_{(A_i,b_i,{\eta}_i)}\in\calm_N(\Theta)$ as a mixing measure.
The above formulations imply that
\begin{align*}
    \|f_G-f_{\widetilde{G}} \|_{L^{\infty}}
    &=
    \sup_{x\in\mathcal{X}}\left| 
    \sum_{i=1}^N    
    \frac{1}{1+\exp
    \left(
    -x^{\top}A_ix-(b_i)^{\top}x-c_i
    \right)
    }
    \cdot
    \left[ 
    \cale
    \left(
    x,\eta_i
    \right)
    -
    \cale
    \left(
    x,\check{\eta}_i
    \right)
    \right]
    \right|
    \\&
    \leq 
    \sum_{i=1}^N 
    \sup_{x\in\mathcal{X}}   
    \frac{1}{1+\exp
    \left(
    -x^{\top}A_ix-(b_i)^{\top}x-c_i
    \right)
    }
    \cdot
    \left| 
    \cale
    \left(
    x,\eta_i
    \right)
    -
    \cale
    \left(
    x,\check{\eta}_i
    \right)
    \right|
    \\&
    \leq 
    \sum_{i=1}^N 
    \sup_{x\in\mathcal{X}}   
    \left| 
    \cale
    \left(
    x,\eta_i
    \right)
    -
    \cale
    \left(
    x,\check{\eta}_i
    \right)
    \right|
    \\&
    \leq 
    \sum_{i=1}^N 
    \sup_{x\in\mathcal{X}}   
    L_1
    \cdot
    \left\| 
    \eta_i-\check{\eta}_i
    \right\|    
    \\&
    \leq
    NL_1\tau\lesssim\tau.
\end{align*}
The second inequality holds because the sigmoid weight is bounded by one, while the third inequality stems from the expert function $ \cale(x, \cdot) $ being Lipschitz continuous with a Lipschitz constant $ L_1 > 0 $. Then we will have 
\begin{align*}
    &\|f_{{\tG}}-f_{\cG} \|_{L^{\infty}}\\
    &=
    \sup_{x\in\mathcal{X}}\left| 
    \sum_{i=1}^N  
        \left[ 
        \frac{1}{1+\exp
    \left(
    -x^{\top}A_ix-(b_i)^{\top}x-c_i
    \right)
    }-
    \frac{1}{1+\exp
    \left(
    -x^{\top}\cAi x-(\cbi)^{\top}x-\cci
    \right)
    }
        \right]
    \cdot 
    \cale
    \left(
    x,\check{\eta}_i
    \right)
    \right|
    \\&
    \leq 
    \sum_{i=1}^N 
    \sup_{x\in\mathcal{X}}   
   \left[ 
        \frac{1}{1+\exp
    \left(
    -x^{\top}A_ix-(b_i)^{\top}x-c_i
    \right)
    }-
    \frac{1}{1+\exp
    \left(
    -x^{\top}\cAi x-(\cbi)^{\top}x-\cci
    \right)
    }
        \right]
    \cdot 
    \left| 
    \cale
    \left(
    x,\check{\eta}_i
    \right)
    \right| 
    \\&
    \leq 
    \sum_{i=1}^N 
    \sup_{x\in\mathcal{X}}   
    M'
   \left[ 
        \frac{1}{1+\exp
    \left(
    -x^{\top}A_ix-(b_i)^{\top}x-c_i
    \right)
    }-
    \frac{1}{1+\exp
    \left(
    -x^{\top}\cAi x-(\cbi)^{\top}x-\cci
    \right)
    }
    \right]
    \\&
    \leq 
    \sum_{i=1}^N 
    \sup_{x\in\mathcal{X}}   
    M'L_2
    \left(
    \|A_i-\cAi \|\cdot \|x \|^2+
    \|b_i-\cbi \|\cdot \|x \|+
    |c_i-\cci|
    \right)
    \\&
    \leq
    NM'(\tau B^2+\tau B+\tau)\lesssim\tau
\end{align*}
The second inequality holds because the expert function is bounded, i.e., $ |\cale(x, \check{\eta}_i)| \leq M' $. The third inequality follows from the sigmoid function being Lipschitz continuous with a Lipschitz constant $ L_2 > 0 $, and the fourth inequality arises due to the boundedness of the input space. Then we know that from the triangle inequality
\begin{align*}
    \|f_{{G}}-f_{\cG} \|_{L^{\infty}}
    \leq
    \|f_{{G}}-f_{\tG} \|_{L^{\infty}}
    +
    \|f_{{\tG}}-f_{\cG} \|_{L^{\infty}}
    \leq\tau.
\end{align*}

\textbf{Step 6: Conclusion.}  
Now we recall the definition of the covering number, we will have
\begin{align*}
    N(\tau,\calf_N(\Theta),\|\cdot\|_{L^{\infty}})
    \leq
    |\Delta_{\tau}|\times|\Omega_{\tau}|
    \leq
    \mathcal{O}_P(n^{-(d^2+d+1)N})\times
    \mathcal{O}_P(n^{-qN})
    \leq
    \mathcal{O}_P(n^{-(d^2+d+1+q)N}).
\end{align*}

Recall the equation \eqref{eq:bracketing_entropy}
, the bracketing entropy of $ \mathcal{F}_N(\Theta) $ under $ L^2(\mu) $ is bounded by $ \log (1/\epsilon) $ by setting $\tau=\epsilon/2$, completing the proof.

\end{proof}

% \section{Proof for the Quadratic Polynomial Theorems}
% \label{appendix:proof_quadratic_poly}
\subsection{Proof of Theorem \ref{thm:done_loss}}
\label{proof:done_loss}

    % If $h(x,\eta)$ is strong identifiable, then
    % \begin{align*}
    %     \Vert f_{G}-f_{\Gs}\Vert_{L^2(\mu)}
    %     \gtrsim
    %     \lone
    % \end{align*}
    % for any $G\in\calm_N(\Theta)$.

\begin{proof}
In order to prove $\Vert f_{G}-f_{\Gs}\Vert_{L^2(\mu)}
        \gtrsim
        \lone$
for any $G\in\calm_N(\Theta)$, 
it is sufficient to show that 
\begin{align}
\label{pfeq:done_aim}
    \inf_{G\in\calm_N(\Theta)}
    \frac{\Vert f_{G}-f_{\Gs}\Vert_{L^2(\mu)}}{\lone}
    >0.
\end{align}
To prove the above inequality, we consider two cases for the denominator $\lone$: either it lies within a ball $B(0, \varepsilon)$ where the loss is sufficiently small, or it falls outside this region, where $\lone$ will not vanish.

\textbf{Local part:}
At first, we focus on that 
\begin{align}
\label{pfeq:done_local}
    \lim_{\varepsilon\to0}
    \inf_{G\in\calm_N(\Theta):\lone\leq\varepsilon}
    \frac{\Vert f_{G}-f_{\Gs}\Vert_{L^2(\mu)}}{\lone}
    >0.
\end{align}
Assume, for contradiction, that the above claim does not hold. Then there exists a sequence of mixing measures
$G_n=\sum_{i=1}^{\ns}\frac{1}{1+\exp(-c_i^n)}\delta_{(A_i^n,b_i^n,\eta_i^n)}$ in
$\calm_N(\Theta)$ such that as $n\to\infty$, we get
\begin{align}
    \begin{cases}
        \call_{1n}:=\call_1(G_n,\Gs)\to 0,\\
        \|f_{G_n}-f_{\Gs} \|_{L^2(\mu)}/\call_{1n}\to0.
    \end{cases}
\end{align}
Let us recall that
    \begin{align*}
    \call_{1n}
    &
    :=\sum_{j=1}^{\bn}
    \left|
    \sum_{i\in\calAj} \frac{1}{1+\exp(-c^n_i)}-\frac{1}{1+\exp(-c_j^*)}
    \right|
    % \nonumber\\
    +\sum_{j=1}^{\bn}
    \sum_{i\in\calAj}
    \left[
    \|\Delta A^n_{ij} \|^2
    +\|\Delta b^n_{ij} \|^2
    +\|\Delta \eta^n_{ij} \|^2
    \right]
    % \nonumber
    \\
    &+\sum_{j=\bn+1}^{\ns}
    \sum_{i\in\calAj}
    \left[
    \|\Delta A^n_{ij} \|
    +\|\Delta b^n_{ij} \|
    +|\Delta c^n_{ij} |
    +\|\Delta \eta^n_{ij} \|
    \right],
\end{align*}
where $\Delta A^n_{ij}:=A^n_i-A^*_j$, 
$\Delta b^n_{ij}:=b^n_i-b^*_j$, 
$\Delta c^n_{ij}:=c^n_i-c^*_j$, 
$\Delta \eta^n_{ij}:=\eta^n_i-\eta^*_j$.
Since $\call_{1n}\to0$ as $n\to0$, there are two different situation for parameter convergence:
\begin{itemize}
    \item For $j=1,\cdots,\bn$, parameters are over-fitted: $\sum_{i\in\calAj} \frac{1}{1+\exp(-c^n_i)}\rightarrow\frac{1}{1+\exp(-c_j^*)}$
    and $(A^n_i,b^n_i,\eta^n_i)\rightarrow (A^*_j,b^*_j,\eta^*_j)$, $\forall i\in\calAj$;
    \item For $j=\bn+1,\cdots,\ns$, parameters are exact-fitted: 
    $(A^n_i,b^n_i,c^n_i,\eta^n_i)\rightarrow (A^*_j,b^*_j,c^*_j,\eta^*_j), \forall i\in\calAj$.
    
\end{itemize}
% \begin{align*}
%     f_{G_n}(x)=\sum_{i=1}^{k}
%     \frac{1}{1+
%     \exp(-x^{\top}A^n_ix-(b^n_i)^{\top}x-c_i^n)
%     }\cdot
%     h(x,\eta^n_i)
% \end{align*}
% \begin{align*}
%     f_{\Gs}(x):=\sum_{i=1}^{\ns}
%     \frac{1}{1+
%     \exp(-x^{\top}A^*_ix-(b^*_i)^{\top}x-c_i^*)
%     }\cdot
%     h(x,\ei)
% \end{align*}
\textbf{Step 1 - Taylor expansion:}
In this step, we decompose the term $f_{G_n}(x) - f_{\Gs}(x)$ using a Taylor expansion. First, let us denote
\begin{align}
\label{eq:done_taylor_expension}
    &f_{G_n}(x)-f_{\Gs}(x)
    % \nonumber
    \\&=
    \sum_{j=1}^{\bn}
    \left[
    \sum_{i\in\calAj}
    \frac{1}{1+
    \exp(-x^{\top}A^n_ix-(b^n_i)^{\top}x-c_i^n)
    }\cdot
    \cale(x,\eta^n_i)
    -
    \frac{1}{1+
    \exp(-x^{\top}A^*_jx-(b^*_j)^{\top}x-c_j^*)
    }\cdot
    \cale(x,\ej)
    \right]
    \nonumber
    \\&
    +
    \sum_{j=\bn+1}^{\ns}
    \left[
    \sum_{i\in\calAj}
    \frac{1}{1+
    \exp(-x^{\top}A^n_ix-(b^n_i)^{\top}x-c_i^n)
    }\cdot
    \cale(x,\eta^n_i)
    -
    \frac{1}{1+
    \exp(-x^{\top}A^*_jx-(b^*_j)^{\top}x-c_j^*)
    }\cdot
    \cale(x,\ej)
    \right]
    \nonumber
    % \\&=
    % \sum_{j=1}^{\bn}
    % \left[
    % \sum_{i\in\calAj}
    % \frac{1}{1+
    % \exp(-x^{\top}A^n_ix-(b^n_i)^{\top}x-c_i^n)
    % }\cdot
    % h(x,\eta^n_i)
    % -
    % \frac{1}{1+
    % \exp(-c_i^*)
    % }\cdot
    % h(x,\ei)
    % \right]
    % \\&
    % +
    % \sum_{j=\bn+1}^{\ns}
    % \left[
    % % \sum_{i\in\calAj}
    % \frac{1}{1+
    % \exp(-x^{\top}A^n_ix-(b^n_i)^{\top}x-c_i^n)
    % }\cdot
    % h(x,\eta^n_i)
    % -
    % \frac{1}{1+
    % \exp(-x^{\top}A^*_ix-(b^*_i)^{\top}x-c_i^*)
    % }\cdot
    % h(x,\ei)
    % \right]
    % \\&=
    % \sum_{j=1}^{\bn}
    % \left[
    % \sum_{i\in\calAj}
    % \frac{1}{1+
    % \exp(-x^{\top}A^n_ix-(b^n_i)^{\top}x-c_i^n)
    % }\cdot
    % h(x,\eta^n_i)
    % -
    % \sum_{i\in\calAj}
    % \frac{1}{1+
    % \exp(-c_i^n)
    % }\cdot
    % h(x,\ej)
    % \right]
    % \\&
    % +
    % \sum_{j=1}^{\bn}
    % \left[
    % \sum_{i\in\calAj}
    % \frac{1}{1+
    % \exp(-c_i^n)
    % }\cdot
    % h(x,\eta^*_j)
    % -
    % \frac{1}{1+
    % \exp(-c_j^*)
    % }\cdot
    % h(x,\ej)
    % \right]
    % \\&
    % +
    % \sum_{j=\bn+1}^{\ns}
    % \left[
    % \sum_{i\in\calAj}
    % \frac{1}{1+
    % \exp(-x^{\top}A^n_ix-(b^n_i)^{\top}x-c_i^n)
    % }\cdot
    % h(x,\eta^n_i)
    % -
    % \frac{1}{1+
    % \exp(-x^{\top}A^*_jx-(b^*_j)^{\top}x-c_j^*)
    % }\cdot
    % h(x,\ej)
    % \right]
    \\&=
    \sum_{j=1}^{\bn}
    \sum_{i\in\calAj}
    \left[
    \frac{1}{1+
    \exp(-x^{\top}A^n_ix-(b^n_i)^{\top}x-c_i^n)
    }\cdot
    \cale(x,\eta^n_i)
    -
    \frac{1}{1+
    \exp(-c_i^n)
    }\cdot
    \cale(x,\ej)
    \right]:=\Ione_n
    \nonumber
    \\&
    +
    \sum_{j=1}^{\bn}
    \left[
    \sum_{i\in\calAj}
    \frac{1}{1+
    \exp(-c_i^n)
    }
    % \cdot
    % h(x,\eta^*_i)
    -
    \frac{1}{1+
    \exp(-c_j^*)
    }
    \right]
    \cdot
    \cale(x,\ej):=\Itwo_n
    \nonumber
    \\&
    +
    \sum_{j=\bn+1}^{\ns}
    \left[
    \sum_{i\in\calAj}
    \frac{1}{1+
    \exp(-x^{\top}A^n_ix-(b^n_i)^{\top}x-c_i^n)
    }\cdot
    \cale(x,\eta^n_i)
    -
    \frac{1}{1+
    \exp(-x^{\top}A^*_jx-(b^*_j)^{\top}x-c_j^*)
    }\cdot
    \cale(x,\ej)
    \right]:=\Ithree_n
    \nonumber
\end{align}
Let us denote $\sigma(x,A,b,c) = {1}/{[1 + \exp(-x^{\top}Ax - b^{\top}x - c)]}$,
then $\Ione_n, \Itwo_n$ and $\Ithree_n$ could be denoted as
\begin{align}
    \Ione_n&=\sum_{j=1}^{\bn}
    \sum_{i\in\calAj}
    \left[
    \frac{1}{1+
    \exp(-x^{\top}A^n_ix-(b^n_i)^{\top}x-c_i^n)
    }\cdot
    \cale(x,\eta^n_i)
    -
    \frac{1}{1+
    \exp(-c_i^n)
    }\cdot
    \cale(x,\ej)
    \right]
    \nonumber
    \\&
    =
    \sum_{j=1}^{\bn}
    \sum_{i\in\calAj}
    \left[
    \sigma(x,A_i^n,b_i^n,c_i^n) \cale(x,\eta^n_i)
    -
    \sigma(x,0_{d\times d},0_d,c_i^n) \cale(x,\eta^*_j)
    \right]
    \nonumber
    \\&
    =
    \sum_{j=1}^{\bn}
    \sum_{i\in\calAj}
    \sum_{|\gamma|=1}^2
    \frac{1}{\gamma!}
    (\Delta A_{ij}^n)^{\gamma_1}
    (\Delta b_{ij}^n)^{\gamma_2}
    (\Delta \eta_{ij}^n)^{\gamma_3}
    \frac{\partial^{|\gamma_1|+|\gamma_2|}\sigma}{\partial A^{\gamma_1}\partial b^{\gamma_2}}
    (x,0_{d\times d},0_d,c_i^n) 
    \frac{\partial^{|\gamma_3|}\cale}{\partial\eta^{\gamma_3}}
    (x,\eta^*_j)
    +R_1(x),
    % \\
\end{align}
\begin{align}
    \Itwo_n=
    \sum_{j=1}^{\bn}
    \left[
    \sum_{i\in\calAj}
    \frac{1}{1+
    \exp(-c_i^n)
    }
    % \cdot
    % h(x,\eta^*_i)
    -
    \frac{1}{1+
    \exp(-c_j^*)
    }
    \right]
    \cdot
    \cale(x,\ej),
    % \\
\end{align}
\begin{align}    \Ithree_n&=\sum_{j=\bn+1}^{\ns}
    \left[
    \sum_{i\in\calAj}
    \frac{1}{1+
    \exp(-x^{\top}A^n_ix-(b^n_i)^{\top}x-c_i^n)
    }\cdot
    \cale(x,\eta^n_i)
    -
    \frac{1}{1+
    \exp(-x^{\top}A^*_jx-(b^*_j)^{\top}x-c_j^*)
    }\cdot
    \cale(x,\ej)
    \right]
    \nonumber
    \\&
    =\sum_{j=\bn+1}^{\ns}
    \left[
    \sum_{i\in\calAj}
    % \frac{1}{1+
    % \exp(-x^{\top}A^n_ix-(b^n_i)^{\top}x-c_i^n)
    % }\cdot
    \sigma(x,A^n_i,b^n_i,c^n_i)
    \cale(x,\eta^n_i)
    -
    % \frac{1}{1+
    % \exp(-x^{\top}A^*_jx-(b^*_j)^{\top}x-c_j^*)
    % }\cdot
    \sigma(x,A^*_j,b^*_j,c^*_j)
    \cale(x,\ej)
    \right]
    \nonumber
    \\&
    =\sum_{j=\bn+1}^{\ns}
    \sum_{|\tau|=1}
    \left[
    \sum_{i\in\calAj}
    % \sigma(x,A^n_i,b^n_i,c^n_i)
    % h(x,\eta^n_i)
    % -
    % \sigma(x,A^*_j,b^*_j,c^*_j)
    % h(x,\ej)
    \frac{1}{\tau !}
    (\Delta A_{ij}^n)^{\tau_1}
    (\Delta b_{ij}^n)^{\tau_2}
    (\Delta c_{ij}^n)^{\tau_3}
    (\Delta \eta_{ij}^n)^{\tau_4}
    \right]
    \frac{\partial^{|\tau_1|+|\tau_2|+|\tau_3|}\sigma}{\partial A^{\tau_1} \partial b^{\tau_2} \partial c^{\tau_3}}(x,A^*_j,b^*_j,c^*_j)
    \frac{\partial^{|\tau_4|} \cale}{\partial \eta^{\tau_4}}(x,\ej)
    +R_2(x),
\end{align}
where $R_i(x), i=1,2$ are Taylor remainder such that $R_i(x)/\call_{1n}\to0$ as $n\to\infty$ for $i=[2]$.

Now we could denote 
\begin{align}
    K^n_{j,i,\gamma_{1:3}}&= \frac{1}{\gamma!}
    (\Delta A_{ij}^n)^{\gamma_1}
    (\Delta b_{ij}^n)^{\gamma_2}
    (\Delta \eta_{ij}^n)^{\gamma_3},~
    j\in[\bn],i\in\calAj,\gamma\in\Real^{d\times d}\times\Real^d\times\Real^q
    \label{pfeq:done_loss_notation1}
    \\
    L^n_j&=\sum_{i\in\calAj}
    \frac{1}{1+
    \exp(-c_i^n)
    }
    -
    \frac{1}{1+
    \exp(-c_j^*)
    },~
    j\in[\bn]
    \label{pfeq:done_loss_notation2}
    \\
    T^n_{j,\tau_{1:4}}&=
    \sum_{i\in\calAj}
    \frac{1}{\tau !}
    (\Delta A_{ij}^n)^{\tau_1}
    (\Delta b_{ij}^n)^{\tau_2}
    (\Delta c_{ij}^n)^{\tau_3}
    (\Delta \eta_{ij}^n)^{\tau_4},~
    j\in\{\bn+1,\cdots,\ns\}, 
    \tau\in\Real^{d\times d}\times\Real^d\times\Real\times\Real^q
    \label{pfeq:done_loss_notation3}
\end{align}
where $1\leq\sum_{i=1}^3|\gamma_i|\leq2$ and 
$\sum_{i=1}^4|\tau_i|=1$. 
Also recall that for the last term, $i$ was settled directly by $j$ since all the parameters are exact-fitted when $\bn+1\leq j\leq\ns$, i.e. $|\calAj|=1$.

Using these notations \eqref{pfeq:done_loss_notation1} - \eqref{pfeq:done_loss_notation3} we can now rewrite the difference
$f_{G_n}(x)-f_{\Gs}(x)$ as
\begin{align}
\label{pfeq:done_decomposition}
    f_{G_n}(x)-f_{\Gs}(x)
    % \\&
    &=
    \sum_{j=1}^{\bn}
    \sum_{i\in\calAj}
    \sum_{|\gamma|=1}^2
    % \frac{1}{\gamma!}
    % (\Delta A_{ij}^n)^{\gamma_1}
    % (\Delta b_{ij}^n)^{\gamma_2}
    % (\Delta \eta_{ij}^n)^{\gamma_3}
    K^n_{j,i,\gamma_{1:3}}
    \frac{\partial^{|\gamma_1|+|\gamma_2|}\sigma}{\partial A^{\gamma_1}\partial b^{\gamma_2}}
    (x,0_{d\times d},0_d,c_i^n) 
    \frac{\partial^{|\gamma_3|}\cale}{\partial\eta^{\gamma_3}}
    (x,\eta^*_j)
    % +R_1(x)
    % \\&
    +\sum_{j=1}^{\bn}
    % \left[
    % \sum_{i\in\calAj}
    % \frac{1}{1+
    % \exp(-c_i^n)
    % }
    % -
    % \frac{1}{1+
    % \exp(-c_j^*)
    % }
    % \right]
    L^n_j
    \cdot
    \cale(x,\ej)\nonumber
    \\&
    +\sum_{j=\bn+1}^{\ns}
    \sum_{|\tau|=1}
    % \left[
    % \sum_{i\in\calAj}
    % \frac{1}{\tau !}
    % (\Delta A_{ij}^n)^{\tau_1}
    % (\Delta b_{ij}^n)^{\tau_2}
    % (\Delta c_{ij}^n)^{\tau_3}
    % (\Delta \eta_{ij}^n)^{\tau_4}
    % \right]
    T^n_{j,\tau_{1:4}}
    \frac{\partial^{|\tau_1|+|\tau_2|+|\tau_3|}\sigma}{\partial A^{\tau_1} \partial b^{\tau_2} \partial c^{\tau_3}}(x,A^*_j,b^*_j,c^*_j)
    \frac{\partial^{|\tau_4|} \cale}{\partial \eta^{\tau_4}}(x,\ej)
    +R_1(x)
    +R_2(x).    
\end{align}

\textbf{Step 2 - Non-vanishing coefficients:}
Now we claim that at least one in the set 
$$
\mathcal{S}=
\left\{ 
\frac{K^n_{j,i,\gamma_{1:3}}}{\call_{1n}},
\frac{L^n_{j}}{\call_{1n}},
\frac{T^n_{j,\tau_{1:4}}}{\call_{1n}}
\right\}$$
will not vanish as n goes to infinity.
We prove by contradiction that all of them converge to zero when $n\to 0$:
\begin{align*}
    \frac{K^n_{j,i,\gamma_{1:3}}}{\call_{1n}}\to0,~
\frac{L^n_{j}}{\call_{1n}}\to0,~
\frac{T^n_{j,\tau_{1:4}}}{\call_{1n}}\to0.
\end{align*}
% \begin{align*}
%     \call_{1n}
%     &
%     :=\sum_{j=1}^{\bn}
%     \left|
%     \sum_{i\in\calAj} \frac{1}{1+\exp(-c^n_i)}-\frac{1}{1+\exp(-c_j^*)}
%     \right|
%     \\&
%     +\sum_{j=1}^{\bn}
%     \sum_{i\in\calAj}
%     \left[
%     \|\Delta A^n_{ij} \|^2
%     +\|\Delta b^n_{ij} \|^2
%     +\|\Delta \eta^n_{ij} \|^2
%     \right]
%     \\
%     &
%     +\sum_{j=\bn+1}^{\ns}
%     \sum_{i\in\calAj}
%     \left[
%     \|\Delta A^n_{ij} \|
%     +\|\Delta b^n_{ij} \|
%     +|\Delta c^n_{ij} |
%     +\|\Delta \eta^n_{ij} \|
%     \right]
% \end{align*}
Then follows directly from $L^n_j/\call_{1n}\to 0$ we could conclude that
\begin{align}
\label{pfeq:done_loss_coefficient1}
    \frac{1}{\call_{1n}}
    \sum_{j=1}^{\bn}
    \left|
    \sum_{i\in\calAj} \frac{1}{1+\exp(-c^n_i)}-\frac{1}{1+\exp(-c_j^*)}
    \right|
    =
    \frac{1}{\call_{1n}}
    \sum_{j=1}^{\bn}
    \left|
    L^n_j
    \right|
    \to0.
\end{align}
Before consider other coefficients, for simplicity, we denote
 $e_{d,u}=\overbrace{(0,\ldots,0,\underbrace{1}_{u\text{-th}},0,\ldots,0)}^{d\text{-tuple}}\in\Real^d$
as a $d$-tuple with all components equal to 0, except the $u$-th, which is $1$;
and $e_{d\times d,uv}$ as a $d\times d$ matrix with all components equal to 0, except the element in the $u$-th row and $v$-th column , which is $1$, i.e. $e_{d\times d,uv}=\overbrace{(0_d^{\top},\ldots,0_d^{\top},\underbrace{e^{\top}_{d,u}}_{v\text{-th}},0_d^{\top},\ldots,0_d^{\top})}^{d\text{-column}}\in\Real^{d\times d}$.

Now consider for arbitrary
$ u,v \in[d]$,
let $\gamma_1=2e_{d\times d,uv}, \gamma_2=0_{d}$ and $\gamma_3=0_{q}$,
we will have 
\begin{align*}
\frac{1}{\call_{1n}}
\left|
\Delta (A^n_{ij})^{(uv)}
\right|^2
=
\frac{1}{2 \call_{1n}}
% \frac{1}{2}
\left| 
K^n_{j,i,2e_{d\times d,uv},0_d,0_q}
\right|
\to 0,~ n\to \infty.
\end{align*}
Then by taking the summation of the term with $u,v\in[d]$, we will have
\begin{align}
\label{pfeq:done_loss_coefficient2}
    \frac{1}{\call_{1n}}
    \sum_{j=1}^{\bn}
    \sum_{i\in\calAj}
    \|\Delta A^n_{ij} \|^2
    =
    \frac{1}{2\call_{1n}}
    % \frac{d^2}{2}
    \sum_{j=1}^{\bn}
    \sum_{i\in\calAj}
    \sum_{u=1}^d
    \sum_{v=1}^d
    | K^n_{j,i,2e_{d\times d,uv},0_d,0_q}|\to0.
\end{align}
For arbitrary
$ u \in[d]$,
let $\gamma_1=0_{d\times d}, \gamma_2=2e_{d,u}$ and $\gamma_3=0_{q}$,
we will have 
\begin{align*}
\frac{1}{\call_{1n}}
\left|
\Delta (b^n_{ij})^{(u)}
\right|^2
=
\frac{1}{2\call_{1n}}
% \frac{1}{2}
\left| 
K^n_{j,i,0_{d\times d},2e_{d,u},0_q}
\right|
\to 0,~ n\to \infty,
\end{align*}
Then by taking the summation of the previous term with $u\in[d]$, we will have
\begin{align}
\label{pfeq:done_loss_coefficient3}
    \frac{1}{\call_{1n}}
    \sum_{j=1}^{\bn}
    \sum_{i\in\calAj}
    \|\Delta b^n_{ij} \|^2
    =
    \frac{1}{2\call_{1n}}
    % \frac{d}{2}
    \sum_{j=1}^{\bn}
    \sum_{i\in\calAj}
    \sum_{u=1}^d
    | K^n_{j,i,0_{d\times d},2e_{d,u},0_q}|\to0.
\end{align}
Similarly we will have 
\begin{align}
\label{pfeq:done_loss_coefficient4}
    \frac{1}{\call_{1n}}
    \sum_{j=1}^{\bn}
    \sum_{i\in\calAj}
    \|\Delta \eta^n_{ij} \|^2
    =
    \frac{1}{2\call_{1n}}
    \sum_{j=1}^{\bn}
    \sum_{i\in\calAj}
    \sum_{w=1}^q
    | K^n_{j,i,0_{d\times d},0_{d},2e_{q,w}}|\to0.
\end{align}
% $K^n_{j,i,\gamma_{1:3}}= \frac{1}{\gamma!}
%     (\Delta A_{ij}^n)^{\gamma_1}
%     (\Delta b_{ij}^n)^{\gamma_2}
%     (\Delta \eta_{ij}^n)^{\gamma_3},~
%     j\in[\bn],i\in\calAj,\gamma\in\Real^{d\times d}\times\Real^d\times\Real^q$
% $$    T^n_{j,\tau_{1:4}}=
%     \sum_{i\in\calAj}
%     \frac{1}{\tau !}
%     (\Delta A_{ij}^n)^{\tau_1}
%     (\Delta b_{ij}^n)^{\tau_2}
%     (\Delta c_{ij}^n)^{\tau_3}
%     (\Delta \eta_{ij}^n)^{\tau_4},~
%     j\in\{\bn+1,\cdots,\ns\}, 
%     \tau\in\Real^{d\times d}\times\Real^d\times\Real\times\Real^q$$
Now, consider 
$\bn+1\leq j\leq\ns$,
such that its corresponding Voronoi cell has only one element, i.e.
$|\calAj|=1$. 
For arbitrary
$ u,v \in[d]$,
let $\tau_1=e_{d\times d,uv}, \tau_2=0_{d}, \tau_3=0$ and $\tau_4=0_{q}$,
we will have 
\begin{align*}
\frac{1}{\call_{1n}}
\sum_{i\in\calAj}
\left|
\Delta (A^n_{ij})^{(uv)}
\right|
=
\frac{1}{\call_{1n}}
\left| 
T^n_{j,e_{d\times d,uv},0_d,0,0_q}
\right|
\to 0,~ n\to \infty.
\end{align*}
Then by taking the summation of the term with $u,v\in[d]$, we will have
\begin{align*}
    \frac{1}{\call_{1n}}
    \sum_{j=\bn+1}^{\ns}
    \sum_{i\in\calAj}
    \|\Delta A^n_{ij} \|_1
    =
    \frac{1}{\call_{1n}}
    \sum_{j=\bn+1}^{\ns}
    \sum_{u=1}^d
    \sum_{v=1}^d
    | T^n_{j,e_{d\times d,uv},0_d,0,0_q}|\to0.
\end{align*}
Recall the topological equivalence between $L_1$-norm and $L_2$-norm on finite-dimensional vector space over $\Real$, we will have
\begin{align}
\label{pfeq:done_loss_coefficient5}
    \frac{1}{\call_{1n}}
    \sum_{j=\bn+1}^{\ns}
    \sum_{i\in\calAj}
    \|\Delta A^n_{ij} \|
    %_1
    % =
    % \frac{1}{\call_{1n}}
    % \sum_{j=\bn+1}^{\ns}
    % \sum_{u=1}^d
    % \sum_{v=1}^d
    % | T^n_{j,e_{d\times d,uv},0_d,0,0_q}|
    \to0.
\end{align}
% $    T^n_{j,\tau_{1:4}}=
%     \sum_{i\in\calAj}
%     \frac{1}{\tau !}
%     (\Delta A_{ij}^n)^{\tau_1}
%     (\Delta b_{ij}^n)^{\tau_2}
%     (\Delta c_{ij}^n)^{\tau_3}
%     (\Delta \eta_{ij}^n)^{\tau_4},~
%     j\in\{\bn+1,\cdots,\ns\}, 
%     \tau\in\Real^{d\times d}\times\Real^d\times\Real\times\Real^q$
Following a similar argument, 
since
\begin{align*}
    \frac{1}{\call_{1n}}
    \sum_{j=\bn+1}^{\ns}
    \sum_{i\in\calAj}
    \|\Delta b^n_{ij} \|_1
    &=
    \frac{1}{\call_{1n}}
    \sum_{j=\bn+1}^{\ns}
    \sum_{u=1}^d
    | T^n_{j,0_{d\times d},e_{d,u},0,0_q}|\to0,
    \\
    \frac{1}{\call_{1n}}
    \sum_{j=\bn+1}^{\ns}
    \sum_{i\in\calAj}
    |\Delta c^n_{ij} |
    &=
    \frac{1}{\call_{1n}}
    \sum_{j=\bn+1}^{\ns}
    % \sum_{u=1}^d
    | T^n_{j,0_{d\times d},0_{d},1,0_q}|\to0,
    \\
    \frac{1}{\call_{1n}}
    \sum_{j=\bn+1}^{\ns}
    \sum_{i\in\calAj}
    \|\Delta \eta^n_{ij} \|_1
    &=
    \frac{1}{\call_{1n}}
    \sum_{j=\bn+1}^{\ns}
    \sum_{w=1}^q
    | T^n_{j,0_{d\times d},0_{d},0,e_{q,w}}|\to0,
\end{align*}
we obtain that 
\begin{align}
\label{pfeq:done_loss_coefficient6}
    % &
    % \frac{1}{\call_{1n}}
    % \sum_{j=1}^{\bn}
    % \left|
    % \sum_{i\in\calAj} \frac{1}{1+\exp(-c^n_i)}-\frac{1}{1+\exp(-c_j^*)}
    % \right|
    % =
    % \frac{1}{\call_{1n}}
    % \sum_{j=1}^{\bn}
    % \left|
    % L^n_j
    % \right|
    % \to0;
    % % \label{pfeq:done_loss_coefficient1}
    % \\
    % &
    % \frac{1}{\call_{1n}}
    % \sum_{j=1}^{\bn}
    % \sum_{i\in\calAj}
    % \|\Delta A^n_{ij} \|^2
    % =
    % \frac{1}{\call_{1n}}
    % \sum_{j=1}^{\bn}
    % \sum_{i\in\calAj}
    % | K^n_{j,i,2e_{d\times d,uv},0_d,0_q}|
    % \label{pf?eq:done_loss_coefficient2}
    % \\&
    % \frac{1}{\call_{1n}}
    % \sum_{j=1}^{\bn}
    % \sum_{i\in\calAj}
    % \|\Delta b^n_{ij} \|^2
    % \\&
    % \frac{1}{\call_{1n}}
    % \sum_{j=1}^{\bn}
    % \sum_{i\in\calAj}
    % \|\Delta \eta^n_{ij} \|^2
    % \\&
    % \frac{1}{\call_{1n}}
    % \sum_{j=\bn+1}^{\ns}
    % \sum_{i\in\calAj}
    % \|\Delta A^n_{ij} \|
    % \\
    % &
    \frac{1}{\call_{1n}}
    \sum_{j=\bn+1}^{\ns}
    \sum_{i\in\calAj}
    \|\Delta b^n_{ij} \|\to0,~
    % \\&
    \frac{1}{\call_{1n}}
    \sum_{j=\bn+1}^{\ns}
    \sum_{i\in\calAj}
    |\Delta c^n_{ij} |\to0,~
    % \\&
    \frac{1}{\call_{1n}}
    \sum_{j=\bn+1}^{\ns}
    \sum_{i\in\calAj}
    \|\Delta \eta^n_{ij} \|\to0.
\end{align}
% where $e_{d,u}=\overbrace{(0,\ldots,0,\underbrace{1}_{u\text{-th}},0,\ldots,0)}^{d\text{-tuple}}\in\Real^d$
% is a $d$-tuple with all components equal to 0, except the $u$-th, which is $1$;
% and $e_{d\times d,uv}$ is a $d\times d$ matrix with all components equal to 0, except the element in the $u$-th row and $v$-th column , which is $1$, i.e. $e_{d\times d,uv}=\overbrace{(0_d^{\top},\ldots,0_d^{\top},\underbrace{e_{d,u}}_{v\text{-th}},0_d^{\top},\ldots,0_d^{\top})}^{d\text{-column}}\in\Real^{d\times d}$. 
% The first equation \eqref{pfeq:done_loss_coefficient1}
% follows directly from $L^n_j/\call_{1n}\to 0$.
% Then in equation \eqref{pfeq:done_loss_coefficient2}, 
% for arbitrary
% $ u,v \in[d]$,
% let $\gamma_1=2e_{d\times d,uv}, \gamma_2=0_{d}$ and $\gamma_3=0_{q}$,
% we will have 
% $$
% |\Delta (A^n_{ij})^{(uv)}|^2
% $$
% $K^n_{j,i,\gamma_{1:3}}= \frac{1}{\gamma!}
%     (\Delta A_{ij}^n)^{\gamma_1}
%     (\Delta b_{ij}^n)^{\gamma_2}
%     (\Delta \eta_{ij}^n)^{\gamma_3},~
%     j\in[\bn],i\in\calAj,\gamma\in\Real^{d\times d}\times\Real^d\times\Real^q$
Now taking
the summation of limits in equations
\eqref{pfeq:done_loss_coefficient1} - 
\eqref{pfeq:done_loss_coefficient6},
we could deduce that
    \begin{align*}
    1=\frac{\call_{1n}}{\call_{1n}}
    &
    =
    \frac{1}{\call_{1n}}\sum_{j=1}^{\bn}
    \left|
    \sum_{i\in\calAj} \frac{1}{1+\exp(-c^n_i)}-\frac{1}{1+\exp(-c_j^*)}
    \right|
    % \nonumber\\
    +
    \frac{1}{\call_{1n}}%\sum_{j=1}^{\bn}
    \sum_{j=1}^{\bn}
    \sum_{i\in\calAj}
    \left[
    \|\Delta A^n_{ij} \|^2
    +\|\Delta b^n_{ij} \|^2
    +\|\Delta \eta^n_{ij} \|^2
    \right]
    % \nonumber
    \\
    &+
    \frac{1}{\call_{1n}}%\sum_{j=1}^{\bn}
    \sum_{j=\bn+1}^{\ns}
    \sum_{i\in\calAj}
    \left[
    \|\Delta A^n_{ij} \|
    +\|\Delta b^n_{ij} \|
    +|\Delta c^n_{ij} |
    +\|\Delta \eta^n_{ij} \|
    \right]\to0,
    % ~~n\to\infty,
\end{align*}
as $n\to\infty$,
which is a contradiction.
Thus, not all the coefficients of elements in the set
% at least one element in the set 
\begin{align*}
\mathcal{S}=
\left\{ 
\frac{K^n_{j,i,\gamma_{1:3}}}{\call_{1n}},
\frac{L^n_{j}}{\call_{1n}},
\frac{T^n_{j,\tau_{1:4}}}{\call_{1n}}
\right\}
\end{align*}
tend to 0 as $n\to\infty$. 
% will not go to zero as n tends to infinity.
% \textbf{Step 3: Application of Fatou’s Lemma.}
Let us denote by $m_n$ the maximum of the absolute values of those elements.
It follows from the previous result that $1/m_n\not\to \infty$ as $n\to\infty$.

\textbf{Step 3 - Application of Fatou’s lemma:} 
% In this step, we use the Fatou’s lemma to argue against the result in Step 2, and then, achieve the desired inequality in equation \ref{pfeq:done_local}.
In this step, we apply Fatou's lemma to obtain the desired inequality in equation \eqref{pfeq:done_local}.
Recall in the beginning we have assumed that
$
        % \call_{1n}:=D_1(G_n.\Gs)\to 0,\\
        \|f_{G_n}-f_{\Gs} \|_{L^2(\mu)}/\call_{1n}\to0
$
and the topological equivalence between $L_1$-norm and $L_2$-norm on finite-dimensional vector space over $\Real$,
we obtain $
        \|f_{G_n}-f_{\Gs} \|_{L^1(\mu)}/\call_{1n}\to0
$. 
By applying the Fatou’s lemma, we get
\begin{align*}
    0=\lim_{n\to\infty}\frac{\|f_{G_n}-f_{\Gs} \|_{L^1(\mu)}}{m_n \call_{1n}}
    \geq
    \int\liminf_{n\to\infty}
    \frac{|f_{G_n}(x)-f_{\Gs}(x)|}{m_n \call_{1n}}d\mu(x)
    \geq 0.
\end{align*}
This result suggests that for almost every $x$, 
\begin{align}
\label{pfeq:done_fatou}
    \frac{f_{G_n}(x)-f_{\Gs}(x)}{m_n \call_{1n}}\to 0,
\end{align}
recall equation \eqref{pfeq:done_decomposition}, we deduce that
\begin{align*}
    \frac{f_{G_n}(x)-f_{\Gs}(x)}{m_n \call_{1n}}
    &=
    \sum_{j=1}^{\bn}
    \sum_{i\in\calAj}
    \sum_{|\gamma|=1}^2
    \frac{K^n_{j,i,\gamma_{1:3}}}{m_n \call_{1n}}
    \frac{\partial^{|\gamma_1|+|\gamma_2|}\sigma}{\partial A^{\gamma_1}\partial b^{\gamma_2}}
    (x,0_{d\times d},0_d,c_i^n) 
    \frac{\partial^{|\gamma_3|}\cale}{\partial\eta^{\gamma_3}}
    (x,\eta^*_j)
    +\sum_{j=1}^{\bn}
    \frac{L^n_j}{m_n \call_{1n}}
    \cdot
    \cale(x,\ej)\nonumber
    \\&
    +\sum_{j=\bn+1}^{\ns}
    \sum_{|\tau|=1}
    \frac{T^n_{j,\tau_{1:4}}}{m_n \call_{1n}}
    \frac{\partial^{|\tau_1|+|\tau_2|+|\tau_3|}\sigma}{\partial A^{\tau_1} \partial b^{\tau_2} \partial c^{\tau_3}}(x,A^*_j,b^*_j,c^*_j)
    \frac{\partial^{|\tau_4|} \cale}{\partial \eta^{\tau_4}}(x,\ej)
    +\frac{R_1(x)}{m_n \call_{1n}}
    +\frac{R_2(x)}{m_n \call_{1n}}.    
\end{align*}
Let us denote 
\begin{align*}
    \frac{K^n_{j,i,\gamma_{1:3}}}{m_n \call_{1n}}\to k_{j,i,\gamma_{1:3}},~~
    \frac{L^n_j}{m_n \call_{1n}}\to l_j,~~
    \frac{T^n_{j,\tau_{1:4}}}{m_n \call_{1n}}\to t_{j,\tau_{1:4}},
\end{align*}
then from equation \eqref{pfeq:done_fatou},
we will have
\begin{align*}
    % \frac{f_{G_n}(x)-f_{\Gs}(x)}{m_n \call_{1n}}
    % &=
    &
    \sum_{j=1}^{\bn}
    \sum_{i\in\calAj}
    \sum_{|\gamma|=1}^2
    k_{j,i,\gamma_{1:3}}
    \frac{\partial^{|\gamma_1|+|\gamma_2|}\sigma}{\partial A^{\gamma_1}\partial b^{\gamma_2}}
    (x,0_{d\times d},0_d,c_i^n) 
    \frac{\partial^{|\gamma_3|}\cale}{\partial\eta^{\gamma_3}}
    (x,\eta^*_j)
    +\sum_{j=1}^{\bn}
    l^n_j
    \cdot
    \cale(x,\ej)\nonumber
    \\
    &
    +\sum_{j=\bn+1}^{\ns}
    \sum_{|\tau|=1}
    t^n_{j,\tau_{1:4}}
    \frac{\partial^{|\tau_1|+|\tau_2|+|\tau_3|}\sigma}{\partial A^{\tau_1} \partial b^{\tau_2} \partial c^{\tau_3}}(x,A^*_j,b^*_j,c^*_j)
    \frac{\partial^{|\tau_4|} \cale}{\partial \eta^{\tau_4}}(x,\ej)=0,    
\end{align*}
for almost every $x$.
Note that the expert function $\cale(\cdot,\eta)$ is strongly identifiable, then the above equation implies that
\begin{align*}
    k_{j,i,\gamma_{1:3}}= l_j= t_{j,\tau_{1:4}}=0,
\end{align*}
for any $j\in[\ns]$,  $(\gamma_1,\gamma_2,\gamma_3)\in\bbN^{d\times d}\times\bbN^d\times\bbN^q$ and 
    $(\tau_1,\tau_2,\tau_3,\tau_4)\in\bbN^{d\times d}\times\bbN^d\times\bbN\times\bbN^q$
such that 
$1\leq \sum_{i=1}^3|\gamma_i| \leq 2$ and
$\sum_{i=1}^4|\tau_i|=1 $.
This violates that at least one among the limits in the set $\{k_{j,i,\gamma_{1:3}}, l_j,  t_{j,\tau_{1:4}} \}$ is different from zero.

Thus, we obtain the local inequality in equation \eqref{pfeq:done_local}. Consequently, there exists some $\varepsilon' > 0$ such that
\begin{align*}
    \inf_{G\in\calm_N(\Theta):\lone\leq\varepsilon'}
    \frac{\Vert f_{G}-f_{\Gs}\Vert_{L^2(\mu)}}{\lone}
    >0.
\end{align*}

\textbf{Global part:}
% Now we will demonstrate equation \eqref{pfeq:done_aim} when the denominator will not vanish 
We now proceed to demonstrate equation \eqref{pfeq:done_aim} for the case where the denominator does not vanish, i.e.
\begin{align}
\label{pfeq:done_global}
    \inf_{G\in\calm_N(\Theta):\lone>\varepsilon'}
    \frac{\Vert f_{G}-f_{\Gs}\Vert_{L^2(\mu)}}{\lone}
    >0.
\end{align}
Suppose, for contradiction, that inequality \eqref{pfeq:done_global} does not hold. Then there exists a sequence of mixing measures $G'_n \in \mathcal{M}_N(\Theta)$ such that $\call_1(G'_n, \Gs) > \varepsilon'$ and
\begin{align*}
    \lim_{n\to\infty}
    \frac{\Vert f_{G'_n}-f_{\Gs}\Vert_{L^2(\mu)}}{\call_1(G'_n,\Gs)}
    =0.
\end{align*}
Under this situation, we could deduce that $\Vert f_{G'_n}-f_{\Gs}\Vert_{L^2(\mu)}\to0$ as $n\to\infty$.
Since $\Theta$ is a compact set, we can replace the sequence $G'_n$ with a convergent subsequence, which approaches a mixing measure $G' \in \mathcal{M}_N(\Theta)$. Given that $\call_1(G'_n, \Gs) > \varepsilon'$, we conclude that $\call_1(G', \Gs) > \varepsilon'$. Then, using Fatou's lemma, we deduce:
\begin{align*}
    0=\lim_{n\to\infty}
    {\|f_{G'_n}-f_{\Gs} \|^2_{L^2(\mu)}}
    \geq
    \int\liminf_{n\to\infty}
    {\left|f_{G'_n}(x)-f_{\Gs}(x)\right|^2}d\mu(x),
\end{align*}
which indicates that $f_{G'}(x)=f_{\Gs}(x)$ for almost every $x$. 
Based on the Proposition \ref{prop:identifiability} followed, we will have that 
$G'\equiv\Gs$.

\end{proof}

\begin{proposition}
\label{prop:identifiability}
    If the equation $f_G (x) = f_{\Gs} (x)$ holds true for almost every x, then it follows that $G \equiv \Gs $.
\end{proposition}

\begin{proof}[Proof of Proposition \ref{prop:identifiability}]
Assume that $ f_G(x) = f_{\Gs}(x) $ for almost every $x$, which gives:
\begin{equation}
\sum_{i=1}^{N} \frac{1}{1 + \exp(-x^{\top}A_ix - b_i^{\top}x - c_i)} \cale(x, \eta_i) = 
\sum_{i=1}^{N_*} \frac{1}{1 + \exp(-x^{\top}\Asi x - (\bsi)^{\top}x - \csi)} \cale(x, \eta_i^*). 
\label{eq:matching}
\end{equation}

Since $ \cale(x, \eta) $ is identifiable, the set $ \{\cale(x, \eta_i), i \in [N]\} $ is linearly independent, and $ \{\cale(x, \eta_i^*), i \in [N_*]\} $ is distinct for some $ N_* \in \mathbb{N} $. Therefore, if $N \neq N_*$, there must exist some index $i \in [N]$ such that $ \eta_i \neq \eta_j^* $ for any $j \in [N_*]$. This leads to a contradiction as the coefficients cannot simultaneously satisfy equation~\eqref{eq:matching}. Consequently, we must have $N = N_*$.

Thus, the sets of weights and gating functions on both sides of \eqref{eq:matching} are equivalent:
\begin{align*}
\left\{ \frac{1}{1 + \exp(-x^{\top}A_ix - b_i^{\top}x - c_i)},    i \in [N] \right\} = 
\left\{ \frac{1}{1 + \exp(-x^{\top}\Asi x - (\bsi)^{\top}x - \csi)},    i \in [N_*] \right\}.
\end{align*}

Without loss of generality, we can assume that the correspondence is such that
\begin{align*}
\frac{1}{1 + \exp(-x^{\top}A_ix - b_i^{\top}x - c_i)} = 
\frac{1}{1 + \exp(-x^{\top}\Asi x - (\bsi)^{\top}x - \csi)},
\end{align*}
for all $ i \in [N] $ and for almost every $x$. This implies the sigmoid function's invariance to translations, leading to:
\begin{align*}
A_i = A_i^* + v_2, \quad b_i = b_i^* + v_1, \quad c_i = c_i^* + v_0,
\end{align*}
for some $v_2 \in \mathbb{R}^{d\times d}$, $v_1 \in \mathbb{R}^d$ and $v_0 \in \mathbb{R}$. However, due to the assumption $ A_k = A_k^* = 0_{d\times d} $, $ b_k = b_k^* = 0_{d} $ and $ c_k = c_k^* = 0 $, it follows that $v_2 = 0_{d\times d}$, $v_1 = 0_{d}$ and $v_0 = 0$, leading to:
\begin{align*}
A_i = A_i^*, \quad b_i = b_i^*, \quad c_i = c_i^*, \quad \text{for any } i \in [N].
\end{align*}

Substituting this back into equation~\eqref{eq:matching}, we have:
\begin{equation}
\sum_{i=1}^{N} \frac{1}{1 + \exp(-x^{\top}A_ix - b_i^{\top}x - c_i)} \cale(x, \eta_i) = 
\sum_{i=1}^{N} \frac{1}{1 + \exp(-x^{\top}A_ix - b_i^{\top}x - c_i)} \cale(x, \eta_i^*).
\end{equation}

Next, let us partition the index set $[N]$ into subsets $J_1, J_2, \ldots, J_m$, where $m \leq k$, such that $ (A_{i}, b_{i}, c_{i}) = (A^*_{i'}, b^*_{i'}, c^*_{i'}) $ for any $i, i' \in J_j$ and $j \in [m]$. For indices $i$ and $i'$ belonging to distinct subsets, it holds that $ (A_{i}, b_{i}, c_{i}) \neq (A_{i'}, b_{i'}, c_{i'}) $.

Using this partition, the equation can be reorganized as:
\begin{equation}
\sum_{j=1}^m \sum_{i \in J_j} \frac{1}{1 + \exp(-x^{\top}A_ix - b_i^{\top}x - c_i)} \cale(x, \eta_i) =
\sum_{j=1}^m \sum_{i \in J_j} \frac{1}{1 + \exp(-x^{\top}A^*_i x - (b^*_i)^{\top}x - c^*_i)} \cale(x, \eta_i^*),
\end{equation}
for almost every $x$.

Recall that $A_i = A_i^*$, $b_i = b_i^*$ and $c_i = c_i^*$ for all $i \in [\ns]$. This implies that for any $j \in [m]$, we can identify the sets:
\begin{align*}
\{\eta_i : i \in J_j\} \equiv \{\eta_i^* : i \in J_j\}.
\end{align*}

Consequently, we obtain:
\begin{align*}
G = \sum_{j=1}^m \sum_{i \in J_j} \frac{1}{1 + \exp(-c_i  )} \delta(A_i, b_i, \eta_i) = 
\sum_{j=1}^m \sum_{i \in J_j} \frac{1}{1 + \exp(-c_i^*)} \delta(A_i^*, b_i^*, \eta_i^*)=\Gs,
\end{align*}
showing $ G \equiv \Gs $. Thus, the proof is complete.
\end{proof}

\subsection{Proof of Theorem \ref{thm:dtwor_loss_linear}}
\label{proof:dtwor_loss_linear}

To prove the result, we first examine the Voronoi loss:
\begin{align*}
    % &
    \ltwor:=
    &
    \sum_{j=1}^{\bn}
    \left|
    \sum_{i\in\calAj} \frac{1}{1+\exp(-c_i)}-\frac{1}{1+\exp(-c_j^*)}
    \right|
    \nonumber
    \\
    &
    +\sum_{j=1}^{\bn}
    \sum_{i\in\calAj}
    \left[
    \|\Delta A_{ij} \|^r
    +\|\Delta b_{ij} \|^r
    +\|\Delta \alpha_{ij} \|^r
    +|\Delta \beta_{ij} |^r
    \right]
    \nonumber
    \\
    &
    +\sum_{j=\bn+1}^{\ns}
    \sum_{i\in\calAj}
    \left[
    \|\Delta A_{ij} \|^r
    +\|\Delta b_{ij} \|^r
    +|\Delta c_{ij} |^r
    +\|\Delta \alpha_{ij} \|^r
    +|\Delta \beta_{ij} |^r
    \right]
\end{align*}
where the difference terms $\Delta A_{ij}:=A_i-A^*_j$,  $\Delta b_{ij}:=b_i-b^*_j$, 
$\Delta c_{ij}:=c_i-c^*_j$, 
$\Delta \alpha_{ij}:=\alpha_i-\alpha^*_j$, 
$\Delta \beta_{ij}:=\beta_i-\beta^*_j$.
These terms are asymmetric by definition because $\Delta A_{ij} \neq \Delta A_{ji}$ in general (i.e., $ A_i - A_j^* \neq A_j - A^*_i$ where $i\in\calAj$).
So $\call_{2,r}(G,\Gs)$ is not symmetric, 
but it still satisfies a weak triangle inequality.
% \begin{align*}
%     \call_{2,r}(G_1,G_2)+
%     \call_{2,r}(G_2,G_3)
%     \gtrsim
%     \min
%     \left\{
%     \call_{2,r}(G_1,G_2),\call_{2,r}(G_2,G_3)
%     \right\}.
% \end{align*}
% Therefore, we will apply the Le Cam's lemma \cite{},
% to handle this $\call_{2,r}$ loss which satisfies the weak triangle inequality.
% For experts satisfies all assumptions in Theorem \ref{thm:dtwor_loss}, 
% based on the Taylor expansion, we have
% the following results:
% Thus, we apply Le Cam's lemma \textcolor{red}{\cite{}} to address the $\call_{2,r}$ loss, which satisfies the weak triangle inequality. 
For experts meeting all the assumptions in Theorem \ref{thm:dtwor_loss_linear}, we derive the following results using Taylor expansion:
\begin{lemma}
\label{lemma:dtwor_loss}
    Given experts in Theorem \ref{thm:dtwor_loss_linear}, we achieve for any $r\geq 1$ that
    \begin{align*}
        \lim_{\epsilon\to0}
        \inf_{G\in\calm_N(\Theta)}
        \left\{
        \frac{\| f_G-f_{\Gs} \|_{L^2(\mu)}}{\call_{2,r}(G,\Gs)}:\call_{2,r}(G,\Gs)\leq\epsilon
        \right\}=0.
    \end{align*}
\end{lemma}
We will prove this lemma later.

\begin{proof}[Main Proof]
Recall that 
$(X_1,Y_1),(X_2,Y_2),\cdots,(X_n,Y_n)\in\Real^d\times\Real $ follow a
standard regression model
\begin{align*}
    Y_i=f_{\Gs}(X_i)+\varepsilon_i,~i=1,\cdots,n,
\end{align*}
where $X_1,X_2,\cdots,X_n$ are i.i.d. samples from a probability distribution $\mu$ on $\Real^d$, and 
$\varepsilon_i$ are i.i.d. Gaussian noise variables with $\bbE[\varepsilon_i|X_i]=0$ and $\text{Var}[\varepsilon_i|X_i]=\nu$, $i\in[n]$.
Under the Gaussian assumption on the noise variables   $\varepsilon_i$, we have that 
$Y_i|X_i\sim\caln(f_{\Gs}(X_i),\nu), i=1,\ldots,n$. 
Given Lemma \ref{lemma:dtwor_loss},
there exists a sufficiently small $\epsilon>0$ and a mixing measure $\Gs^{\prime}\in\calm_N(\Theta)$
such that $\call_{2,r}(\Gs^{\prime},\Gs)=2\epsilon$ and $\|f_{\Gs^{\prime}}-f_{\Gs} \|\leq \sqrt{C_1}\epsilon$ for a fixed constant $C_1$.
Thus, applying Le Cam's lemma 
% \cite{yu97lecam} 
to address the $\call_{2,r}$ loss, which satisfies the weak triangle inequality, we obtain:
\begin{align*}
    &\inf_{\tGn\in\calm_N(\Theta)}
    \sup_{G\in\calm_N(\Theta)\setminus\calm_{\ns-1}(\Theta)}
    \bbE_{f_G}[\call_{2,r}(\tGn,G)]\\
    &\qquad \gtrsim
    \frac{\call_{2,r}(\Gsp,\Gs)}{8}
    \exp
    \left\{
    -n\bbE_{X\sim\mu}
    \left[
    \kl
    \left(
    \caln(f_{\Gsp}(X),\nu), \caln(f_{\Gs}(X),\nu)
    \right)
    \right]
    \right\}
\end{align*}
Recall that the KL divergence between two Gaussian distributions with the same variance is
\begin{align*}
    \kl\left(
    \caln(f_{\Gsp}(X),\nu), \caln(f_{\Gs}(X),\nu)
    \right)
    =
    \frac{1}{2\nu}
    \left(
    f_{\Gsp}(X)-f_{\Gs}(X)
    \right)^2.
\end{align*}
Hence, we can deduce that:
\begin{align*}
    \inf_{\tGn\in\calm_N(\Theta)}
    \sup_{G\in\calm_N(\Theta)\setminus\calm_{\ns-1}(\Theta)}
    \bbE_{f_G}[\call_{2,r}(\tGn,G)]
    &\gtrsim
    \epsilon
    \exp
    \left(
    -n\|f_{\Gsp}-f_{\Gs} \|^2_{L^2(\mu)}
    \right)\\
    &\gtrsim
    \epsilon
    \exp(-nC_1\epsilon^2).
\end{align*}
By setting $\epsilon = n^{-1/2}$, we obtain $\epsilon\exp(- n C_1 \varepsilon^2) = n^{-1/2} \exp(-C_1)$, which means that
\begin{align*}
    \inf_{\tGn\in\calm_N(\Theta)}
    \sup_{G\in\calm_N(\Theta)\setminus\calm_{\ns-1}(\Theta)}
    \bbE_{f_G}[\call_{2,r}(\tGn,G)]
    \gtrsim
    n^{-\frac{1}{2}},
\end{align*}
Consequently, we establish the result for Theorem \ref{thm:dtwor_loss_linear}.
\end{proof}

\begin{proof}[Proof of Lemma \ref{lemma:dtwor_loss}]
% Following the proof of Theorem \ref{thm:dtwor_loss}, 
% It suffices to 
Now we want to
demonstrate that the following limit holds true for any $ r \geq 1 $:
\begin{align}
\label{pfeq:dtwor_loss_linear_aim}
    \lim_{\epsilon\to0}
    \inf_{G\in\calm_N(\Theta):\call_{2,r}(G,\Gs)\leq\epsilon}
    \frac{\|f_G-f_{\Gs} \|_{L^2(\mu)}}{\call_{2,r}(G,\Gs)}=0.
\end{align}
To achieve this, we need to construct a sequence of mixing measures $ (G_n) $ that satisfies
$\call_{2,r}(G_n,\Gs)\to0$ and 
\begin{align*}
    \frac{\|f_G-f_{\Gs} \|_{L^2(\mu)}}{\call_{2,r}(G,\Gs)}\to0,
\end{align*}
as $n\to\infty$. 
% Under Scenario 2, recall that at least one of the parameters $ \alpha^*_1, \ldots, \alpha^*_{k^*} $ equals $ 0_d $. Without loss of generality, we assume $ \alpha^*_1 = 0_d $. Now, let us consider the sequence $ (G_n) $ with $ k^* + 1 $ atoms, where
Recall we consider that at least one expert parameter $\alpha_i^*$,
where $i\in\calAj,j\in[\ns]$ in the over-specified gating parameters Voronoi cell 
% $\alpha_i^*,i\in\calAj,j\in[\ns]$ 
equals $0_d$.
% Then under the Regime 1 that all the over-specified parameter $\Asi$ and $\bsi$ are zero, the strong identifiability will no longer hold.
WLOG, assume $\alpha^*_1=0_d$.
Next, we consider the sequence $(G_n)$ with $k^* + 1$ atoms, in which
for over-specified parameters $i=1,2$: 
$A_i^n= A^*_1=0_{d\times d}$,
$b_i^n= b^*_1=0_{d}$,
$c_1^n=c_2^n$ such that 
\begin{align*}
    \sum_{i=1}^2\frac{1}{1+\exp(-c^n_i)}=\frac{1}{1+\exp(-c^*_1)}+\frac{1}{n^{r+1}},
\end{align*}
$\alpha_i^n=\alpha^*_1=0_{d}$,
$\beta_1^n=\beta_1^*+1/n$ and
$\beta_2^n=\beta_1^*-1/n$.
And for exactly-specified parameters $i=3,\cdots,\ns+1$: 
$A_i^n= A^*_{i-1}$,
$b_i^n= b^*_{i-1}$,
$c_i^n= c^*_{i-1}$,
$\alpha_i^n= \alpha^*_{i-1}$,
$\beta_i^n= \beta^*_{i-1}$.
Recall the construction of $\call_{2,r}$ loss, we will have
\begin{align*}
    \call_{2,r}{(G_n,\Gs)}=
    \frac{1}{n^{r+1}}
    +\frac{2}{n^r}
    =\mathcal{O}(n^{-r}).
\end{align*}
Now recall the Taylor expansion
in equation \eqref{eq:done_taylor_expension} for $f_{G_n}(x)-f_{\Gs}(x)$:
\begin{align*}
% \label{eq:dtwor_taylor_expension}
    &f_{G_n}(x)-f_{\Gs}(x)=
    \sum_{i=1}^2
    \left[
    \frac{(\alpha_i^n)^{\top}x+\beta_i^n}{1+
    \exp(-x^{\top}A^n_ix-(b^n_i)^{\top}x-c_i^n)
    }
    -\frac{(\alpha_1^*)^{\top}x+\beta_1^*}{1+
    \exp(-c_i^n)
    }
    \right]
    :=\Ione_n
    % \nonumber
    \\&
    +
    % \sum_{j=1}^{\bn}
    \left[
    \sum_{i=1}^2
    \frac{1}{1+
    \exp(-c_i^n)
    }
    % \cdot
    % h(x,\eta^*_i)
    -
    \frac{1}{1+
    \exp(-c_1^*)
    }
    \right]
    \cdot
    \left[(\alpha_1^*)^{\top}x+\beta_1^*\right]
    :=\Itwo_n\nonumber
    \\&
    +
    \sum_{i=3}^{\ns+1}
    \left[
    % \sum_{i\in\calAj}
    \frac{(\alpha_i^n)^{\top}x+\beta_i^n}{1+
    \exp(-x^{\top}A^n_ix-(b^n_i)^{\top}x-c_i^n)
    }
    -
    \frac{(\alpha_{i-1}^*)^{\top}x+\beta_{i-1}^*}{1+
    \exp(-x^{\top}A^*_{i-1}x-(b^*_{i-1})^{\top}x-c_{i-1}^*)
    }
    \right]
    :=\Ithree_n.
    % \nonumber
\end{align*}
From our construction for the sequence $(A_i^n,b_i^n,c_i^n,\alpha_i^n,\beta_i^n)_{i=1}^{\ns+1}$, its easy to verify that 
% $\Itwo_n=\Ithree_n=0$, thus
\begin{align*}
    f_{G_n}(x)-f_{\Gs}(x)= 
    \sum_i^2
    \left[ 
    \frac{\beta_i^n}{1+\exp(-c^n_i)}
    -
    \frac{\beta_1^*}{1+\exp(-c^n_i)}
    \right]
    +
    \frac{1}{n^{r+1}}\left[(\alpha_1^*)^{\top}x+\beta_1^*\right]
    =
    \frac{\beta_1^*}{n^{r+1}}.
\end{align*}
Based on the above result, we conclude that $ {[f_{G_n}(x) - f_{\Gs}(x)]}/{\call_{2,r}(G_n, \Gs)} \to 0 $ for almost every $ x $. As a result, $ {\|f_{G_n} - f_{\Gs}\|_{L^2(\mu)}}{\call_{2,r}(G_n, \Gs)} \to 0 $ as $ n \to \infty $. This establishes the claim stated in equation \eqref{pfeq:dtwor_loss_linear_aim}.

\end{proof}

\subsection{Proof of Theorem \ref{thm:dthree_loss}}
\label{proof:dthree_loss}

\begin{proof}
Following from the result of 
Theorem \ref{thm:function-convergence-misspecified}, it is sufficient to show that the following inequality holds true:
    \begin{align}
    \label{pfeq:dthree_loss_aim}
        \inf_{\G\in\calm_N(\Theta)}
        \frac{\Vert f_{G}-f_{\lG}\Vert_{L^2(\mu)}}
        {\lthree}
        >0,
    \end{align}
for any mixing measure ${\lG\in\lcalm_N(\Theta)}$.
% To prove the above inequality, similar to the proof in Appendix \ref{proof:done_loss}, we consider two cases for the denominator $\lthree$: either it lies within a ball $B(0, \varepsilon)$ where the loss is sufficiently small, or it falls outside this region, where $\lthree$ will not vanish.
% However, since the arguments for the global part remain the same (up to some changes of notations) for the over-specified setting, they will be omitted.
To prove the above inequality, we follow a similar approach to the proof in Appendix \ref{proof:done_loss}, dividing the analysis into a local part and a global part. However, since the arguments for the global part remain the same (up to some notational changes) in the over-specified setting, they are omitted.

% \textbf{Local part:}
% At first, we focus on that 
Therefore, for an arbitrary mixing measure $ \lG := \sum_{i=1}^k\frac{1}{1+\exp(-\bci)}\delta_{(\bai, \bbi,\betai)}\in\lcalm_N(\Theta) $, we focus exclusively on demonstrating that:
\begin{align}
\label{pfeq:dthree_local}
    \lim_{\varepsilon\to0}
    \inf_{G\in\calm_N(\Theta):\lthree\leq\varepsilon}
    \frac{\Vert f_{G}-f_{\lG}\Vert_{L^2(\mu)}}{\lthree}
    >0.
\end{align}
Assume, for contradiction, that the above claim does not hold. Then there exists a sequence of mixing measures
$G_n=\sum_{i=1}^{k}\frac{1}{1+\exp(-c_i^n)}\delta_{(A_i^n,b_i^n,\eta_i^n)}$ in
$\calm_N(\Theta)$ such that as $n\to\infty$, we get
\begin{align}
    \begin{cases}
        \call_{3n}:=\call_3(G_n,\lG)\to 0,\\
        \|f_{G_n}-f_{\lG} \|_{L^2(\mu)}/\call_{3n}\to0.
    \end{cases}
\end{align}
Let us denote by $\calA_i^n:=\calA_i(G_n)$ a Voronoi cell of $G_n$ generated by the $j$-th components of $\lG$. 
Since our analysis is asymptotic, we can assume that the Voronoi cells are independent of the sample size, i.e. $\calAj=\calA_i^n$.

In Dense Regime, since $G_n$ and $\lG$ have the same number of atoms $N$, and $\call_{3n} \to 0$, it follows that each Voronoi cell $\calA_i$ contains precisely one element for all $i \in [N]$. Without loss of generality, we assume $\calA_i = \{i\}$ for all $i \in [N]$. This ensures that $(A_{i}^n, b_{i}^n, c_{i}^n,\eta_i^n) \to (\bai, \bbi, \bci,\betai)$ as $n \to \infty$ for every $i \in [N]$.

Consequently, the Voronoi loss $\call_{3n}$ can be expressed as:  
\begin{align}
\call_{3n} := \sum_{i=1}^N 
\left( 
\|\Delta \bai^n\| + 
\|\Delta \bbi^n\| + 
|\Delta \bci^n | + 
\|\betai^n\| 
\right),
\end{align}
where the increments are given by:  
$
\Delta \bai^n = A_{i}^n - \bai, 
\Delta \bbi^n = b_{i}^n - \bbi, 
\Delta \bci^n = c_{i}^n - \bci, 
\Delta \betai^n = \eta_i^n - \betai.
$

% Let us recall that
%     \begin{align*}
%     \call_{1n}
%     &
%     :=\sum_{j=1}^{\bn}
%     \left|
%     \sum_{i\in\calAj} \frac{1}{1+\exp(-c^n_i)}-\frac{1}{1+\exp(-c_j^*)}
%     \right|
%     % \nonumber\\
%     +\sum_{j=1}^{\bn}
%     \sum_{i\in\calAj}
%     \left[
%     \|\Delta A^n_{ij} \|^2
%     +\|\Delta b^n_{ij} \|^2
%     +\|\Delta \eta^n_{ij} \|^2
%     \right]
%     % \nonumber
%     \\
%     &+\sum_{j=\bn+1}^{\ns}
%     \sum_{i\in\calAj}
%     \left[
%     \|\Delta A^n_{ij} \|
%     +\|\Delta b^n_{ij} \|
%     +|\Delta c^n_{ij} |
%     +\|\Delta \eta^n_{ij} \|
%     \right],
% \end{align*}
% where $\Delta A^n_{ij}:=A^n_i-A^*_j$, 
% $\Delta b^n_{ij}:=b^n_i-b^*_j$, 
% $\Delta c^n_{ij}:=c^n_i-c^*_j$, 
% $\Delta \eta^n_{ij}:=\eta^n_i-\eta^*_j$.
% Since $\call_{1n}\to0$ as $n\to0$, there are two different situation for parameter convergence:
% \begin{itemize}
%     \item For $j=1,\cdots,\bn$, parameters are over-fitted: $\sum_{i\in\calAj} \frac{1}{1+\exp(-c^n_i)}\rightarrow\frac{1}{1+\exp(-c_j^*)}$
%     and $(A^n_i,b^n_i,\eta^n_i)\rightarrow (A^*_j,b^*_j,\eta^*_j)$, $\forall i\in\calAj$;
%     \item For $j=\bn+1,\cdots,\ns$, parameters are exact-fitted: 
%     $(A^n_i,b^n_i,c^n_i,\eta^n_i)\rightarrow (A^*_j,b^*_j,c^*_j,\eta^*_j), \forall i\in\calAj$.
    
% \end{itemize}
% \begin{align*}
%     f_{G_n}(x)=\sum_{i=1}^{k}
%     \frac{1}{1+
%     \exp(-x^{\top}A^n_ix-(b^n_i)^{\top}x-c_i^n)
%     }\cdot
%     h(x,\eta^n_i)
% \end{align*}
% \begin{align*}
%     f_{\Gs}(x):=\sum_{i=1}^{\ns}
%     \frac{1}{1+
%     \exp(-x^{\top}A^*_ix-(b^*_i)^{\top}x-c_i^*)
%     }\cdot
%     h(x,\ei)
% \end{align*}

We now break the proof of the local part into the following three steps:

\textbf{Step 1 - Taylor expansion:}
In this step, we decompose the term $f_{G_n}(x) - f_{\lG}(x)$ using a Taylor expansion. First, let us denote
$\sigma(x,A,b,c):=\frac{1}{1+\exp(-x^{\top}Ax-b^{\top}x-c)}$,
then we have that
\begin{align}
\label{eq:dthree_taylor_expension}
    &f_{G_n}(x)-f_{\lG}(x)\nonumber\\
    % \nonumber
    % \\
    &=\sum_{i=1}^{N}
    \left[
    \frac{1}{1+
    \exp(-x^{\top}A^n_ix-(b^n_i)^{\top}x-c_i^n)
    }\cdot
    \cale(x,\eta^n_i)
    -
    \frac{1}{1+
    \exp(-x^{\top}\bai x-(\bbi)^{\top}x-\bci)
    }\cdot
    \cale(x,\betai)
    \right]
    \nonumber
    \\&
    =
    \sum_{i=1}^N
    \sum_{|\alpha|=1}
    \frac{1}{\alpha!}
    (\Delta\bai^n)^{\alpha_1}
    (\Delta\bbi^n)^{\alpha_2}
    (\Delta\bci^n)^{\alpha_3}
    (\Delta\betai^n)^{\alpha_4}
    % \sum_{|\tau|=1}
    % \frac{1}{\tau !}
    % (\Delta A_{ij}^n)^{\tau_1}
    % (\Delta b_{ij}^n)^{\tau_2}
    % (\Delta c_{ij}^n)^{\tau_3}
    % (\Delta \eta_{ij}^n)^{\tau_4}
    \frac{\partial^{|\alpha_1|+|\alpha_2|+|\alpha_3|}\sigma}{\partial A^{\alpha_1} \partial b^{\alpha_2} \partial c^{\alpha_3}}(x,\bai,\bbi,\bci)
    \frac{\partial^{|\alpha_4|} \cale}{\partial \eta^{\alpha_4}}(x,\betai)
    +R_1(x)    
    \nonumber
    \\&
    =
    \sum_{i=1}^N
    \sum_{|\alpha|=1}
    % \frac{1}{\alpha!}
    % (\Delta\bai^n)^{\alpha_1}
    % (\Delta\bbi^n)^{\alpha_2}
    % (\Delta\bci^n)^{\alpha_3}
    % (\Delta\betai^n)^{\alpha_4}
    S^n_{i,\alpha_1,\alpha_2,\alpha_3,\alpha_4}
% S^n_{i,\alpha_{1:4}}
    \frac{\partial^{|\alpha_1|+|\alpha_2|+|\alpha_3|}\sigma}{\partial A^{\alpha_1} \partial b^{\alpha_2} \partial c^{\alpha_3}}(x,\bai,\bbi,\bci)
    \frac{\partial^{|\alpha_4|} \cale}{\partial \eta^{\alpha_4}}(x,\betai)
    +R_1(x),    
\end{align}
where $R_1(x)$ is a Taylor remainder such that $R_1(x)/\call_{3n}\to0$ as $n\to\infty$.
Now we could denote 
\begin{align}
    % K^n_{j,i,\gamma_{1:3}}&= \frac{1}{\gamma!}
    % (\Delta A_{ij}^n)^{\gamma_1}
    % (\Delta b_{ij}^n)^{\gamma_2}
    % (\Delta \eta_{ij}^n)^{\gamma_3},~
    % j\in[\bn],i\in\calAj,\gamma\in\Real^{d\times d}\times\Real^d\times\Real^q
    % \label{pfeq:done_loss_notation1}
    % \\
    % L^n_j&=\sum_{i\in\calAj}
    % \frac{1}{1+
    % \exp(-c_i^n)
    % }
    % -
    % \frac{1}{1+
    % \exp(-c_j^*)
    % },~
    % j\in[\bn]
    % \label{pfeq:done_loss_notation2}
    % \\
    S^n_{i,\alpha_{1:4}}&=
    % \sum_{i\in\calAj}
    % \frac{1}{\tau !}
    % (\Delta A_{ij}^n)^{\tau_1}
    % (\Delta b_{ij}^n)^{\tau_2}
    % (\Delta c_{ij}^n)^{\tau_3}
    % (\Delta \eta_{ij}^n)^{\tau_4},
    \frac{1}{\alpha!}
    (\Delta\bai^n)^{\alpha_1}
    (\Delta\bbi^n)^{\alpha_2}
    (\Delta\bci^n)^{\alpha_3}
    (\Delta\betai^n)^{\alpha_4},
    ~
    i\in[N], 
    \alpha\in\Real^{d\times d}\times\Real^d\times\Real\times\Real^q
    \label{pfeq:dthree_loss_notation}
\end{align}
where
$\sum_{i=1}^4|\alpha_i|=1$.

\textbf{Step 2 - Non-vanishing coefficients:}
Now we claim that at least one among the ratios
$S^n_{i,\alpha_{1:4}}/\call_{3n}$
will not vanish as n goes to infinity.
We prove by contradiction that all of them converge to zero when $n\to 0$:
\begin{align*}
\frac{S^n_{i,\alpha_{1:4}}}{\call_{3n}}\to0.
\end{align*}
for any $i\in[N], 
    \alpha\in\Real^{d\times d}\times\Real^d\times\Real\times\Real^q$ 
such that
$\sum_{i=1}^4|\alpha_i|=1$. 

Now, consider 
$1\leq i\leq N$,
for arbitrary
$ u,v \in[d]$,
let $\alpha_1=e_{d\times d,uv}, \alpha_2=0_{d}, \alpha_3=0$ and $\alpha_4=0_{q}$,
we will have 
\begin{align*}
\frac{1}{\call_{3n}}
\sum_{i=1}^N
\left|
\Delta (\bai^n)^{(uv)}
\right|
=
\frac{1}{\call_{3n}}
\left| 
S^n_{i,e_{d\times d,uv},0_d,0,0_q}
\right|
\to 0,~ n\to \infty.
\end{align*}
Then by taking the summation of the term with $u,v\in[d]$, we will have
\begin{align*}
    \frac{1}{\call_{3n}}
    \sum_{i=1}^N
    \|\Delta \bai^n \|_1
    =
    \frac{1}{\call_{3n}}
    \sum_{i=1}^{N}
    \sum_{u=1}^d
    \sum_{v=1}^d
    | S^n_{i,e_{d\times d,uv},0_d,0,0_q}|\to0.
\end{align*}
Recall the topological equivalence between $L_1$-norm and $L_2$-norm on finite-dimensional vector space over $\Real$, we will have
\begin{align}
\label{pfeq:dthree_loss_coefficient1}
    \frac{1}{\call_{3n}}
    \sum_{i=1}^N
    \|\Delta \bai^n \|
    \to0.
\end{align}
Following a similar argument, 
since
\begin{align*}
    \frac{1}{\call_{3n}}
    % \sum_{j=\bn+1}^{\ns}
    % \sum_{i\in\calAj}
    \sum_{i=1}^N
    \|\Delta \bbi^n \|_1
    &=
    \frac{1}{\call_{3n}}
    \sum_{i=1}^N
    \sum_{u=1}^d
    | S^n_{i,0_{d\times d},e_{d,u},0,0_q}|\to0,
    \\
    \frac{1}{\call_{3n}}
    % \sum_{j=\bn+1}^{\ns}
    % \sum_{i\in\calAj}
    \sum_{i=1}^N
    |\Delta \bci^n |
    &=
    \frac{1}{\call_{3n}}
    % \sum_{j=\bn+1}^{\ns}
    % \sum_{u=1}^d
    \sum_{i=1}^N
    | S^n_{i,0_{d\times d},0_{d},1,0_q}|\to0,
    \\
    \frac{1}{\call_{3n}}
    % \sum_{j=\bn+1}^{\ns}
    % \sum_{i\in\calAj}
    \sum_{i=1}^N
    \|\Delta \betai^n \|_1
    &=
    \frac{1}{\call_{3n}}
    % \sum_{j=\bn+1}^{\ns}
    \sum_{i=1}^N
    \sum_{w=1}^q
    | S^n_{i,0_{d\times d},0_{d},0,e_{q,w}}|\to0,
\end{align*}
we obtain that 
\begin{align}
\label{pfeq:dthree_loss_coefficient2}
    \frac{1}{\call_{3n}}
\sum_{i=1}^N
    \|\Delta \bbi^n \|\to0,~
    \frac{1}{\call_{3n}}
\sum_{i=1}^N
    |\Delta \bci^n |\to0,~
    \frac{1}{\call_{3n}}
\sum_{i=1}^N
    \|\Delta \betai^n \|\to0.
\end{align}
Now taking
the summation of limits in equations
\eqref{pfeq:dthree_loss_coefficient1} - 
\eqref{pfeq:dthree_loss_coefficient2},
we could deduce that
    \begin{align*}
    1=\frac{\call_{3n}}{\call_{3n}}
    =
    \frac{1}{\call_{3n}}\cdot
    \sum_{i=1}^N
    \left(
    \|\Delta \bai^n \|+
    \|\Delta \bbi^n \|+
    |\Delta \bci^n |+
    \|\Delta \betai^n \|
    \right)
    \to 0,
\end{align*}
as $n\to\infty$,
which is a contradiction.
Thus, at least one among the ratios
$S^n_{i,\alpha_{1:4}}/\call_{3n}$
must not approach zero as $n\to\infty$.
Let us denote by $m_n$ the maximum of the absolute values of those elements.
It follows from the previous result that $1/m_n\not\to \infty$ as $n\to\infty$.

\textbf{Step 3 - Application of Fatou’s lemma:} 
In this step, we apply Fatou's lemma to obtain the desired inequality in equation \eqref{pfeq:dthree_local}.
Recall in the beginning we have assumed that
$
        \|f_{G_n}-f_{\lG} \|_{L^2(\mu)}/\call_{3n}\to0
$
and the topological equivalence between $L_1$-norm and $L_2$-norm on finite-dimensional vector space over $\Real$,
we obtain $
        \|f_{G_n}-f_{\lG} \|_{L^1(\mu)}/\call_{3n}\to0
$. 
By applying the Fatou’s lemma, we get
\begin{align*}
    0=\lim_{n\to\infty}\frac{\|f_{G_n}-f_{\lG} \|_{L^1(\mu)}}{m_n \call_{3n}}
    \geq
    \int\liminf_{n\to\infty}
    \frac{|f_{G_n}(x)-f_{\lG}(x)|}{m_n \call_{3n}}d\mu(x)
    \geq 0.
\end{align*}
This result suggests that for almost every $x$, 
\begin{align}
\label{pfeq:dthree_fatou}
    \frac{f_{G_n}(x)-f_{\lG}(x)}{m_n \call_{3n}}\to 0.
\end{align}
% recall equation \eqref{pfeq:done_decomposition}, we deduce that
% \begin{align*}
%     \frac{f_{G_n}(x)-f_{\Gs}(x)}{m_n \call_{1n}}
%     &=
%     \sum_{j=1}^{\bn}
%     \sum_{i\in\calAj}
%     \sum_{|\gamma|=1}^2
%     \frac{K^n_{j,i,\gamma_{1:3}}}{m_n \call_{1n}}
%     \frac{\partial^{|\gamma_1|+|\gamma_2|}\sigma}{\partial A^{\gamma_1}\partial b^{\gamma_2}}
%     (x,0_{d\times d},0_d,c_i^n) 
%     \frac{\partial^{|\gamma_3|}h}{\partial\eta^{\gamma_3}}
%     (x,\eta^*_j)
%     +\sum_{j=1}^{\bn}
%     \frac{L^n_j}{m_n \call_{1n}}
%     \cdot
%     h(x,\ej)\nonumber
%     \\&
%     +\sum_{j=\bn+1}^{\ns}
%     \sum_{|\tau|=1}
%     \frac{T^n_{j,\tau_{1:4}}}{m_n \call_{1n}}
%     \frac{\partial^{|\tau_1|+|\tau_2|+|\tau_3|}\sigma}{\partial A^{\tau_1} \partial b^{\tau_2} \partial c^{\tau_3}}(x,A^*_j,b^*_j,c^*_j)
%     \frac{\partial^{|\tau_4|} h}{\partial \eta^{\tau_4}}(x,\ej)
%     +\frac{R_1(x)}{m_n \call_{1n}}
%     +\frac{R_2(x)}{m_n \call_{1n}}.    
% \end{align*}
Let us denote 
\begin{align*}
    % \frac{K^n_{j,i,\gamma_{1:3}}}{m_n \call_{1n}}\to k_{j,i,\gamma_{1:3}},~~
    % \frac{L^n_j}{m_n \call_{1n}}\to l_j,~~
    \frac{S^n_{i,\alpha_{1:4}}}{m_n \call_{3n}}\to s_{i,\alpha_{1:4}},
\end{align*}
as $n\to\infty$ with a note that at least one among the limits $s_{i,\alpha_{1:4}}$
is non-zero.
Then from equation \eqref{pfeq:dthree_fatou},
we will have
\begin{align*}
    \sum_{i=1}^N
    \sum_{|\alpha|=1}
    s^n_{i,\alpha_1,\alpha_2,\alpha_3,\alpha_4}
% S^n_{i,\alpha_{1:4}}
    \frac{\partial^{|\alpha_1|+|\alpha_2|+|\alpha_3|}\sigma}{\partial A^{\alpha_1} \partial b^{\alpha_2} \partial c^{\alpha_3}}(x,\bai,\bbi,\bci)
    \frac{\partial^{|\alpha_4|} \cale}{\partial \eta^{\alpha_4}}(x,\betai)
    =0,   
\end{align*}
for almost every $x$.
Note that the expert function $\cale(\cdot,\eta)$ is weakly identifiable, then the above equation implies that
\begin{align*}
     s_{i,\alpha_{1:4}}=0,
\end{align*}
for any $i\in[N]$,   
    $(\alpha_1,\alpha_2,\alpha_3,\alpha_4)\in\bbN^{d\times d}\times\bbN^d\times\bbN\times\bbN^q$
such that 
$\sum_{i=1}^4|\alpha_i|=1 $.
This violates that at least one among the limits in the set $\{s_{i,\alpha_{1:4}} \}$ is different from zero.

Thus, we obtain the local inequality in equation \eqref{pfeq:dthree_local}. 
\end{proof}

\subsection{Convergence to a Single Sigmoid under Vanishing Gating}
\label{appendix:cov_single_sigmoid}

In this subsection, we formally justify the statement that the limit of a sum of two sigmoid functions can reduce to a single sigmoid function only when one of the experts effectively vanishes. This result is used in Section~\ref{sec:problem-setup} to explain why the regression function learned by the over-specified model converges to a single-expert form. The key condition for this collapse is that the gating parameters of the redundant expert must vanish asymptotically. We state and prove this result below.

\begin{proposition}
\label{lem:single_sigmoid_limit}
Let $(\widehat{A}_i^n, \widehat{b}_i^n, \widehat{c}_i^n) \to (A_i^*, b_i^*, c_i^*)$ in probability for $i=1,2$. Suppose that for $\mu$-almost every $x \in \mathbb{R}^d$,
\begin{align*}
    \sum_{i=1}^2
    \frac{1}{1+\exp\left(-x^\top \widehat{A}_i^n x - (\widehat{b}_i^n)^\top x - \widehat{c}_i^n\right)}
    \xrightarrow{P}
    \frac{1}{1+\exp\left(-x^\top A_1^* x - (b_1^*)^\top x - c_1^*\right)}.
\end{align*}
Then it must hold that $A_2^* = 0_{d \times d}$, $b_2^* = 0_d$, and $c_2^* = +\infty$ (or tends to $+\infty$ in probability), i.e., the second sigmoid term vanishes in the limit.
\end{proposition}

\begin{proof}
Let us denote the sigmoid function as $\sigma(z) = 1 / (1 + \exp(-z))$, and consider the limiting expressions:
\begin{align*}
    \sigma_1(x) &= \frac{1}{1 + \exp(-x^\top A_1^* x - (b_1^*)^\top x - c_1^*)}, \\
    \sigma_2(x) &= \frac{1}{1 + \exp(-x^\top A_2^* x - (b_2^*)^\top x - c_2^*)}.
\end{align*}
Suppose toward a contradiction that $(A_2^*, b_2^*) \neq (0, 0)$ or $c_2^*$ is finite. Then there exists a set of positive $\mu$-measure where $\sigma_2(x)$ is bounded away from $0$. In that case, the sum $\sigma_1(x) + \sigma_2(x)$ must strictly exceed $\sigma_1(x)$ for such $x$, and hence cannot equal $\sigma_1(x)$.

But the convergence in the hypothesis implies that for $\mu$-almost every $x$,
\begin{align*}
    \sigma_1^n(x) + \sigma_2^n(x) \xrightarrow{P} \sigma_1(x),
\end{align*}
which can only hold if $\sigma_2^n(x) \to 0$ in probability. This, in turn, requires that the argument of the exponential in $\sigma_2^n(x)$ tends to $+\infty$ in probability for almost every $x$, i.e.,
\[
x^\top A_2^* x + (b_2^*)^\top x + c_2^* \to +\infty.
\]
This is only possible if $A_2^* = 0$, $b_2^* = 0$, and $c_2^* \to +\infty$. Hence, the second sigmoid term vanishes in the limit.
\end{proof}

\section{Additional Results}
\label{appsec:additional-results}

\subsection{Sample Complexity under the Partially Quadratic Score Function}
\label{appendix:analysis_quadratic_mono}
% \begin{align*}
%     F(x;A,c,\alpha,\beta)=\sigma(x,A,c)h(x,\alpha, \beta)
%     % \\
%     =
%     \frac{1}{1+\exp
%     \left(
%     -x^{\top}Ax-c
%     \right)
%     }
%     h
%     \left(
%     \alpha^{\top}x+\beta
%     \right)
% \end{align*}
% Define:
% $\sigma(x,A,c) = \frac{1}{1 + \exp(-z)}$, 
% where $z = x^{\top}Ax  + c$.
% $h(x,\alpha,\beta) = h(\alpha^{\top}x + \beta)$.

% $\pi(x,\theta^*_i)=x^{\top}A^*_ix+c_i^*$

% In this section, we consider quadratic monomial gating that in the regression function

In this section, we proceed the analysis of the sigmoid gating MoE with partially quadratic affinity score function based on the regression frame work in equation \eqref{eq:quadratic_MoE}, the corresponding regression function is redefined as follows:
\begin{align}
    {\tilde{f}}_{\Gs}(x):=\sum_{i=1}^{\ns}\frac{1}{1+
    % \exp(-x^{\top}\alpha_i^* x)
    \exp(-x^{\top}A^*_ix-c_i^*)
    }\cdot
    \cale(x,\ei)
\end{align}
% where the expert function $h(x,\eta)$ is of parametric form, specified by parameter $\eta\in\Real^q$.
% where $\tilde{\pi}(x,\theta^*_i)$ is of quadratic monomial form, i.e. $\tilde{\pi}(x,\theta^*_i)=x^{\top}A^*_ix+c_i^*$, 
where $\{(A_i, c_i, \eta_i)\in\tTheta,i=1,\cdots,N \}$ is the set of learnable parameters,
$\tTheta\subseteq\mathbb{R}^{d\times d}\times\Real\times\Real^{q}$ stands for the parameter space
and 
$\tcalm_N(\tTheta):=\{ \tG=\sum_{i=1}^{N^\prime}\frac{1}{1+\exp(-c_i)}\delta_{(A_i,\eta_i)}:1\leq N^\prime \leq N, (A_i,c_i,\eta_i)\in\tTheta \}$
is the set of all mixing measures with at most $N$ atoms.
We assume $N>\ns$.
And for simplicity, we still denote $\tTheta$ as $\Theta$ , $\tcalm_N(\tTheta)$ as $\calm_N(\Theta)$ and $\tG$ as $G$ in the following.

% Here the first-degree monomial term  $b^{\top}x$ was removed from the scoring function.
% Consequently, the least squares estimator in this setting adapts accordingly to
In contrast to the fully quadratic score function, the first-degree monomial term $ b^\top x $ has been omitted from the scoring function. Consequently, the least squares estimator in this context is modified as follows:
\begin{align}
\label{eq:lse-mono}
    \tGn:=\argmin_{G\in\calm_N(\Theta)}
    \sum_{i=1}^n
    \left(
    Y_i-\tilde{f}_{G}(X_i)
    \right)^2.
\end{align}
% where 
% $\calm_N(\Theta):=\{ G=\sum_{i=1}^{k^\prime}\frac{1}{1+\exp(-c_i)}\delta_{(A_i,b_i,\eta_i)}:1\leq k^\prime \leq k, (A_i,b_i,c_i,\eta_i)\in\Theta \}$
% is the set of all mixing measures with at most $k$ atoms.
% We assume $k>\ns$.
% Based on the above estimator, Theorem \ref{thm:function-convergence-specified-mono} establishes the convergence rate of the regression estimate $ \tilde{f}_{\tG_n}(\cdot) $ to the true regression function $ \tilde{f}_{\tG_*}(\cdot) $.
% \begin{theorem}[Regression Estimation Rate]
% \label{thm:function-convergence-specified-mono}
% With the least squares estimator 
% $\tGn$ defined in equation \eqref{eq:lse-mono}, 
% the regression estimator $\tilde{f}_{\tGn}$ admits the following rate of convergence to $\tilde{f}_{\tGs}$:
%     \begin{align}
%     \label{eq:convergence-rate-mono}
%         \Vert \tilde{f}_{\tGn}-\tilde{f}_{\tGs}\Vert_{L^2(\mu)}=
%         \widetilde{\mathcal{O}}_P\left(
%         n^{-1/2}
%         % \sqrt{
%         % \frac{\log(n)}{n} 
%         % }
%         \right).
%     \end{align}
% \end{theorem}
% The proof is in Appendix \ref{proof:function-convergence-specified-mono}.
% From the bound in \eqref{eq:convergence-rate-mono}, it can be concluded that the regression estimation rate remains parametric with respect to the sample size. This aligns with the results in Theorem \ref{thm:function-convergence-specified}, where a quadratic polynomial gating function is employed in the MoE-type regression framework.

% \subsubsection{Convergence behavior of the mixture weights}
\subsubsection{Convergence of the Regression Function Estimator}

Similar with the polynomial case, we fit the ground-truth MoE model with a mixture of $N>\ns$, 
there must be some true atoms $(A_i^*,\eta_i^*)$ fitted by more than one component. We over-specify the true MoE model by a mixture of $N$ experts where $N>\ns$.
There exist some atoms $(\Asi,\etasi)$ approximated by at least two fitted components, the over-specified atoms.
We assume that $(\hAin,\hetain)\rightarrow(\Asone,\etasone)$ for $i=1,2$, in probability.
Then, the term $\| \tf_{\tGn}-\tf_{\Gs} \|_{L^2(\mu)}\rightarrow 0$ only when
\begin{align*}
    \sum_{i=1}^2
    \frac{1}{1+\exp
    \left(
    -x^{\top}\hAin x-\hcin
    \right)
    }
    \rightarrow
    \frac{1}{1+\exp
    \left(
    -x^{\top}\Asone x-\csone
    \right)
    },
\end{align*}
as $n\rightarrow\infty$ for $\mu$-almost every $x$, which occurs only when $\Asone=0_d$.
% Then we may encounter
Following the approach outlined in Section \ref{sec:problem-setup}, we will partition our analysis into two complementary regimes of the gating parameters:
\begin{enumerate}
% [\title{1}]
    \item[(i)] \emph{Sparse Regime}: all the over-specified gating parameters are zero: $\Asi=0_{d\times d}$ ;
    \item[(ii)] \emph{Dense Regime}: at least one among the over-specified gating parameters is non-zero: $\Asi\neq0_{d\times d}$.
\end{enumerate}
We now present the convergence behavior of the regression function estimator under each of the two regimes, respectively.

% In the Sparse Regime, we have $\tGn\to\Gs$.

% In contrast, in other cases, $\tGn\to\llG\in\argmin_{\G\in\calm_N(\Theta)\setminus\calm_{N^*}(\Theta)}\| \tf_{\G}-\tf_{\Gs} \|_{L_2{(\mu)}}$,
% which is generally different from the true $\Gs$.

% $$\tGn\to\grave{G}\in\argmin_{\G\in\calm_N(\Theta)\setminus\calm_{N^*}(\Theta)}\| \tf_{\G}-\tf_{\Gs} \|_{L_2{(\mu)}}$$

% $$\tGn\to\underline{G}\in\argmin_{\G\in\calm_N(\Theta)\setminus\calm_{N^*}(\Theta)}\| \tf_{\G}-\tf_{\Gs} \|_{L_2{(\mu)}}$$

% $$\tGn\to\overline{\overline{G}}\in\argmin_{\G\in\calm_N(\Theta)\setminus\calm_{N^*}(\Theta)}\| \tf_{\G}-\tf_{\Gs} \|_{L_2{(\mu)}}$$

% $$\tGn\to\dot{G}\in\argmin_{\G\in\calm_N(\Theta)\setminus\calm_{N^*}(\Theta)}\| \tf_{\G}-\tf_{\Gs} \|_{L_2{(\mu)}}$$

% \subsection{Proof of Theorem \ref{thm:function-convergence-specified-mono}}
% \label{proof:function-convergence-specified-mono}

\begin{proposition}
\label{thm:function-convergence-specified-mono}
Under the sparse regime and with the least squares estimator 
$\tGn$ defined in equation \eqref{eq:lse-mono}, 
the regression estimator $\tf_{\tGn}$ admits the following rate of convergence to $\tf_{\Gs}$:
    \begin{align}
    \label{eq:function-convergence-specified-mono}
        \Vert \tf_{\tGn}-\tf_{\Gs}\Vert_{L^2(\mu)}=
        \mathcal{O}_P\left(
        \sqrt{
        {\log(n)}/{n} 
        }
        \right).
    \end{align}
\end{proposition}
The proof of Proposition \ref{thm:function-convergence-specified-mono} can be found in Appendix \ref{proof:function-convergence-specified-mono}.
The bound~\eqref{eq:function-convergence-specified-mono} reveals that the convergence rate of the regression function estimator $\tf_{\tGn}$ to the ground-truth regression function is of parametric order $\mathcal{O}_P(\sqrt{\log(n)/n})$ under the sparse regime. 
However, under the dense regime, such convergence does not happen as previously mentioned.
Instead, the regression function estimator $\tf_{\tGn}$ converges to the closest MoE with more than $\ns$ experts to $\tf_{\Gs}$ , that is, $\tf_{\llG}$ where 
$\llG\in\llcalm_N(\Theta):=\argmin_{\G\in\calm_N(\Theta)\setminus\calm_{N^*}(\Theta)}\| \tf_{\G}-\tf_{\Gs} \|_{L_2{(\mu)}}$.
Using similar arguments to Proposition \ref{thm:function-convergence-specified-mono}, we
can also determine the convergence behavior of the regression function estimator in the following corollary:
% ,which is generally different from the true $\Gs$.
% Then, $\tf_{\tGn}$ converge to the misspecified regression function $\tf_{\llG}$ rather than $\tf_{\Gs}$.
% , where $\llG\in{\calg}_k(\Theta):=\argmin_{\tG\in\calm_N(\Theta)\setminus\calg_{k^*}(\Theta)}\| \tf_{\tG}-\tf_{\tGs} \|_{L_2{(\mu)}}$.
\begin{corollary}
\label{crl:function-convergence-misspecified-mono}
Under the dense regime and with the least squares estimator 
$\tGn$ defined in equation \eqref{eq:lse-mono}, 
the regression estimator $\tf_{\tGn}$ admits the following rate of convergence to $\tf_{\llG}$:
    \begin{align*}
    \inf_{\llG\in\llcalm_N(\Theta)}
        \Vert \tf_{\tGn}-\tf_{\llG}\Vert_{L^2(\mu)}=
        \mathcal{O}_P\left(
        \sqrt{
        {\log(n)}/{n} 
        }
        \right).
    \end{align*}
\end{corollary}
Similar with what we have discussed in Section \ref{sec:problem-setup}, given the results of Propositions \ref{thm:function-convergence-specified-mono} and \ref{crl:function-convergence-misspecified-mono} and a loss function $\call(G, \Gs)$ among parameters satisfying $\call(\tGn, \Gs)\lesssim \Vert \tf_{\tGn}-\tf_{\Gs}\Vert_{L^2(\mu)}$,
we will have that $\call(\tGn, \Gs)=\mathcal{O}_P(\sqrt{{\log(n)}/{n} })$.
% Then we could deduce the expert convergence rate that we will study in the following.
From this, we can derive the expert convergence rate, which will be analyzed in the following.

\subsubsection{Sparse Regime of Gating Parameters}

Recall that under the sparse regime, all the over-specified parameters $\Asi=0_{d\times d}$.
Assume $\{\Asi\}_{i=1}^{\bn}$ are over-specified parameters, 
i.e.,
those fitted by at least two estimators,
where $1\leq\bn\leq\ns$. And the remaining parameters fitted by exactly one estimator $\{\Asi \}_{\bn+1}^{\ns}$:
\begin{align*}
    \underbrace{A^*_1,\cdots,A^*_{\bn}}_{\text{over-specified}},~
    \underbrace{A^*_{\bn+1},\cdots,A^*_{\ns}}_{\text{exactly-specified}}.
\end{align*}
% As mentioned in Section \ref{sec:sample-efficiency}, in order to obtain the expert convergence rate,
% we will decompose the difference $\tf_{\tGn}-\tf_{\Gs}$ into a combination of linearly independent terms,
% which means that we need to consider the Taylor expansion to the product of the sigmoid gating function and the expert function given by
% \begin{align*}
%     F(x;A, c,\eta):=\sigma(x^{\top}Ax+c)\cdot \cale(x,\eta)
%     % \\
%     =
%     \frac{1}{1+\exp
%     \left(
%     -x^{\top}Ax  -c
%     \right)
%     }
%     \cdot
%     \cale
%     \left(
%     x,\eta
%     \right).
% \end{align*}
As discussed in Section \ref{sec:sample-efficiency}, to derive the expert convergence rate, we decompose the difference $ \tf_{\tGn} - \tf_{\Gs} $ into a sum of linearly independent terms. This requires considering the Taylor expansion of the product of the sigmoid gating function and the expert function, given by  
\begin{align*}
     F(x; A, c, \eta) := \sigma(x^{\top}Ax + c) \cdot \mathcal{E}(x, \eta)  
    = \frac{1}{1+\exp(-x^{\top}Ax - c)} \cdot \mathcal{E}(x, \eta).
\end{align*}
% where $\sigma(x,A, c)=1/(1+\exp
%     \left(
%     -x^{\top}Ax -c
%     \right))$
% is the sigmoid function.
% \subsubsection{Strong identifiability}
% \textcolor{red}{Could we still assume a Strong identifiability situation?}
% At first we consider that the derivatives of $F(x;A,b,c,\eta)$ with respect to parameters are linearly independent.
% At first we consider the Sparse Regime.
% This decomposition of regression function is expected to consist of linearly independent terms so that when $\Vert f_{\hGn}-f_{\Gs}\Vert_{L^2(\mu)}\to0$ as $n\to\infty$, the parameter discrepancies in the decomposition will also converge to zero, leading to the parameter and expert convergence.
% In order to make the parameter discrepancies converge to zero in align with $\Vert f_{\hGn}-f_{\Gs}\Vert_{L^2(\mu)}\to0$, 
% we need to calrify the linear independency relation w.r.t. $F(x; A, c, \eta)$, a basic expert with partially quadratic score in the mixture of experts.
% So we need to impose a so called strong
% identifiability condition given below 
To ensure that the parameter discrepancies converge to zero—thereby guaranteeing parameter and expert convergence—along with $\Vert f_{\hGn} - f_{\Gs} \Vert_{L^2(\mu)} \to 0 $, it is crucial to clarify the linear independence properties of $F(x; A, c, \eta) $, a fundamental expert component with a partially quadratic score in the mixture of experts model. Therefore, we impose a \textit{partial-strong identifiability} condition, defined as follows:
\begin{definition}[Partial-strong identifiability]
\label{def:strong_identifiability-mono}
    An expert function $\cale(x,\eta)$ is \textit{partial-strongly identifiable} if it is twice differentiable w.r.t its parameter $\eta$ for $\mu$-almost all $x$ and, for any positive integer $\ell$
    and any pair-wise distinct choices of parameters 
    $\left\{ (A_i,  c_i,\eta_i) \right\}_{i=1}^{\ell}$, 
    the functions in the classes
    \begin{align*}
        \left\{  
        \frac{\partial^{|\gamma_1|+|\gamma_2| } F}{\partial A^{\gamma_1} \partial\eta^{\gamma_2}}
    (x,0_{d\times d}, c_i,\eta_i):
    i\in[\ell],
    % 1\leq
    \sum_{i=1}^2|\gamma_i|
    \in[2]
    % \leq 2
    % (\gamma_1,\gamma_2,\gamma_3)\in\Real^{d\times d}
        \right\}
    % \end{align*}
    \text{ and }
    % \begin{align*}
        \left\{  
        \frac{\partial
        % ^{|\tau_1|+|\tau_2|+|\tau_3|+|\tau_4|} 
        F}
        {\partial A^{\tau_1} \partial c^{\tau_2}\partial\eta^{\tau_3}}
    (x,A_i, c_i,\eta_i):
    i\in[\ell],
    \sum_{i=1}^3|\tau_i|
    =1
    % (\gamma_1,\gamma_2,\gamma_3)\in\Real^{d\times d}
        \right\}
    \end{align*}
    are linearly independent, for $\mu$-almost all $x$,
    where $(\gamma_1,\gamma_2 )\in\bbN^{d\times d} \times\bbN^q$ and 
    $(\tau_1,\tau_2,\tau_3)\in\bbN^{d\times d} \times\bbN\times\bbN^q$.
\end{definition}
% \paragraph{Example.}
% \textcolor{red}{Give detailed examples for strongly identifiable and non-strongly identifiable.}
% \textbf{Voronoi loss.} 
% Given an arbitrary mixing measure $G$ with $\kp\leq k$ atoms, we distribute its atoms across the Voronoi cells $\{\mathcal{A}_j\equiv
%     \mathcal{A}_j(G),j\in[\ns] \}$ generated by the atoms of $\Gs$, where
% \begin{align}
%     \mathcal{A}_j:=
%     % \mathcal{A}_j(G):=
%     \left\{
%     i\in[\kp]:
%     \| \omega_i-\omega_j^* \|
%     \leq
%     \| \omega_i-\omega_{\ell}^* \|,
%     \forall \ell\neq j
%     \right\},
%     % j=1,\cdots,\ns
% \end{align}
% for $j=1,\cdots,\ns$, 
% where $\omega_i:=(A_i,b_i,\eta_i)$ and
% $\omega^*_j:=(A^*_j,b^*_j,\eta^*_j)$.

\textbf{Examples.}
% Similar with Section \ref{sec:sparse_regime}, when we consider two-layer neural networks of the form $\mathcal{E}(x,(\alpha,\beta,\lambda))=\lambda\phi(\alpha^{\top}x+\beta)$, where $\phi$ is some activation function and $(\alpha,\beta,\lambda)\in\mathbb{R}^d\times\mathbb{R}\times\mathbb{R}$. 
% It can be verified that if $\phi$ is $\relu$ or $\gelu$ function, $\alpha\neq0_d$, and $\lambda\neq 0$, then the function $x\mapsto\mathcal{E}(x,(\alpha,\beta,\lambda))$ is partial-strongly identifiable.
Similar to Section \ref{sec:sparse_regime}, we consider two-layer neural networks of the form  
$
\mathcal{E}(x; (\alpha, \beta, \lambda)) = \lambda \phi(\alpha^{\top}x + \beta),
$
where $ \phi $ is an activation function and $ (\alpha, \beta, \lambda) \in \mathbb{R}^d \times \mathbb{R} \times \mathbb{R} $.  
It can be verified that if $ \phi $ is the ReLU or GELU function, and if $ \alpha \neq 0_d $ and $ \lambda \neq 0 $, then the function $ x \mapsto \mathcal{E}(x; (\alpha, \beta, \lambda)) $ is partial-strongly identifiable.

% However, different from the fully quadratic case, since now the term $b^{\top}x$ is removed, then the strong identifiability condition in Definition \ref{def:strong_identifiability-mono} 
% still hold even when the expert function is of polynomial form $\mathcal{E}(x,(\alpha,\beta))=(\alpha^{\top}x+\beta)^p$ for some $p\in\mathbb{N}$.
% Now $F(x;A,c,\alpha,\beta)=1/[1+\exp\left(-x^{\top}Ax -c\right)]\cdot (\alpha^{\top}x+\beta)^p$.
% Here we could notify that the sigmoid function only has a second order term $-x^{\top}Ax$ with a constant, and the expert function only has a first order term $\alpha^{\top}x$ with a constant. 
% This special structure ensures that the interaction with in the parameters dismiss in equation \eqref{eq:PDE}.
However, unlike the fully quadratic score case, where the term $ b^{\top}x $ is present, its removal allows the partial-strong identifiability condition in Definition \ref{def:strong_identifiability-mono} to remain valid even when the expert function follows a polynomial form, given by $ \mathcal{E}(x; (\alpha, \beta)) = (\alpha^{\top}x + \beta)^p $ for some $ p \in \mathbb{N} $.  
Now, the function takes the form  
$F(x;A,c,\alpha,\beta)=1/[1+\exp\left(-x^{\top}Ax -c\right)]\cdot (\alpha^{\top}x+\beta)^p$.
Notably, the sigmoid function contributes only a second-order term $ -x^{\top}Ax $ along with a constant, while the expert function contains only a first-order term $ \alpha^{\top}x $ and a constant. This distinct structure ensures that parameter interactions disappear as PDEs in equation \eqref{eq:PDE}.

% $\cale\left(x,(\alpha_j^*, \beta_j^*)\right)=\varphi\left((\alpha_j^*)^{\top}x+\beta_j^*\right)$ and $\varphi$ is a scalar function, i.e.
% \begin{align*}
%     &F(x;A,c,\alpha,\beta)=\sigma(x,A ,c)
%     \cdot
%     \cale(x,\alpha, \beta)
%     % \\
%     % &
%     =
%     \frac{1}
%     {1+\exp\left(-x^{\top}Ax -c\right)}
%     \cdot
%     \varphi
%     \left(
%     \alpha^{\top}x+\beta
%     \right).
% \end{align*}
% Here we could notify that the sigmoid function only has a second order term $-x^{\top}Ax$ with a constant, and the expert function only has a first order term $\alpha^{\top}x$ with a constant. 
% This special structure ensures that the interaction with in the parameters dismiss.
% , even some expert parameters $\alpha_i^*$,
% where $i\in\calAj,j\in[\ns]$ in the over-specified gating parameters Voronoi cell equals $0_d$ where $\varphi$ are input independent 
% or  $\varphi$ is a polynomial of the form $\varphi(z)=z^p$, for $p\in\bbN$.

\textbf{Voronoi loss.}
For a mixing measure $G$ with $1\leq \np\leq N$ atoms, we allocate its atoms across the Voronoi cells $\{\mathcal{A}_j\equiv
    \mathcal{A}_j(G),j\in[\ns] \}$ generated by the atoms of $\Gs$, where 
\begin{align}
    \mathcal{A}_j:=
    \left\{
    i\in[\np]:
    \| \theta_i-\theta_j^* \|
    \leq
    \| \theta_i-\theta_{\ell}^* \|,
    \forall \ell\neq j
    \right\},
\end{align}
with $\theta_i:=(A_i,\eta_i)$ and
$\theta^*_j:=(A^*_j,\eta^*_j)$ for all $j\in[\ns]$.
Similar to the fully quadratic score case, we define the Voronoi loss function as
\begin{align}
\label{eq:done-dfour-mono}
    % &
    \lfour&:=\sum_{j=1}^{\bn}
    \left|
    \sum_{i\in\calAj} \frac{1}{1+\exp(-c_i)}-\frac{1}{1+\exp(-c_j^*)}
    \right|
    \nonumber
    % \\
    % &
    +\sum_{j=1}^{\bn}
    \sum_{i\in\calAj}
    \left[
    \|\Delta A_{ij} \|^2
    % +\|\Delta b_{ij} \|^2
    +\|\Delta \eta_{ij} \|^2
    \right]
    \nonumber
    \\
    &
    +\sum_{j=\bn+1}^{\ns}
    \sum_{i\in\calAj}
    \left[
    \|\Delta A_{ij} \|
    % +\|\Delta b_{ij} \|
    +|\Delta c_{ij} |
    +\|\Delta \eta_{ij} \|
    \right]
\end{align}
where we denote $\Delta A_{ij}:=A_i-A_j^*$,
% $\Delta b_{ij}:=b_i-b_j^*$,
$\Delta c_{ij}:=c_i-c_j^*$
and $\Delta \eta_{ij}:=\eta_i-\eta_j^*$.
% In the statement above, if the Voronoi cell $ \mathcal{A}_j $ is empty, the corresponding summation term is conventionally defined to be zero.
% Additionally, the Voronoi loss function $ \mathcal{L}_1 $ can be computed efficiently, with a computational complexity of $ \mathcal{O}(N \times \ns) $.
In the aforementioned statement, if the Voronoi cell $ \mathcal{A}_j $ is empty, the associated summation term is conventionally set to zero. Additionally, the computation of the Voronoi loss function $ \mathcal{L}_4 $ is efficient, with a computational complexity of $ \mathcal{O}(N \times \ns) $.

% Given the above Voronoi loss function, we are ready to present the parameter convergence rate in Theorem~\ref{thm:done_loss-dfour-mono}.
With the Voronoi loss function established, we are now prepared to present the parameter convergence rate in Theorem~\ref{thm:done_loss-dfour-mono}.
\begin{theorem}
\label{thm:done_loss-dfour-mono}
    If the expert function $x\mapsto\cale(x,\eta)$ is partial-strongly identifiable, then the lower bound $\Vert\tf_{G}-\tf_{\Gs}\Vert_{L^2(\mu)}\gtrsim\lfour$ 
    holds true for any $G\in\calm_N(\Theta)$.
    % As a result, combined with the regression estimation rate in Proposition \ref{thm:function-convergence-specified-mono}, suggests that
    As a result, when combined with the regression estimation rate in Proposition \ref{thm:function-convergence-specified-mono}, this suggests that
    $$\call_4(\tGn,\Gs)=
        \mathcal{O}_P\left(
        \sqrt{
        {\log(n)}/{n} 
        }
        \right).$$
\end{theorem}
Proof of Theorem \ref{thm:done_loss-dfour-mono} is in Appendix \ref{proof:done_loss-dfour-mono}.
A few comments regarding the above result are in order.

\emph{(i) Parameter convergence rates:} 
% From the construction of the Voronoi loss $\mathcal{L}_4$, the convergence rates of estimating the over-specified parameters $\Asi,\etasi, i\in[\bn]$, share the same order $\mathcal{O}_P([{\log(n)}/{n}]^{\frac{1}{4}})$. At the same time, those for the exactly-specified parameters $\Asi,\etasi, \bn+1\leq i\leq \ns$ are faster, standing at the order $\mathcal{O}_P([{\log(n)}/{n}]^{\frac{1}{2}})$.
From the construction of the Voronoi loss $ \mathcal{L}_4 $, the estimation of over-specified parameters $ \Asi, \etasi $ for $ i \in [\bn] $ converges at a rate of $ \mathcal{O}_P([\log(n)/n]^{1/4}) $. In the mean time, the exactly-specified parameters $ \Asi, \etasi $ for $ \bn+1 \leq i \leq \ns $ exhibit a faster convergence rate of $ \mathcal{O}_P([\log(n)/n]^{1/2}) $.

\emph{(ii) Expert convergence rates:} 
% Since the expert function $\cale(\cdot,\eta)$ is twice differentiable w.r.t $\eta$ over a bounded domain, it is also a Lipschitz function w.r.t $\eta$. Therefore, by denoting $\tGn=\sum_{i=1}^{\tilde{N}_n}\frac{1}{1+\exp(-\tilde{c}^n_i)}\delta_{(\tilde{A}^n_i,\tilde{b}^n_i,\tilde{\eta}^n_i)}$, we deduce that
% \begin{align}
%     \label{eq:expert-rates_partial}
%     \sup_{x}|\cale(x,\widetilde{\eta}_i^n)-\cale(x,\eta_j^*)|
%     \leq
%     \tilde{L}
%     \|\widetilde{\eta}_i^n-\eta^*_j \|,
% \end{align}
% for any $i\in\calAj$, where $\tilde{L}\geq 0$ is a Lipschitz constant. The above bound indicates that the convergence rates of estimating the exactly-specified experts and over-specified experts are of orders $\mathcal{O}_P([{\log(n)}/{n}]^{\frac{1}{2}})$ and $\mathcal{O}_P([{\log(n)}/{n}]^{\frac{1}{4}})$, respectively. Thus, it takes the exactly-specified experts a polynomial number $\mathcal{O}(\epsilon^{-2})$ of data to reach the approximation error $\epsilon$, while the over-specified experts need a polynomial number $\mathcal{O}(\epsilon^{-4})$ of data to achieve the same error.
Because $\mathcal{E}(x,\eta)$ is twice differentiable (and thus Lipschitz) in $\eta$ over a bounded domain, we have, for any $i \in \mathcal{A}_j$,
$
\sup_{x} \bigl|\mathcal{E}(x,\widetilde{\eta}_i^n) - \mathcal{E}(x,\eta_j^*)\bigr|
\;\le\;
\tilde{L}\,\bigl\|\widetilde{\eta}_i^n - \eta_j^*\bigr\|,
$
where $\tilde{L}$ is a Lipschitz constant and
$
\tGn 
= 
\sum_{i=1}^{\tilde{N}_n} \frac{1}{1+\exp(-\tilde{c}^n_i)} \delta_{(\tilde{A}^n_i, \tilde{b}^n_i, \tilde{\eta}^n_i)}.
$
This implies that exactly-specified experts converge at a rate of $\mathcal{O}_P\bigl(({\log(n)/n})^{1/2}\bigr)$, while over-specified experts converge at $\mathcal{O}_P\bigl((\log(n)/n)^{1/4}\bigr)$. Consequently, exactly-specified experts require $\mathcal{O}(\epsilon^{-2})$ samples to achieve an approximation error $\epsilon$, whereas over-specified experts need $\mathcal{O}(\epsilon^{-4})$ samples for the same error.

\textbf{Sample complexity comparison under the sparse regime:} 
% According to the results in \cite{akbarian2024quadratic}, the ``experts'' in the softmax self-attention share the same sample efficiency as those in the sigmoid version. 
% Recall our discussion in Section \ref{sec:sparse_regime},
% polynomial experts are not strongly identifiable experts, therefore, it might require the polynomial experts an exponential number of data $\mathcal{O}(\exp(\epsilon^{-1/\tau}))$ to obtain the approximation error $\epsilon$, for some constant $\tau>0$.
% But polynomial experts are partial-strongly identifiable experts, so they need only a polynomial number $\mathcal{O}(\epsilon^{-4})$ data points to attain the approximation error $\epsilon$ for partially quadratic score function. 
% In particular, according to the results in \cite{akbarian2024quadratic}, linear experts with quadratic monomial gate
% require the polynomial experts an exponential number of data $\mathcal{O}(\exp(\epsilon^{-1/\tau}))$ to obtain the approximation error $\epsilon$, for some constant $\tau>0$.
% But the linear expert $\cale(x,(\alpha,\beta))=\alpha^{\top}x + \beta$ still satisfies partial-strong identifiability when $\beta\neq0$, so it need only a polynomial number $\mathcal{O}(\epsilon^{-4})$ data points to attain the approximation error $\epsilon$.
% Thus, we claim that the sigmoid self-attention is more sample-efficient than the softmax self-attention under the sparse regime of gating parameters for partially quadratic score function for linear experts. 
Recall from Section \ref{sec:sparse_regime} that polynomial experts, although not strongly identifiable, are partial-strongly identifiable. Hence, while they might require $\mathcal{O}\bigl(\exp(\epsilon^{-1/\tau})\bigr)$ samples to achieve an approximation error $\epsilon$ under a fully quadratic score function (for some constant $\tau > 0$), they only need $\mathcal{O}(\epsilon^{-4})$ samples under a partially quadratic score function.

According to \cite{akbarian2024quadratic}, linear experts with quadratic monomial gates require an exponential number of samples.
But the linear expert $\cale(x,(\alpha,\beta))=\alpha^{\top}x + \beta$  satisfies partial-strong identifiability when $\beta\neq0$, so it need only a polynomial number $\mathcal{O}(\epsilon^{-4})$ data points to attain the approximation error $\epsilon$.

Consequently, for partially quadratic score functions with linear experts, the sigmoid self-attention is more sample-efficient than the softmax self-attention under the sparse gating regime.

\subsubsection{Dense Regime of Gating Parameters}
% Now we consider the Dense Regime.
% \paragraph{Weak identifiability}
In this section, we focus on the dense regime.
Under dense regime, for the overspecified parameters, there exists $i\in[\ns]$, s.t. $ \Asi \neq 0_{d\times d} $.
In this situation, the least squares regression estimator $f_{\hGn}\to f_{\llG}$, a regression function, where the parameters 
\begin{align*}
\llG\in\llcalm_N(\Theta):=\argmin_{G\in\calm_N(\Theta)\setminus\calm_{\ns}(\Theta)}
\|
\tf_G-\tf_{\Gs}
\|_{L^2(\mu)}.
\end{align*}
In other words, the estimators of the parameters defining $ \tf_{\tGn} $ converge to the parameters of $ \tf_{\llG} $.
WLOG, we assume that 
\begin{align*}
    \llG:=\sum_{i=1}^N
    \frac{1}{1+\exp(-\llc_i)}
    \delta_{(\lla_i,\lleta_i)}.
\end{align*}
Now we introduce a definition for \textit{partial-weak identifiability} to specify the experts that has a similar fast convergence rate as the partial-strongly identifiable experts in Definition \ref{def:strong_identifiability-mono} under sparse regime. 
Partial-weakly identifiable experts need to satisfy only a subset of the conditions required for partial-strongly identifiable experts.
\begin{definition}[Partial-weak identifiability]
\label{def:weak_identifiability-mono}
    An expert function $\cale(x,\eta)$ is \textit{partial-weakly identifiable} if it is differentiable w.r.t its parameter $\eta$ for $\mu$-almost all $x$ and, for any positive integer $\ell$
    and any pair-wise distinct choices of parameters 
    $\left\{ (A_i ,c_i,\eta_i) \right\}_{i=1}^{\ell}$, 
    the functions in the class
    \begin{align*}
        \left\{  
        \frac{\partial
        F}
        {\partial A^{\tau_1} \partial c^{\tau_2}\partial\eta^{\tau_3}}
    (x,A_i ,c_i,\eta_i):
    i\in[\ell],
    \sum_{i=1}^3|\tau_i|
    =1
        \right\}
    \end{align*}
    are linearly independent, for $\mu$-almost all $x$,
    where 
    $(\tau_1,\tau_2,\tau_3 )\in\bbN^{d\times d} \times\bbN\times\bbN^q$.
\end{definition}
% \paragraph{Example.} 
% Consider an expert network $ h(x, (\alpha, \beta)) = \varphi(\alpha^\top x + \beta) $. It can be verified that if $ \alpha \neq 0_d $ and the activation function $ \varphi $ is a ReLU, GELU, or a polynomial, then the expert $ h(x, (\alpha, \beta)) $ is weakly identifiable. Conversely, if $ \alpha = 0_d $, meaning the expert does not depend on the input, the weak identifiability condition is not satisfied, regardless of the choice of the activation function.
% \textbf{Examples.} It is worth noting that partial-strongly identifiable experts also meet the partial-weak identifiability condition. For instance, it can be verified that the previously mentioned two-layer neural networks of the form $\mathcal{E}(x,(\alpha,\beta,\lambda))=\lambda\phi(\alpha^{\top}x+\beta)$, where $\phi$ is $\relu$ or $\gelu$ function, $\alpha\neq0_d$, and $\lambda\neq 0$ are weakly identifiable. Moreover, polynomial experts $\mathcal{E}(x,(\alpha,\beta))=(\alpha^{\top}x+\beta)^p$, for $p\in\mathbb{N}$, also satisfy the weak identifiability condition. 
\textbf{Examples.} It is worth noting that partial-strongly identifiable experts also meet the partial-weak identifiability condition. 
In particular, the previously mentioned two-layer neural networks
$
\mathcal{E}(x; (\alpha,\beta,\lambda)) 
= 
\lambda \phi(\alpha^\top x + \beta),
$
where $\phi$ is either $\mathrm{ReLU}$ or $\mathrm{GELU}$, $\alpha \neq 0_d$, and $\lambda \neq 0$, are weakly identifiable. Likewise, polynomial experts of the form
$
\mathcal{E}(x; (\alpha,\beta)) 
= 
(\alpha^\top x + \beta)^p, p \in \mathbb{N},
$
also satisfy the weak identifiability condition.

% Next, we will exhibit the convergence behavior under Dense Regime of weakly identifiable experts in Theorem \ref{thm:dthree_loss-dfive-mono}, equipped with the following Voronoi loss function: 
We now establish, in Theorem \ref{thm:dthree_loss-dfive-mono}, the convergence behavior of partial-weakly identifiable experts in the dense regime. Specifically, let $\mathcal{L}_5$ denote the Voronoi loss function defined by
\begin{align}
    \lfive:=
    \sum_{j=1}^N
    \sum_{i\in\calAj}
    \big[
    \|A_i-\llai\|
    % +\|b_i-\bbi\|
    +|c_i-\llci|
    +\|\eta_i-\lletai\|
    \big].
\end{align}
% Under this loss, Theorem \ref{thm:dthree_loss-dfive-mono} characterizes the convergence rates for weakly identifiable experts.
% \begin{align}
%     \lfive:=
%     \sum_{j=1}^N
%     \sum_{i\in\calAj}
%     \big[
%     \|A_i-\llai\|
%     % +\|b_i-\bbi\|
%     +|c_i-\llci|
%     +\|\eta_i-\lletai\|
%     \big].
% \end{align}
% Then we will have
\begin{theorem}
\label{thm:dthree_loss-dfive-mono}
    If the function $x\mapsto\cale(x,\eta)$ is partial-weakly identifiable, then the lower bound 
    $
        \inf_{\llG\in\llcalm_N(\Theta)}
        \Vert \tf_{G}-\tf_{\llG}\Vert_{L^2(\mu)}
        \gtrsim
        \lfive    
    $
    holds true for any mixing measure $G\in\calm_N(\Theta)$.
As a consequence, we obtain that 
\begin{align*}
    \inf_{\llG\in\llcalm_N(\Theta)}
    \call_5(\tGn,\llG)=
        \mathcal{O}_P\left(
        \sqrt{
        {\log(n)}/{n} 
        }
        \right).
\end{align*}
\end{theorem}
Proof of Theorem \ref{thm:dthree_loss-dfive-mono} is in Appendix \ref{proof:dthree_loss-dfive-mono}.
% Since the convergence rates of parameter estimators $\widetilde{\eta}^n_i$ are of order $\mathcal{O}_P([\log(n)/n]^{\frac{1}{2}})$, the weakly identfiable expert estimators $\mathcal{E}(x,\widetilde{\eta}^n_i)$ also admits the same convergence rates. 
% Therefore, it takes those experts only a polynomial number $\mathcal{O}(\epsilon^{-2})$ to achieve the approximation error $\epsilon$.
Since $\widetilde{\eta}^n_i$ converges at a rate of $\mathcal{O}_P\bigl(\sqrt{{\log(n)}/{n}}\bigr)$, the partial-weakly identifiable expert estimators $\mathcal{E}\bigl(x,\widetilde{\eta}^n_i\bigr)$ enjoy the same convergence order. Consequently, these experts require a polynomial number of samples, $\mathcal{O}(\epsilon^{-2})$, to achieve an approximation error of $\epsilon$.

\textbf{Sample complexity comparison under the dense regime:} 
% Recall that it costs partial-strongly identifiable ``experts'' and polynomial ``experts'' in the softmax self-attention $\mathcal{O}(\epsilon^{-4})$ and $\mathcal{O}(\exp(\epsilon^{-1/\tau}))$ data points to reach the approximation error $\epsilon$. On the other hand, as those experts satisfy the partial-weak identifiability condition, they need only the sample size of $\mathcal{O}(\epsilon^{-2})$ to achieve the same error under the dense regime of sigmoid self-attention. For that reason, we claim that the sigmoid self-attention is more sample efficient than its softmax counterpart under the dense regime, which is more likely to occur in practice than the sparse regime.
Recall that, under softmax self-attention, partial-strongly identifiable and polynomial “experts” require $\mathcal{O}(\epsilon^{-4})$ and $\mathcal{O}\bigl(\exp(\epsilon^{-1/\tau})\bigr)$ samples, respectively, to achieve an approximation error $\epsilon$. However, because these experts all satisfy partial-weak identifiability under sigmoid self-attention, they need only $\mathcal{O}(\epsilon^{-2})$ samples in the dense regime for partially quadratic scores. We therefore conclude that, in practice, where the dense regime is more common, sigmoid self-attention is more sample-efficient than its softmax counterpart.

\subsection{Proofs of the results for Partially Quadratic Scores in Appendix \ref{appendix:analysis_quadratic_mono}}
\label{appendix:proof_quadratic_mono}

\subsubsection{Proof of Proposition \ref{thm:function-convergence-specified-mono}}
\label{proof:function-convergence-specified-mono}

\begin{proof}
    The proof of Proposition \ref{thm:function-convergence-specified-mono} can be done in a similar fashion to that of Proposition \ref{thm:function-convergence-specified} in Appendix 
% \ref{proof:function-convergence-specified}.
\ref{appendix:proof_function_convergence}.
\end{proof}

\subsubsection{Proof of Theorem \ref{thm:done_loss-dfour-mono}}
\label{proof:done_loss-dfour-mono}

    % If $h(x,\eta)$ is strong identifiable, then
    % \begin{align*}
    %     \Vert f_{G}-f_{\Gs}\Vert_{L^2(\mu)}
    %     \gtrsim
    %     \lone
    % \end{align*}
    % for any $G\in\calm_N(\Theta)$.

\begin{proof}
In order to prove $\Vert \tf_{G}-\tf_{\Gs}\Vert_{L^2(\mu)}
        \gtrsim
        \lfour$
for any $G\in\calm_N(\Theta)$, 
it is sufficient to show that 
\begin{align}
\label{pfeq:done_aim-dfour-mono}
    \inf_{G\in\calm_N(\Theta)}
    \frac{\Vert \tf_{G}-\tf_{\Gs}\Vert_{L^2(\mu)}}{\lfour}
    >0.
\end{align}
To prove the above inequality, we consider two cases for the denominator $\lfour$: either it lies within a ball $B(0, \varepsilon)$ where the loss is sufficiently small, or it falls outside this region, where $\lfour$ will not vanish.

\textbf{Local part:}
At first, we focus on that 
\begin{align}
\label{pfeq:done_local-dfour-mono}
    \lim_{\varepsilon\to0}
    \inf_{G\in\calm_N(\Theta):\lfour\leq\varepsilon}
    \frac{\Vert f_{G}-f_{\Gs}\Vert_{L^2(\mu)}}{\lfour}
    >0.
\end{align}
Assume, for contradiction, that the above claim does not hold. Then there exists a sequence of mixing measures
$\tGn=\sum_{i=1}^{\ns}\delta_{(A_i^n ,c_i^n,\eta_i^n)}$ in
$\calm_N(\Theta)$ such that as $n\to\infty$, we get
\begin{align}
    \begin{cases}
        \call_{4n}:=\call_4(\tGn,\Gs)\to 0,\\
        \|\tf_{\tGn}-\tf_{\Gs} \|_{L^2(\mu)}/\call_{4n}\to0.
    \end{cases}
\end{align}
Let us recall that
    \begin{align*}
    \call_{4n}
    &
    :=\sum_{j=1}^{\bn}
    \left|
    \sum_{i\in\calAj} \frac{1}{1+\exp(-c^n_i)}-\frac{1}{1+\exp(-c_j^*)}
    \right|
    % \nonumber\\
    +\sum_{j=1}^{\bn}
    \sum_{i\in\calAj}
    \left[
    \|\Delta A^n_{ij} \|^2
    % +\|\Delta b^n_{ij} \|^2
    +\|\Delta \eta^n_{ij} \|^2
    \right]
    % \nonumber
    \\
    &+\sum_{j=\bn+1}^{\ns}
    \sum_{i\in\calAj}
    \left[
    \|\Delta A^n_{ij} \|
    % +\|\Delta b^n_{ij} \|
    +|\Delta c^n_{ij} |
    +\|\Delta \eta^n_{ij} \|
    \right],
\end{align*}
where $\Delta A^n_{ij}:=A^n_i-A^*_j$, 
% $\Delta b^n_{ij}:=b^n_i-b^*_j$, 
$\Delta c^n_{ij}:=c^n_i-c^*_j$, 
$\Delta \eta^n_{ij}:=\eta^n_i-\eta^*_j$.
Since $\call_{4n}\to0$ as $n\to0$, there are two different situation for parameter convergence:
\begin{itemize}
    \item For $j=1,\cdots,\bn$, parameters are over-fitted: $\sum_{i\in\calAj} \frac{1}{1+\exp(-c^n_i)}\rightarrow\frac{1}{1+\exp(-c_j^*)}$
    and $(A^n_i ,\eta^n_i)\rightarrow (A^*_j,\eta^*_j)$, $\forall i\in\calAj$;
    \item For $j=\bn+1,\cdots,\ns$, parameters are exact-fitted: 
    $(A^n_i ,c^n_i,\eta^n_i)\rightarrow (A^*_j ,c^*_j,\eta^*_j), \forall i\in\calAj$.
\end{itemize}
 
\textbf{Step 1 - Taylor expansion:}
In this step, we decompose the term $\tf_{\tGn}(x) - \tf_{\Gs}(x)$ using a Taylor expansion. First, let us denote
\begin{align}
\label{eq:done_taylor_expension-dfour-mono}
    &{\tf}_{\tGn}(x)-\tf_{\Gs}(x)
    % \nonumber
    % \\&
    =
    \sum_{j=1}^{\bn}
    \left[
    \sum_{i\in\calAj}
    \frac{1}{1+
    \exp(-x^{\top}A^n_ix
    % -(b^n_i)^{\top}x
    -c_i^n)
    }\cdot
    \cale(x,\eta^n_i)
    -
    \frac{1}{1+
    \exp(-x^{\top}A^*_jx
    % -(b^*_j)^{\top}x
    -c_j^*)
    }\cdot
    \cale(x,\ej)
    \right]
    \nonumber
    \\&
    +
    \sum_{j=\bn+1}^{\ns}
    \left[
    \sum_{i\in\calAj}
    \frac{1}{1+
    \exp(-x^{\top}A^n_ix
    % -(b^n_i)^{\top}x
    -c_i^n)
    }\cdot
    \cale(x,\eta^n_i)
    -
    \frac{1}{1+
    \exp(-x^{\top}A^*_jx
    % -(b^*_j)^{\top}x
    -c_j^*)
    }\cdot
    \cale(x,\ej)
    \right]
    % \nonumber
    \\&=
    \sum_{j=1}^{\bn}
    \sum_{i\in\calAj}
    \left[
    \frac{1}{1+
    \exp(-x^{\top}A^n_ix
    % -(b^n_i)^{\top}x
    -c_i^n)
    }\cdot
    \cale(x,\eta^n_i)
    -
    \frac{1}{1+
    \exp(-c_i^n)
    }\cdot
    \cale(x,\ej)
    \right]:=\Ione_n
    \nonumber
    \\&
    +
    \sum_{j=1}^{\bn}
    \left[
    \sum_{i\in\calAj}
    \frac{1}{1+
    \exp(-c_i^n)
    }
    % \cdot
    % h(x,\eta^*_i)
    -
    \frac{1}{1+
    \exp(-c_j^*)
    }
    \right]
    \cdot
    \cale(x,\ej):=\Itwo_n
    \nonumber
    \\&
    +
    \sum_{j=\bn+1}^{\ns}
    \left[
    \sum_{i\in\calAj}
    \frac{1}{1+
    \exp(-x^{\top}A^n_ix
    % -(b^n_i)^{\top}x
    -c_i^n)
    }\cdot
    \cale(x,\eta^n_i)
    -
    \frac{1}{1+
    \exp(-x^{\top}A^*_jx
    % -(b^*_j)^{\top}x
    -c_j^*)
    }\cdot
    \cale(x,\ej)
    \right]:=\Ithree_n
    \nonumber
\end{align}
Let us denote $\sigma(x,A, c) = \frac{1}{1 + \exp(-x^{\top}Ax   - c)}$,
then $\Ione_n, \Itwo_n$ and $\Ithree_n$ could be denoted as
\begin{align}
    \Ione_n&=\sum_{j=1}^{\bn}
    \sum_{i\in\calAj}
    \left[
    \frac{1}{1+
    \exp(-x^{\top}A^n_ix
    % -(b^n_i)^{\top}x
    -c_i^n)
    }\cdot
    \cale(x,\eta^n_i)
    -
    \frac{1}{1+
    \exp(-c_i^n)
    }\cdot
    \cale(x,\ej)
    \right]
    \nonumber
    \\&
    =
    \sum_{j=1}^{\bn}
    \sum_{i\in\calAj}
    \left[
    \sigma(x,A_i^n
    % ,b_i^n
    ,c_i^n) \cale(x,\eta^n_i)
    -
    \sigma(x,0_{d\times d}
    % ,0_d
    ,c_i^n) \cale(x,\eta^*_j)
    \right]
    \nonumber
    \\&
    =
    \sum_{j=1}^{\bn}
    \sum_{i\in\calAj}
    \sum_{|\gamma|=1}^2
    \frac{1}{\gamma!}
    (\Delta A_{ij}^n)^{\gamma_1}
    % (\Delta b_{ij}^n)^{\gamma_2}
    (\Delta \eta_{ij}^n)^{\gamma_2}
    \frac{\partial^{|\gamma_1| }\sigma}{\partial A^{\gamma_1} }
    (x,0_{d\times d} ,c_i^n) 
    \frac{\partial^{|\gamma_2|}\cale}{\partial\eta^{\gamma_2}}
    (x,\eta^*_j)
    +R_1(x),
    % \\
\end{align}
\begin{align}
    \Itwo_n=
    \sum_{j=1}^{\bn}
    \left[
    \sum_{i\in\calAj}
    \frac{1}{1+
    \exp(-c_i^n)
    }
    -
    \frac{1}{1+
    \exp(-c_j^*)
    }
    \right]
    \cdot
    \cale(x,\ej),
    % \\
\end{align}
\begin{align}    \Ithree_n&=\sum_{j=\bn+1}^{\ns}
    \left[
    \sum_{i\in\calAj}
    \frac{1}{1+
    \exp(-x^{\top}A^n_ix -c_i^n)
    }\cdot
    \cale(x,\eta^n_i)
    -
    \frac{1}{1+
    \exp(-x^{\top}A^*_jx -c_j^*)
    }\cdot
    \cale(x,\ej)
    \right]
    \nonumber
    \\&
    =\sum_{j=\bn+1}^{\ns}
    \left[
    \sum_{i\in\calAj}
    \sigma(x,A^n_i ,c^n_i)
    \cale(x,\eta^n_i)
    -
    \sigma(x,A^*_j ,c^*_j)
    \cale(x,\ej)
    \right]
    \nonumber
    \\&
    =\sum_{j=\bn+1}^{\ns}
    \sum_{|\tau|=1}
    \left[
    \sum_{i\in\calAj}
    \frac{1}{\tau !}
    (\Delta A_{ij}^n)^{\tau_1}
    % (\Delta b_{ij}^n)^{\tau_2}
    (\Delta c_{ij}^n)^{\tau_2}
    (\Delta \eta_{ij}^n)^{\tau_3}
    \right]
    \frac{\partial^{|\tau_1|+|\tau_2| }\sigma}{\partial A^{\tau_1}   \partial c^{\tau_2}}(x,A^*_j ,c^*_j)
    \frac{\partial^{|\tau_3|} \cale}{\partial \eta^{\tau_3}}(x,\ej)
    +R_2(x),
\end{align}
where $R_i(x), i=1,2$ are Taylor remainder such that $R_i(x)/\call_{4n}\to0$ as $n\to\infty$ for $i=[2]$.

Now we could denote 
\begin{align}
    K^n_{j,i,\gamma_{1:2}}&= \frac{1}{\gamma!}
    (\Delta A_{ij}^n)^{\gamma_1}
    % (\Delta b_{ij}^n)^{\gamma_2}
    (\Delta \eta_{ij}^n)^{\gamma_2},~
    j\in[\bn],i\in\calAj,\gamma\in\Real^{d\times d} \times\Real^q
    \label{pfeq:done_loss_notation1-dfour-mono}
    \\
    L^n_j&=\sum_{i\in\calAj}
    \frac{1}{1+
    \exp(-c_i^n)
    }
    -
    \frac{1}{1+
    \exp(-c_j^*)
    },~
    j\in[\bn]
    \label{pfeq:done_loss_notation2-dfour-mono}
    \\
    T^n_{j,\tau_{1:3}}&=
    \sum_{i\in\calAj}
    \frac{1}{\tau !}
    (\Delta A_{ij}^n)^{\tau_1}
    % (\Delta b_{ij}^n)^{\tau_2}
    (\Delta c_{ij}^n)^{\tau_2}
    (\Delta \eta_{ij}^n)^{\tau_3},~
    j\in\{\bn+1,\cdots,\ns\}, 
    \tau\in\Real^{d\times d} \times\Real\times\Real^q
    \label{pfeq:done_loss_notation3-dfour-mono}
\end{align}
where $1\leq\sum_{i=1}^2|\gamma_i|\leq2$ and 
$\sum_{i=1}^3|\tau_i|=1$. 
Also recall that for the last term, $i$ was settled directly by $j$ since all the parameters are exact-fitted when $\bn+1\leq j\leq\ns$, i.e. $|\calAj|=1$.

Using these notations \eqref{pfeq:done_loss_notation1-dfour-mono} - \eqref{pfeq:done_loss_notation3-dfour-mono} we can now rewrite the difference
${\tf}_{\tGn}(x)-\tf_{\Gs}(x)$ as
\begin{align}
\label{pfeq:done_decomposition-dfour-mono}
    {\tf}_{\tGn}(x)-\tf_{\Gs}(x)
    % \\&
    &=
    \sum_{j=1}^{\bn}
    \sum_{i\in\calAj}
    \sum_{|\gamma|=1}^2
    K^n_{j,i,\gamma_{1:2}}
    \frac{\partial^{|\gamma_1| }\sigma}{\partial A^{\gamma_1} }
    (x,0_{d\times d} ,c_i^n) 
    \frac{\partial^{|\gamma_2|}\cale}{\partial\eta^{\gamma_2}}
    (x,\eta^*_j)
    +\sum_{j=1}^{\bn}
    L^n_j
    \cdot
    \cale(x,\ej)\nonumber
    \\&
    +\sum_{j=\bn+1}^{\ns}
    \sum_{|\tau|=1}
    T^n_{j,\tau_{1:3}}
    \frac{\partial^{|\tau_1|+|\tau_2|}\sigma}{\partial A^{\tau_1} \partial c^{\tau_2}}(x,A^*_j,c^*_j)
    \frac{\partial^{|\tau_3|} \cale}{\partial \eta^{\tau_3}}(x,\ej)
    +R_1(x)
    +R_2(x).    
\end{align}

\textbf{Step 2 - Non-vanishing coefficients:}
Now we claim that at least one in the set 
$$
\mathcal{S}=
\left\{ 
\frac{K^n_{j,i,\gamma_{1:2}}}{\call_{4n}},
\frac{L^n_{j}}{\call_{4n}},
\frac{T^n_{j,\tau_{1:3}}}{\call_{4n}}
\right\}$$
will not vanish as n goes to infinity.
We prove by contradiction that all of them converge to zero when $n\to 0$:
\begin{align*}
    \frac{K^n_{j,i,\gamma_{1:2}}}{\call_{4n}}\to0,~
\frac{L^n_{j}}{\call_{4n}}\to0,~
\frac{T^n_{j,\tau_{1:3}}}{\call_{4n}}\to0.
\end{align*}
Then follows directly from $L^n_j/\call_{4n}\to 0$ we could conclude that
\begin{align}
\label{pfeq:done_loss_coefficient1-dfour-mono}
    \frac{1}{\call_{4n}}
    \sum_{j=1}^{\bn}
    \left|
    \sum_{i\in\calAj} \frac{1}{1+\exp(-c^n_i)}-\frac{1}{1+\exp(-c_j^*)}
    \right|
    =
    \frac{1}{\call_{4n}}
    \sum_{j=1}^{\bn}
    \left|
    L^n_j
    \right|
    \to0.
\end{align}
Before consider other coefficients, for simplicity, we denote
 $e_{d,u}=\overbrace{(0,\ldots,0,\underbrace{1}_{u\text{-th}},0,\ldots,0)}^{d\text{-tuple}}\in\Real^d$
as a $d$-tuple with all components equal to 0, except the $u$-th, which is $1$;
and $e_{d\times d,uv}$ as a $d\times d$ matrix with all components equal to 0, except the element in the $u$-th row and $v$-th column , which is $1$, i.e. $e_{d\times d,uv}=\overbrace{(0_d^{\top},\ldots,0_d^{\top},\underbrace{e^{\top}_{d,u}}_{v\text{-th}},0_d^{\top},\ldots,0_d^{\top})}^{d\text{-column}}\in\Real^{d\times d}$.

Now consider for arbitrary
$ u,v \in[d]$,
let $\gamma_1=2e_{d\times d,uv} $ and $\gamma_2=0_{q}$,
we will have 
\begin{align*}
\frac{1}{\call_{4n}}
\left|
\Delta (A^n_{ij})^{(uv)}
\right|^2
=
\frac{1}{2 \call_{4n}}
\left| 
K^n_{j,i,2e_{d\times d,uv}, 0_q}
\right|
\to 0,~ n\to \infty.
\end{align*}
Then by taking the summation of the term with $u,v\in[d]$, we will have
\begin{align}
\label{pfeq:done_loss_coefficient2-dfour-mono}
    \frac{1}{\call_{4n}}
    \sum_{j=1}^{\bn}
    \sum_{i\in\calAj}
    \|\Delta A^n_{ij} \|^2
    =
    \frac{1}{2\call_{4n}}
    % \frac{d^2}{2}
    \sum_{j=1}^{\bn}
    \sum_{i\in\calAj}
    \sum_{u=1}^d
    \sum_{v=1}^d
    | K^n_{j,i,2e_{d\times d,uv}, 0_q}|\to0.
\end{align}
% For arbitrary
% $ u \in[d]$,
% let $\gamma_1=0_{d\times d}, \gamma_2=2e_{d,u}$ and $\gamma_2=0_{q}$,
% we will have 
% \begin{align*}
% \frac{1}{\call_{4n}}
% \left|
% \Delta (b^n_{ij})^{(u)}
% \right|^2
% =
% \frac{1}{2\call_{4n}}
% % \frac{1}{2}
% \left| 
% K^n_{j,i,0_{d\times d},2e_{d,u},0_q}
% \right|
% \to 0,~ n\to \infty,
% \end{align*}
% Then by taking the summation of the previous term with $u\in[d]$, we will have
% \begin{align}
% \label{pfeq:done_loss_coefficient3-dfour-mono}
%     \frac{1}{\call_{4n}}
%     \sum_{j=1}^{\bn}
%     \sum_{i\in\calAj}
%     \|\Delta b^n_{ij} \|^2
%     =
%     \frac{1}{2\call_{4n}}
%     % \frac{d}{2}
%     \sum_{j=1}^{\bn}
%     \sum_{i\in\calAj}
%     \sum_{u=1}^d
%     | K^n_{j,i,0_{d\times d},2e_{d,u},0_q}|\to0.
% \end{align}
Similarly we will have 
\begin{align}
\label{pfeq:done_loss_coefficient4-dfour-mono}
    \frac{1}{\call_{4n}}
    \sum_{j=1}^{\bn}
    \sum_{i\in\calAj}
    \|\Delta \eta^n_{ij} \|^2
    =
    \frac{1}{2\call_{4n}}
    \sum_{j=1}^{\bn}
    \sum_{i\in\calAj}
    \sum_{w=1}^q
    | K^n_{j,i,0_{d\times d}, 2e_{q,w}}|\to0.
\end{align}
Now, consider 
$\bn+1\leq j\leq\ns$,
such that its corresponding Voronoi cell has only one element, i.e.
$|\calAj|=1$. 
For arbitrary
$ u,v \in[d]$,
let $\tau_1=e_{d\times d,uv},  \tau_2=0$ and $\tau_3=0_{q}$,
we will have 
\begin{align*}
\frac{1}{\call_{4n}}
\sum_{i\in\calAj}
\left|
\Delta (A^n_{ij})^{(uv)}
\right|
=
\frac{1}{\call_{4n}}
\left| 
T^n_{j,e_{d\times d,uv},  0,0_q}
\right|
\to 0,~ n\to \infty.
\end{align*}
Then by taking the summation of the term with $u,v\in[d]$, we will have
\begin{align*}
    \frac{1}{\call_{4n}}
    \sum_{j=\bn+1}^{\ns}
    \sum_{i\in\calAj}
    \|\Delta A^n_{ij} \|_1
    =
    \frac{1}{\call_{4n}}
    \sum_{j=\bn+1}^{\ns}
    \sum_{u=1}^d
    \sum_{v=1}^d
    | T^n_{j,e_{d\times d,uv} ,0,0_q}|\to0.
\end{align*}
Recall the topological equivalence between $L_1$-norm and $L_2$-norm on finite-dimensional vector space over $\Real$, we will have
\begin{align}
\label{pfeq:done_loss_coefficient5-dfour-mono}
    \frac{1}{\call_{4n}}
    \sum_{j=\bn+1}^{\ns}
    \sum_{i\in\calAj}
    \|\Delta A^n_{ij} \|
    \to0.
\end{align}
Following a similar argument, 
since
\begin{align*}
    % \frac{1}{\call_{4n}}
    % \sum_{j=\bn+1}^{\ns}
    % \sum_{i\in\calAj}
    % \|\Delta b^n_{ij} \|_1
    % &=
    % \frac{1}{\call_{4n}}
    % \sum_{j=\bn+1}^{\ns}
    % \sum_{u=1}^d
    % | T^n_{j,0_{d\times d},e_{d,u},0,0_q}|\to0,
    % \\
    \frac{1}{\call_{4n}}
    \sum_{j=\bn+1}^{\ns}
    \sum_{i\in\calAj}
    |\Delta c^n_{ij} |
    &=
    \frac{1}{\call_{4n}}
    \sum_{j=\bn+1}^{\ns}
    % \sum_{u=1}^d
    | T^n_{j,0_{d\times d} ,1,0_q}|\to0,
    \\
    \frac{1}{\call_{4n}}
    \sum_{j=\bn+1}^{\ns}
    \sum_{i\in\calAj}
    \|\Delta \eta^n_{ij} \|_1
    &=
    \frac{1}{\call_{4n}}
    \sum_{j=\bn+1}^{\ns}
    \sum_{w=1}^q
    | T^n_{j,0_{d\times d} ,0,e_{q,w}}|\to0,
\end{align*}
we obtain that 
\begin{align}
\label{pfeq:done_loss_coefficient6-dfour-mono}
    % \frac{1}{\call_{4n}}
    % \sum_{j=\bn+1}^{\ns}
    % \sum_{i\in\calAj}
    % \|\Delta b^n_{ij} \|\to0,~
    % \\&
    \frac{1}{\call_{4n}}
    \sum_{j=\bn+1}^{\ns}
    \sum_{i\in\calAj}
    |\Delta c^n_{ij} |\to0,~
    % \\&
    \frac{1}{\call_{4n}}
    \sum_{j=\bn+1}^{\ns}
    \sum_{i\in\calAj}
    \|\Delta \eta^n_{ij} \|\to0.
\end{align}
Now taking
the summation of limits in equations
\eqref{pfeq:done_loss_coefficient1-dfour-mono} - 
\eqref{pfeq:done_loss_coefficient6-dfour-mono},
we could deduce that
    \begin{align*}
    1=\frac{\call_{4n}}{\call_{4n}}
    &
    =
    \frac{1}{\call_{4n}}\sum_{j=1}^{\bn}
    \left|
    \sum_{i\in\calAj} \frac{1}{1+\exp(-c^n_i)}-\frac{1}{1+\exp(-c_j^*)}
    \right|
    % \nonumber\\
    +
    \frac{1}{\call_{4n}}%\sum_{j=1}^{\bn}
    \sum_{j=1}^{\bn}
    \sum_{i\in\calAj}
    \left[
    \|\Delta A^n_{ij} \|^2
    % +\|\Delta b^n_{ij} \|^2
    +\|\Delta \eta^n_{ij} \|^2
    \right]
    % \nonumber
    \\
    &+
    \frac{1}{\call_{4n}}%\sum_{j=1}^{\bn}
    \sum_{j=\bn+1}^{\ns}
    \sum_{i\in\calAj}
    \left[
    \|\Delta A^n_{ij} \|
    % +\|\Delta b^n_{ij} \|
    +|\Delta c^n_{ij} |
    +\|\Delta \eta^n_{ij} \|
    \right]\to0,
    % ~~n\to\infty,
\end{align*}
as $n\to\infty$,
which is a contradiction.
Thus, not all the coefficients of elements in the set
% at least one element in the set 
\begin{align*}
\mathcal{S}=
\left\{ 
\frac{K^n_{j,i,\gamma_{1:2}}}{\call_{4n}},
\frac{L^n_{j}}{\call_{4n}},
\frac{T^n_{j,\tau_{1:3}}}{\call_{4n}}
\right\}
\end{align*}
tend to 0 as $n\to\infty$. 
Let us denote by $m_n$ the maximum of the absolute values of those elements.
It follows from the previous result that $1/m_n\not\to \infty$ as $n\to\infty$.

\textbf{Step 3 - Application of Fatou’s lemma:} 
In this step, we apply Fatou's lemma to obtain the desired inequality in equation \eqref{pfeq:done_local-dfour-mono}.
Recall in the beginning we have assumed that
$
        % \call_{1n}:=D_1(G_n.\Gs)\to 0,\\
        \|{\tf}_{\tGn}-\tf_{\Gs} \|_{L^2(\mu)}/\call_{4n}\to0
$
and the topological equivalence between $L_1$-norm and $L_2$-norm on finite-dimensional vector space over $\Real$,
we obtain $
        \|{\tf}_{\tGn}-\tf_{\Gs} \|_{L^1(\mu)}/\call_{4n}\to0
$. 
By applying the Fatou’s lemma, we get
\begin{align*}
    0=\lim_{n\to\infty}\frac{\|{\tf}_{\tGn}-\tf_{\Gs} \|_{L^1(\mu)}}{m_n \call_{4n}}
    \geq
    \int\liminf_{n\to\infty}
    \frac{|{\tf}_{\tGn}(x)-\tf_{\Gs}(x)|}{m_n \call_{4n}}d\mu(x)
    \geq 0.
\end{align*}
This result suggests that for almost every $x$, 
\begin{align}
\label{pfeq:done_fatou-dfour-mono}
    \frac{{\tf}_{\tGn}(x)-\tf_{\Gs}(x)}{m_n \call_{4n}}\to 0,
\end{align}
recall equation \eqref{pfeq:done_decomposition-dfour-mono}, we deduce that
\begin{align*}
    \frac{{\tf}_{\tGn}(x)-\tf_{\Gs}(x)}{m_n \call_{4n}}
    &=
    \sum_{j=1}^{\bn}
    \sum_{i\in\calAj}
    \sum_{|\gamma|=1}^2
    \frac{K^n_{j,i,\gamma_{1:2}}}{m_n \call_{4n}}
    \frac{\partial^{|\gamma_1| }\sigma}{\partial A^{\gamma_1} }
    (x,0_{d\times d} ,c_i^n) 
    \frac{\partial^{|\gamma_2|}\cale}{\partial\eta^{\gamma_2}}
    (x,\eta^*_j)
    +\sum_{j=1}^{\bn}
    \frac{L^n_j}{m_n \call_{4n}}
    \cdot
    \cale(x,\ej)\nonumber
    \\&
    +\sum_{j=\bn+1}^{\ns}
    \sum_{|\tau|=1}
    \frac{T^n_{j,\tau_{1:3}}}{m_n \call_{4n}}
    \frac{\partial^{|\tau_1|+|\tau_2|}\sigma}{\partial A^{\tau_1} \partial c^{\tau_2}}(x,A^*_j,c^*_j)
    \frac{\partial^{|\tau_3|} \cale}{\partial \eta^{\tau_3}}(x,\ej)
    +\frac{R_1(x)}{m_n \call_{4n}}
    +\frac{R_2(x)}{m_n \call_{4n}}.    
\end{align*}
Let us denote 
\begin{align*}
    \frac{K^n_{j,i,\gamma_{1:2}}}{m_n \call_{4n}}\to k_{j,i,\gamma_{1:2}},~~
    \frac{L^n_j}{m_n \call_{4n}}\to l_j,~~
    \frac{T^n_{j,\tau_{1:3}}}{m_n \call_{4n}}\to t_{j,\tau_{1:3}},
\end{align*}
then from equation \eqref{pfeq:done_fatou-dfour-mono},
we will have
\begin{align*}
    &
    \sum_{j=1}^{\bn}
    \sum_{i\in\calAj}
    \sum_{|\gamma|=1}^2
    k_{j,i,\gamma_{1:2}}
    \frac{\partial^{|\gamma_1| }\sigma}{\partial A^{\gamma_1} }
    (x,0_{d\times d} ,c_i^n) 
    \frac{\partial^{|\gamma_2|}\cale}{\partial\eta^{\gamma_2}}
    (x,\eta^*_j)
    +\sum_{j=1}^{\bn}
    l^n_j
    \cdot
    \cale(x,\ej)\nonumber
    \\
    &
    +\sum_{j=\bn+1}^{\ns}
    \sum_{|\tau|=1}
    t^n_{j,\tau_{1:3}}
    \frac{\partial^{|\tau_1|+|\tau_2|}\sigma}{\partial A^{\tau_1} \partial c^{\tau_2}}(x,A^*_j,c^*_j)
    \frac{\partial^{|\tau_3|} \cale}{\partial \eta^{\tau_3}}(x,\ej)=0,    
\end{align*}
for almost every $x$.
Note that the expert function $\cale(\cdot,\eta)$ is strongly identifiable, then the above equation implies that
\begin{align*}
    k_{j,i,\gamma_{1:2}}= l_j= t_{j,\tau_{1:3}}=0,
\end{align*}
for any $j\in[\ns]$,  $(\gamma_1, \gamma_2)\in\bbN^{d\times d} \times\bbN^q$ and 
    $(\tau_1,\tau_2,\tau_3 )\in\bbN^{d\times d} \times\bbN\times\bbN^q$
such that 
$1\leq \sum_{i=1}^2|\gamma_i| \leq 2$ and
$\sum_{i=1}^3|\tau_i|=1 $.
This violates that at least one among the limits in the set $\{k_{j,i,\gamma_{1:2}}, l_j,  t_{j,\tau_{1:3}} \}$ is different from zero.

Thus, we obtain the local inequality in equation \eqref{pfeq:done_local-dfour-mono}. Consequently, there exists some $\varepsilon' > 0$ such that
\begin{align*}
    \inf_{G\in\calm_N(\Theta):\lfour\leq\varepsilon'}
    \frac{\Vert \tf_{G}-\tf_{\Gs}\Vert_{L^2(\mu)}}{\lfour}
    >0.
\end{align*}

\textbf{Global part:}
We now proceed to demonstrate equation \eqref{pfeq:done_aim-dfour-mono} for the case where the denominator does not vanish, i.e.
\begin{align}
\label{pfeq:done_global-dfour-mono}
    \inf_{G\in\calm_N(\Theta):\lfour>\varepsilon'}
    \frac{\Vert \tf_{G}-\tf_{\Gs}\Vert_{L^2(\mu)}}{\lfour}
    >0.
\end{align}
Suppose, for contradiction, that inequality \eqref{pfeq:done_global-dfour-mono} does not hold. Then there exists a sequence of mixing measures $G'_n \in \mathcal{M}_N(\Theta)$ such that $\call_4(G'_n, \Gs) > \varepsilon'$ and
\begin{align*}
    \lim_{n\to\infty}
    \frac{\Vert \tf_{G'_n}-\tf_{\Gs}\Vert_{L^2(\mu)}}{\call_4(G'_n,\Gs)}
    =0.
\end{align*}
Under this situation, we could deduce that $\Vert \tf_{G'_n}-\tf_{\Gs}\Vert_{L^2(\mu)}\to0$ as $n\to\infty$.
Since $\Theta$ is a compact set, we can replace the sequence $G'_n$ with a convergent subsequence, which approaches a mixing measure $G' \in \mathcal{M}_N(\Theta)$. Given that $\call_4(G'_n, \Gs) > \varepsilon'$, we conclude that $\call_4(G', \Gs) > \varepsilon'$. Then, using Fatou's lemma, we deduce:
\begin{align*}
    0=\lim_{n\to\infty}
    {\|\tf_{G'_n}-\tf_{\Gs} \|^2_{L^2(\mu)}}
    \geq
    \int\liminf_{n\to\infty}
    {\left|\tf_{G'_n}(x)-\tf_{\Gs}(x)\right|^2}d\mu(x),
\end{align*}
which indicates that $\tf_{G'}(x)=\tf_{\Gs}(x)$ for almost every $x$. 
Based on the Proposition \ref{prop:identifiability-mono} followed, we will have that 
$G'\equiv\Gs$.

\end{proof}

\begin{proposition}
\label{prop:identifiability-mono}
    If the equation $\tf_G (x) = \tf_{\Gs} (x)$ holds true for almost every x, then it follows that $G \equiv \Gs $.
\end{proposition}

\begin{proof}
    The proof of Proposition \ref{prop:identifiability-mono} can be done in a similar fashion to that of Proposition \ref{prop:identifiability}.
\end{proof}

\subsubsection{Proof of Theorem \ref{thm:dthree_loss-dfive-mono}}
\label{proof:dthree_loss-dfive-mono}

% $\check{G}$

\begin{proof}
Following from the result of 
Corollary \ref{crl:function-convergence-misspecified-mono}, it is sufficient to show that the following inequality holds true:
    \begin{align}
    \label{pfeq:dthree_loss_aim-dfive-mono}
        \inf_{\G\in\calm_N(\Theta)}
        \frac{\Vert \tf_{G}-\tf_{\llG}\Vert_{L^2(\mu)}}
        {\lfive}
        >0,
    \end{align}
for any mixing measure ${\llG\in\llcalm_N(\Theta)}$.
To prove the above inequality, we follow a similar approach to the proof in Appendix \ref{proof:dthree_loss}, dividing the analysis into a local part and a global part and omitting the global part. 
% . However, since the arguments for the global part remain the same (up to some notational changes) in the over-specified setting, they are omitted.

Therefore, for an arbitrary mixing measure $ \llG := \sum_{i=1}^N\frac{1}{1+\exp(-\llci)}\delta_{(\llai, \lletai)}\in\llcalm_N(\Theta) $, we focus exclusively on demonstrating that:
\begin{align}
\label{pfeq:dthree_local-dfive-mono}
    \lim_{\varepsilon\to0}
    \inf_{G\in\calm_N(\Theta):\lfive\leq\varepsilon}
    \frac{\Vert \tf_{G}-\tf_{\llG}\Vert_{L^2(\mu)}}{\lfive}
    >0.
\end{align}
Assume, for contradiction, that the above claim does not hold. Then there exists a sequence of mixing measures
$G_n=\sum_{i=1}^{N}\frac{1}{1+\exp(-c_i^n)}\delta_{(A_i^n,\eta_i^n)}$ in
$\calm_N(\Theta)$ such that as $n\to\infty$, we get
\begin{align}
    \begin{cases}
        \call_{5n}:=\call_5(\tG_n,\llG)\to 0,\\
        \|{\tf}_{\tGn}-\tf_{\llG} \|_{L^2(\mu)}/\call_{5n}\to0.
    \end{cases}
\end{align}
Let us denote by $\calA_i^n:=\calA_i(\tG_n)$ a Voronoi cell of $\tG_n$ generated by the $j$-th components of $\llG$. 
Since our analysis is asymptotic, we can assume that the Voronoi cells are independent of the sample size, i.e. $\calAj=\calA_i^n$.

In Dense Regime, since $\tG_n$ and $\llG$ have the same number of atoms $k$, and $\call_{5n} \to 0$, it follows that each Voronoi cell $\calA_i$ contains precisely one element for all $i \in [N]$. Without loss of generality, we assume $\calA_i = \{i\}$ for all $i \in [N]$. This ensures that $(A_{i}^n , c_{i}^n,\eta_i^n) \to (\llai,   \llci,\lletai)$ as $n \to \infty$ for every $i \in [N]$.

Consequently, the Voronoi loss $\call_{5n}$ can be expressed as:  
\begin{align}
\call_{5n} := \sum_{i=1}^N 
\left( 
\|\Delta \llai^n\| + 
% \|\Delta \bbi^n\| + 
|\Delta \llci^n | + 
\|\lletai^n\| 
\right),
\end{align}
where the increments are given by:  
$
\Delta \llai^n = A_{i}^n - \llai, 
\Delta \llci^n = c_{i}^n - \llci, 
\Delta \lletai^n = \eta_i^n - \lletai.
$

We now break the proof of the local part into the following three steps:

\textbf{Step 1 - Taylor expansion:}
In this step, we decompose the term $\tf_{\tG_n}(x) - \tf_{\llG}(x)$ using a Taylor expansion. First, let us denote
$\sigma(x,A, c):=\frac{1}{1+\exp(-x^{\top}Ax -c)}$,
then we have that
\begin{align}
\label{eq:dthree_taylor_expension-dfive-mono}
    &\tf_{\tG_n}(x)-\tf_{\llG}(x)
    =
    \sum_{i=1}^{N}
    \left[
    \frac{1}{1+
    \exp(-x^{\top}A^n_ix-c_i^n)
    }\cdot
    \cale(x,\eta^n_i)
    -
    \frac{1}{1+
    \exp(-x^{\top}\llai x -\llci)
    }\cdot
    \cale(x,\lletai)
    \right]
    \nonumber
    \\&
    =
    \sum_{i=1}^N
    \sum_{|\alpha|=1}
    \frac{1}{\alpha!}
    (\Delta\llai^n)^{\alpha_1}
    (\Delta\llci^n)^{\alpha_2}
    (\Delta\lletai^n)^{\alpha_3}
    \frac{\partial^{|\alpha_1| +|\alpha_2|}\sigma}{\partial A^{\alpha_1}   \partial c^{\alpha_2}}(x,\llai,\llci)
    \frac{\partial^{|\alpha_3|} \cale}{\partial \eta^{\alpha_3}}(x,\lletai)
    +R_1(x)    
    \nonumber
    \\&
    =
    \sum_{i=1}^N
    \sum_{|\alpha|=1}
    S^n_{i,\alpha_1,\alpha_2,\alpha_3}
% S^n_{i,\alpha_{1:4}}
    \frac{\partial^{|\alpha_1| +|\alpha_2|}\sigma}{\partial A^{\alpha_1}   \partial c^{\alpha_2}}(x,\llai,\llci)
    \frac{\partial^{|\alpha_3|} \cale}{\partial \eta^{\alpha_3}}(x,\lletai)
    +R_1(x),    
\end{align}
where $R_1(x)$ is a Taylor remainder such that $R_1(x)/\call_{5n}\to0$ as $n\to\infty$.
Now we could denote 
\begin{align}
    S^n_{i,\alpha_{1:3}}&=
    \frac{1}{\alpha!}
    (\Delta\llai^n)^{\alpha_1}
    (\Delta\llci^n)^{\alpha_2}
    (\Delta\lletai^n)^{\alpha_3},
    ~
    i\in[N], 
    \alpha\in\Real^{d\times d} \times\Real\times\Real^q
    \label{pfeq:dthree_loss_notation-dfive-mono}
\end{align}
where
$\sum_{i=1}^3|\alpha_i|=1$.

\textbf{Step 2 - Non-vanishing coefficients:}
Now we claim that at least one among the ratios
$S^n_{i,\alpha_{1:3}}/\call_{5n}$
will not vanish as n goes to infinity.
We prove by contradiction that all of them converge to zero when $n\to 0$:
\begin{align*}
\frac{S^n_{i,\alpha_{1:3}}}{\call_{5n}}\to0.
\end{align*}
for any $i\in[N], 
    \alpha\in\Real^{d\times d} \times\Real\times\Real^q$ 
such that
$\sum_{i=1}^3|\alpha_i|=1$. 

Now, consider 
$1\leq i\leq k$,
for arbitrary
$ u,v \in[d]$,
let $\alpha_1=e_{d\times d,uv}, \alpha_2=0$ and $\alpha_3=0_{q}$,
we will have 
\begin{align*}
\frac{1}{\call_{5n}}
\sum_{i=1}^N
\left|
\Delta (\llai^n)^{(uv)}
\right|
=
\frac{1}{\call_{5n}}
\left| 
S^n_{i,e_{d\times d,uv} ,0,0_q}
\right|
\to 0,~ n\to \infty.
\end{align*}
Then by taking the summation of the term with $u,v\in[d]$, we will have
\begin{align*}
    \frac{1}{\call_{5n}}
    \sum_{i=1}^N
    \|\Delta \llai^n \|_1
    =
    \frac{1}{\call_{5n}}
    \sum_{i=1}^{N}
    \sum_{u=1}^d
    \sum_{v=1}^d
    | S^n_{i,e_{d\times d,uv} ,0,0_q}|\to0.
\end{align*}
Recall the topological equivalence between $L_1$-norm and $L_2$-norm on finite-dimensional vector space over $\Real$, we will have
\begin{align}
\label{pfeq:dthree_loss_coefficient1-dfive-mono}
    \frac{1}{\call_{5n}}
    \sum_{i=1}^N
    \|\Delta \llai^n \|
    \to0.
\end{align}
Following a similar argument, 
since
\begin{align*}
    \frac{1}{\call_{5n}}
    % \sum_{j=\bn+1}^{\ns}
    % \sum_{i\in\calAj}
    \sum_{i=1}^N
    |\Delta \llci^n |
    &=
    \frac{1}{\call_{5n}}
    % \sum_{j=\bn+1}^{\ns}
    % \sum_{u=1}^d
    \sum_{i=1}^N
    | S^n_{i,0_{d\times d}, 1,0_q}|\to0,
    \\
    \frac{1}{\call_{5n}}
    % \sum_{j=\bn+1}^{\ns}
    % \sum_{i\in\calAj}
    \sum_{i=1}^N
    \|\Delta \lletai^n \|_1
    &=
    \frac{1}{\call_{1n}}
    % \sum_{j=\bn+1}^{\ns}
    \sum_{i=1}^N
    \sum_{w=1}^q
    | S^n_{i,0_{d\times d}, 0,e_{q,w}}|\to0,
\end{align*}
we obtain that 
\begin{align}
\label{pfeq:dthree_loss_coefficient2-dfive-mono}
    \frac{1}{\call_{5n}}
\sum_{i=1}^N
    |\Delta \llci^n |\to0,~
    \frac{1}{\call_{5n}}
\sum_{i=1}^N
    \|\Delta \lletai^n \|\to0.
\end{align}
Now taking
the summation of limits in equations
\eqref{pfeq:dthree_loss_coefficient1-dfive-mono} - 
\eqref{pfeq:dthree_loss_coefficient2-dfive-mono},
we could deduce that
    \begin{align*}
    1=\frac{\call_{5n}}{\call_{5n}}
    =
    \frac{1}{\call_{5n}}\cdot
    \sum_{i=1}^N
    \left(
    \|\Delta \llai^n \|+
    |\Delta \llci^n |+
    \|\Delta \lletai^n \|
    \right)
    \to 0,
\end{align*}
as $n\to\infty$,
which is a contradiction.
Thus, at least one among the ratios
$S^n_{i,\alpha_{1:3}}/\call_{5n}$
must not approach zero as $n\to\infty$.
Let us denote by $m_n$ the maximum of the absolute values of those elements.
It follows from the previous result that $1/m_n\not\to \infty$ as $n\to\infty$.

\textbf{Step 3 - Application of Fatou’s lemma:} 
In this step, we apply Fatou's lemma to obtain the desired inequality in equation \eqref{pfeq:dthree_local-dfive-mono}.
Recall in the beginning we have assumed that
$
        \|\tf_{\tG_n}-\tf_{\llG} \|_{L^2(\mu)}/\call_{5n}\to0
$
and the topological equivalence between $L_1$-norm and $L_2$-norm on finite-dimensional vector space over $\Real$,
we obtain $
        \|\tf_{\tG_n}-\tf_{\llG} \|_{L^1(\mu)}/\call_{5n}\to0
$. 
By applying the Fatou’s lemma, we get
\begin{align*}
    0=\lim_{n\to\infty}\frac{\|\tf_{\tG_n}-\tf_{\llG} \|_{L^1(\mu)}}{m_n \call_{5n}}
    \geq
    \int\liminf_{n\to\infty}
    \frac{|\tf_{\tG_n}(x)-\tf_{\llG}(x)|}{m_n \call_{5n}}d\mu(x)
    \geq 0.
\end{align*}
This result suggests that for almost every $x$, 
\begin{align}
\label{pfeq:dthree_fatou-dfive-mono}
    \frac{\tf_{\tG_n}(x)-\tf_{\llG}(x)}{m_n \call_{5n}}\to 0.
\end{align}
Let us denote 
\begin{align*}
    \frac{S^n_{i,\alpha_{1:3}}}{m_n \call_{5n}}\to s_{i,\alpha_{1:3}},
\end{align*}
as $n\to\infty$ with a note that at least one among the limits $s_{i,\alpha_{1:3}}$
is non-zero.
Then from equation \eqref{pfeq:dthree_fatou-dfive-mono},
we will have
\begin{align*}
    \sum_{i=1}^N
    \sum_{|\alpha|=1}
    s^n_{i,\alpha_1,\alpha_2,\alpha_3}
    \frac{\partial^{|\alpha_1| +|\alpha_2|}\sigma}{\partial A^{\alpha_1}   \partial c^{\alpha_2}}(x,\llai,\llci)
    \frac{\partial^{|\alpha_3|} \cale}{\partial \eta^{\alpha_3}}(x,\lletai)
    =0,   
\end{align*}
for almost every $x$.
Note that the expert function $\cale(\cdot,\eta)$ is weakly identifiable, then the above equation implies that
\begin{align*}
     s_{i,\alpha_{1:3}}=0,
\end{align*}
for any $i\in[N]$,   
    $(\alpha_1 ,\alpha_2,\alpha_3)\in\bbN^{d\times d}\times\bbN\times\bbN^q$
such that 
$\sum_{i=1}^3|\alpha_i|=1 $.
This violates that at least one among the limits in the set $\{s_{i,\alpha_{1:3}} \}$ is different from zero.

Thus, we obtain the local inequality in equation \eqref{pfeq:dthree_local-dfive-mono}. Completes the proof.
\end{proof}

\section{Additional Experiments}
\label{appsec:additional-experiments}

\begin{figure*}[ht]
    \centering
    \subfloat[\textbf{Sparse Regime}\label{fig:sparse-regime}]{
        \includegraphics[width=0.75\textwidth]{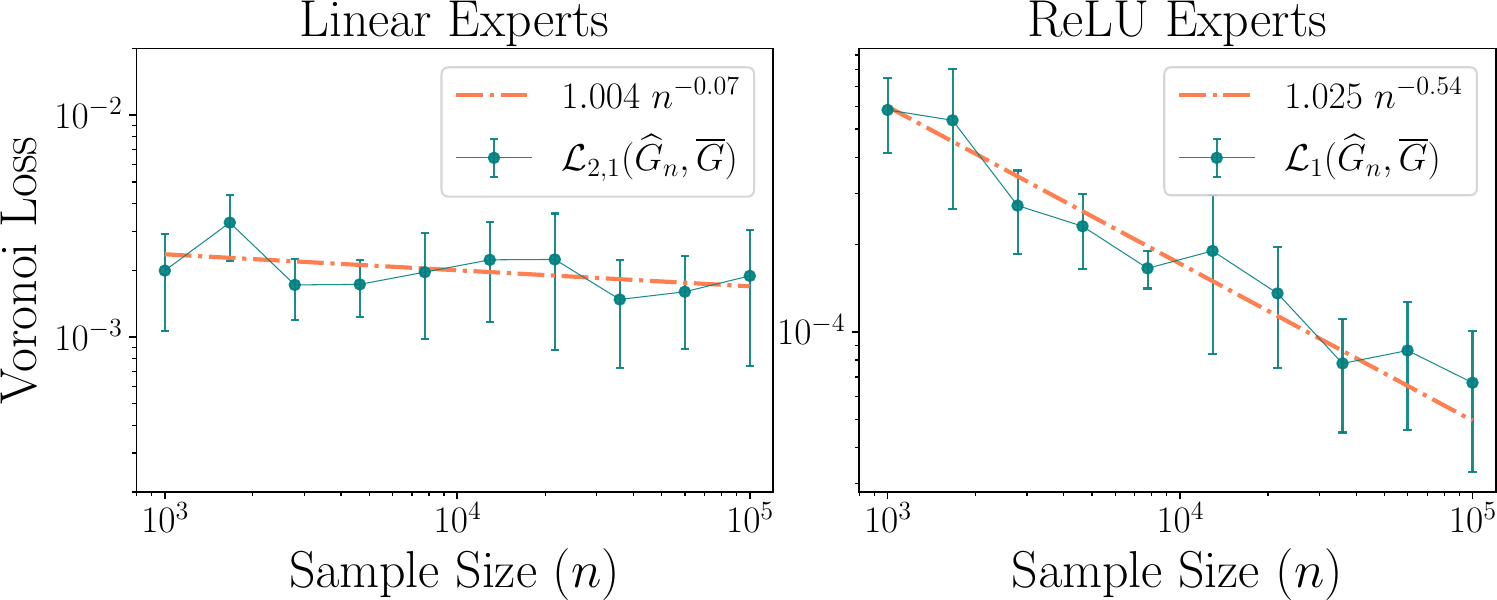}
    }\\
    % \vspace{1em}
    \subfloat[\textbf{Dense Regime}\label{fig:dense-regime}]{
        \includegraphics[width=0.75\textwidth]{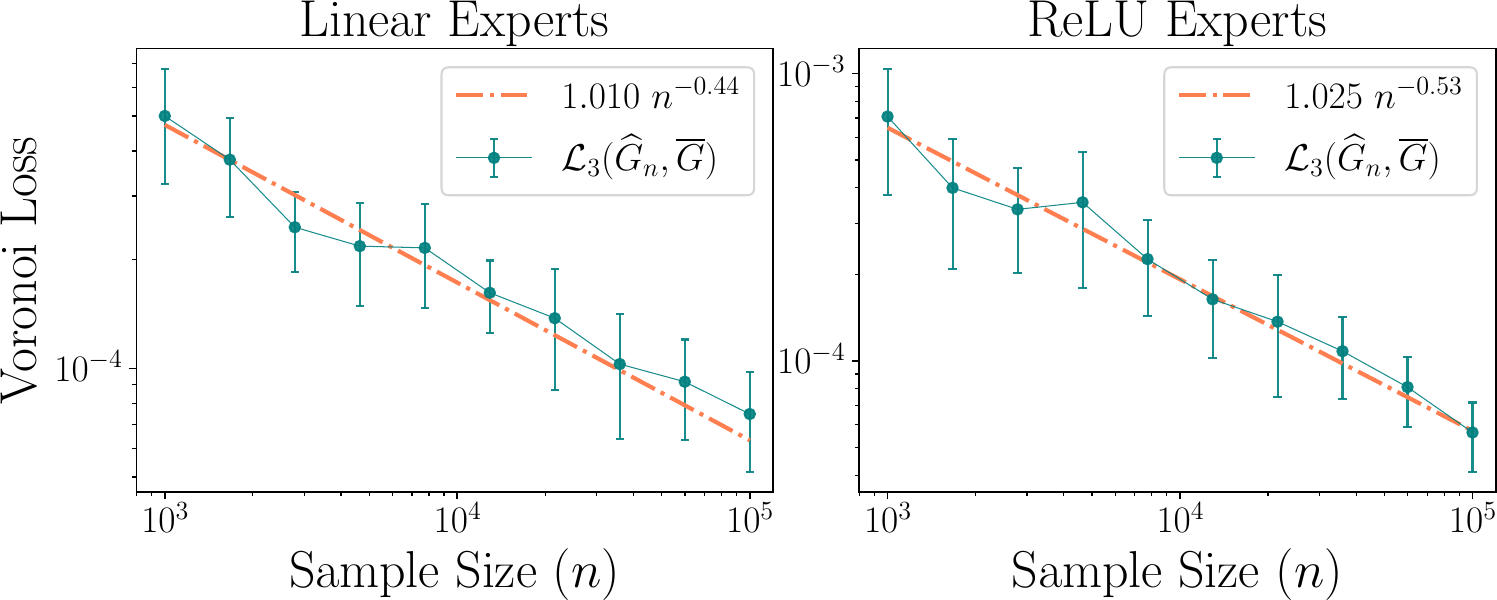}
    }
    % \vspace{1em}
    \caption{Log-log plots of empirical convergence rates of Voronoi losses for softmax and sigmoid quadratic gating mechanisms. 
    \ref{fig:sparse-regime} Comparison between sigmoid quadratic gating MoE with ReLU and linear experts under a sparse regime for ground-truth parameters. 
    \ref{fig:dense-regime} Comparison between sigmoid quadratic gating with ReLU and linear experts under a dense regime for ground-truth parameters. 
    Each plot illustrates the empirical Voronoi loss convergence rates, with solid lines representing the Voronoi losses and dash-dotted lines showing fitted trends to emphasize the empirical rates.}
    \label{fig:sparse-vs-dense}
    % \vspace{-1em}
\end{figure*}

% \begin{figure*}[ht]
%     \centering
%     \subfloat[\textbf{Sparse Regime}\label{fig:sparse-regime}]{
%         \includegraphics[width=0.48\textwidth]{icml2025/figures/exp3.pdf}
%     }
%     \hfill
%     \subfloat[\textbf{Dense Regime}\label{fig:dense-regime}]{
%         \includegraphics[width=0.48\textwidth]{icml2025/figures/exp4.pdf}
%     }
%     \vspace{-0.5em}
%       \caption{Log-log plots of empirical convergence rates of Voronoi losses for softmax and sigmoid quadratic gating mechanisms. \ref{fig:relu-experts} Comparison between sigmoid quadratic gating MoE with ReLU and linear experts under a sparse regime for ground-truth parameters. \ref{fig:linear-experts} Comparison between sigmoid quadratic gating with ReLU and linear experts under a dense regime for ground-truth parameters. Each plot illustrates the empirical Voronoi loss convergence rates, with solid lines representing the Voronoi losses and dash-dotted lines showing fitted trends to emphasize the empirical rates.}
%     \label{fig:sparse-vs-dense}
%     \vspace{-1em}
% \end{figure*}

\textbf{Sigmoid gating convergence rate.} To empirically validate the theoretical findings presented in Section~\ref{sec:sample-efficiency}, we analyze the parameter estimation rates of the sigmoid quadratic gating MoE in a well-specified setting. Our evaluation spans various expert configurations, considering both dense and sparse regimes, to assess the sample complexity of the model.

\textbf{Setup.} Synthetic data is generated using a sigmoid quadratic gating MoE model with both ReLU and linear experts. Models are then fitted to the data for varying sample sizes, and the empirical parameter estimation rates are evaluated. Comparisons are made for ReLU and linear experts under both dense and sparse regimes to verify the theoretical findings. Please refer to Appendix~\ref{appendix:experimental-details} for a detailed description of synthetic data generation and ground-truth parameters.

\textbf{Results.}
Figure~\ref{fig:sparse-vs-dense} illustrates the empirical convergence rates of Voronoi loss for different expert configurations in the sigmoid quadratic gating MoE model under both sparse and dense regimes. Error bars indicate three standard deviations to account for variability across different runs. Specifically, Figure~\ref{fig:sparse-regime} compares the empirical Voronoi loss for MoE models with ReLU and linear experts in the sparse regime. The sigmoid quadratic gating MoE with linear experts exhibits a slower convergence rate of $\mathcal{O}(n^{-0.07})$, whereas the model with ReLU experts achieves a significantly faster rate of $\mathcal{O}(n^{-0.54})$.

Similarly, Figure~\ref{fig:dense-regime} presents results for the dense regime, where both expert types achieve fast convergence rates. Notably, the MoE model with ReLU experts attains a convergence rate of $\mathcal{O}(n^{-0.53})$, while the model with linear experts follows with a rate of $\mathcal{O}(n^{-0.44})$. These results empirically validate our theoretical predictions, demonstrating that the sigmoid quadratic gating MoE yields superior parameter estimation rates across different expert configurations, particularly benefiting from the use of ReLU experts.

\section{Experimental Details}
\label{appendix:experimental-details}
% provide the details for the numerical experiments on synthetic data, and the
% experiments with real data on language modeling conducted in Section~\ref{sec:experiments}.
% \subsection{Numerical Experiments}
% \label{appendix:numerical-details}
% Here, we 
In this appendix, we provide a comprehensive overview of the generation process of synthetic data in our numerical experiments.

%\subsubsection{Synthetic Data Generation}
For each experiment, we generate synthetic data using the corresponding MoE model. Specifically, we construct a dataset $\{(X_i, Y_i)\}_{i=1}^n \subset \mathbb{R}^d \times \mathbb{R}$ as follows. First, we sample the input features independently from a uniform distribution:
\begin{align}
    X_i \sim \mathrm{Uniform}([-1, 1]^d), \quad \text{for } i = 1, \dots, n.
\end{align}
The response variable $Y_i$ is then generated according to the following model:
\begin{align}
    Y_{i} = f_{G_{*}}(X_{i}) + \varepsilon_{i}, \quad \text{for } i = 1, \dots, n,
\end{align}
where $\varepsilon_i$ represents independent Gaussian noise with variance $\nu = 0.01$. The underlying regression function $f_{G_{*}}(\cdot)$ is defined as:
\begin{align}
    f_{G_{*}}(x) = \sum_{i=1}^{k} \mathcal{G}_i(x) \cdot \phi((\alpha_i^*)^\top x + \beta_i^*),
\end{align}
where $\mathcal{G}_i(x)$ represents the gating function, which is modeled using either a softmax or a sigmoid quadratic transformation, and $\phi(\cdot)$ denotes the expert activation function, which can be either a ReLU or a linear transformation. The input dimension is set to $d = 8$, and we employ $\ns = 8$ experts. The gating and expert parameters are drawn from pre-specified distributions, as described in the following sections. 

\textbf{Ground-truth gating parameters.}  
In the dense regime, for both sigmoid and softmax quadratic gating functions, the ground-truth parameters $(A_i^*, b_i^*, c_i^*)$ are independently drawn from an isotropic Gaussian distribution with variance $\nu_g = 1/d$ for all $1 \leq i \leq 8$.  

In the sparse regime, we enforce sparsity by setting $(A_i^*, b_i^*, c_i^*) = 0$ for $i = 7, 8$, while the remaining parameters are sampled as in the dense regime.

\textbf{Ground-truth expert parameters.} All ground-truth expert parameters are similarly drawn independently from an isotropic Gaussian distribution with zero mean and variance $\nu_e = 1/d$. These parameters are sampled once before running the experiments and remain unchanged throughout all trials.

% \textbf{Computer Resource.} All the numerical experiments are run on the Macbook Air Apple M4 Chip.

\textbf{Computational Resources.} All numerical experiments were conducted on a MacBook Air equipped with an Apple M4 chip.

\bibliography{references}
\bibliographystyle{abbrv}
\end{document}